\documentclass[a4paper,11pt]{article}
\usepackage[english]{babel}
\usepackage[latin1]{inputenc}
\usepackage{amsmath,amssymb,latexsym,amsfonts}
\usepackage{amsthm}
\usepackage{bm}
\usepackage{bbm}
\usepackage{mathtools}
\usepackage{mathrsfs}
\usepackage{graphics}
\usepackage{url}

\usepackage{algorithm}
\usepackage{algorithmic}
\usepackage{color}
\usepackage{float}
\usepackage{floatpag}
\floatpagestyle{empty}
\usepackage{booktabs}

\usepackage{graphicx} \usepackage{verbatim}
\usepackage[font=footnotesize,labelfont=footnotesize]{caption}
\usepackage[font=scriptsize,labelfont=scriptsize]{subcaption}
\usepackage{enumerate}
\usepackage{tabulary}
\usepackage{multirow}
\usepackage{adjustbox}
\allowdisplaybreaks
\usepackage[bb=boondox]{mathalfa}

\DeclarePairedDelimiter\ceil{\lceil}{\rceil}

\DeclareMathOperator*{\argmax}{arg\,max}
\DeclareMathOperator*{\argmin}{arg\,min}

\makeatletter
\newcommand{\distas}[1]{\mathbin{\overset{#1}{\kern\z@\sim}}}%
\newsavebox{\mybox}\newsavebox{\mysim}
\newcommand{\distras}[1]{%
  \savebox{\mybox}{\hbox{\kern3pt$\scriptstyle#1$\kern3pt}}
  \savebox{\mysim}{\hbox{$\sim$}}%
  \mathbin{\overset{#1}{\kern\z@\resizebox{\wd\mybox}{\ht\mysim}{$\sim$}}}%
}
\makeatother

\usepackage{hyperref}

\newtheorem{thm}{Theorem}[section]
\newtheorem{cor}[thm]{Corollary}
\newtheorem{prop}[thm]{Proposition}
\newtheorem{lem}[thm]{Lemma}

\newtheorem{defn}[thm]{Definition}

\newtheorem{assumption}{Assumption}

\renewcommand{\L}{\mathcal{L}}

\newcommand{\D}{\mathscr{D}}
\newcommand{\degg}{\text{deg}}
\newcommand{\diam}{\text{diam}}
\newcommand{\Length}{T}
\newcommand{\Lip}{F}
\newcommand{\K}{\mathcal{K}}
\newcommand{\kappaOne}{\beta_{1}}
\newcommand{\kappaTwo}{\beta_{2}}
\newcommand{\NumSim}{N_{\text{sim}}}
\newcommand{\Prob}{\mathbb{P}}
\newcommand{\Ex}{\mathbb{E}}
\newcommand{\Var}{\text{Var}}
\newcommand{\X}{\mathcal{X}}

\newcommand{\fmax}{f_{\text{max}}}
\newcommand{\fmin}{f_{\text{min}}}
\newcommand{\Deg}{\text{Deg}}

\begin{document}

\title{Balancing Geometry and Density: Path Distances on High-Dimensional Data}

\author{Anna Little\footnote{Department of Mathematics, University of Utah, Salt Lake City, UT 84112, USA (\url{little@math.utah.edu})}
	\and Daniel McKenzie\footnote{Department of Mathematics, UCLA, Los Angeles, CA 90095, USA (\url{mckenzie@math.ucla.edu})}
        \and James M. Murphy\footnote{Department of Mathematics, Tufts University, Medford, MA 02155, USA \url{jm.murphy@tufts.edu}}        }

\maketitle

\begin{abstract}New geometric and computational analyses of power-weighted shortest-path distances (PWSPDs) are presented.  By illuminating the way these metrics balance geometry and density in the underlying data, we clarify their key parameters and illustrate how they provide multiple perspectives for data analysis.  Comparisons are made with related data-driven metrics, which illustrate the broader role of density in kernel-based unsupervised and semi-supervised machine learning.   Computationally, we relate PWSPDs on complete weighted graphs to their analogues on weighted nearest neighbor graphs, providing high probability guarantees on their equivalence that are near-optimal.  Connections with percolation theory are developed to establish estimates on the bias and variance of PWSPDs in the finite sample setting.  The theoretical results are bolstered by illustrative experiments, demonstrating the versatility of PWSPDs for a wide range of data settings.  Throughout the paper, our results generally require only that the underlying data is sampled from a compact low-dimensional manifold, and depend most crucially on the intrinsic dimension of this manifold, rather than its ambient dimension.
\end{abstract}

\section{Introduction}

The analysis of high-dimensional data is a challenge in modern statistical and machine learning.  In order to defeat the \emph{curse of dimensionality} \cite{Hughes1968, Gyorfi2006, Bellman2015}, distance metrics that efficiently and accurately capture intrinsically low-dimensional latent structure in high-dimensional data are required.  Indeed, this need to capture low-dimensional linear and nonlinear structure in data has led to the development of a range of data-dependent distances and related dimension reduction methods, which have been widely employed in applications \cite{Mahalanobis1936generalized, Tenenbaum2000global, Belkin2003laplacian, Donoho2003hessian, Coifman2006diffusion, Maaten2008visualizing}.  Understanding how these metrics trade off fundamental properties in the data (e.g. local versus global structure, geometry versus density) when making pointwise comparisons is an important challenge in their use, and may be understood as a form of model selection in unsupervised and semi-supervised machine learning problems.

\subsection{Power-Weighted Shortest Path Distances}

In this paper we analyze \emph{power}-\emph{weighted shortest path distances (PWSPDs)} and develop their applications to problems in machine learning.  These metrics compute the shortest path between two points in the data, accounting for the underlying density of the points along the path.  Paths through low-density regions are penalized, so that the optimal path must balance being ``short" (in the sense of the classical geodesic distance) with passing through high-density regions.  We consider a finite data set $\mathcal{X} = \{x_{i}\}_{i=1}^{n}\subset \mathbb{R}^{D}$, which we usually assume to be intrinsically low-dimensional, in the sense that there exists a compact $d$-dimensional Riemannian {\em data manifold} $\mathcal{M}\subset\mathbb{R}^{D}$ and a probability density function $f(x)$ supported on $\mathcal{M}$ such that $\{x_i\}_{i=1}^{n}\distras{i.i.d.} f(x)$. 

\begin{defn}\label{defn:PWSPD}For  $p \in [1,\infty)$ and for $x,y\in\mathcal{X}$, the \emph{(discrete) $p$-weighted shortest path distance (PWSPD)} from $x$ to $y$ is:\begin{align}
\label{eqn:DPWSPD}
\ell_p(x,y) =\min_{\pi=\{x_{i_{j}}\}_{j=1}^\Length}\left(\sum_{j=1}^{\Length-1} \|x_{i_{j}}-x_{i_{j+1}}\|^p\right)^{\frac{1}{p}},
\end{align}
where $\pi$ is a path of points in $\mathcal{X}$ with $x_{i_{1}}=x$ and $x_{i_{\Length}}=y$ and $\|\cdot\|$ is the Euclidean norm.
\end{defn}
Early uses of density-based distances for interpolation \cite{Saul1997_Variational} led to the formulation of PWSPD in the context of unsupervised and semi-supervised learning and applications \cite{fischer2001path, Vincent2003_Density, Bousquet2004_Measure, Sajama2005_Estimating, Chang2008_Robust, Bijral2011_Semi, Moscovich2017_Minimax, mckenzie2019power, Little2020path, Zhang2021_Hyperspectral, Borghini2020_Intrinsic}.  
It will occasionally be useful to think of $\ell_{p}^{p}(\cdot,\cdot)$ as the path distance in the complete graph on $\mathcal{X}$ with edge weights $\|x_i - x_j\|^{p}$, which we shall denote $\mathcal{G}_{\mathcal{X}}^{p}$. When $p=1$, $\ell_{1}(x,y) = \|x - y\|$, i.e. the Euclidean distance.  As $p$ increases, the largest elements in the set of path edge lengths $\left\{\left\|x_{i_{j}}-x_{i_{j+1}}\right\|\right\}_{j=1}^{\Length-1}$ begin to dominate the optimization (\ref{eqn:DPWSPD}), so that paths through higher density regions (with shorter edge lengths) are promoted.  When $p\rightarrow\infty$, $\ell_{p}$ converges (up to rescaling by the number of edges achieving maximal length) to the longest-leg path distance $\displaystyle\ell_{\infty}(x,y)=\min_{\pi=\{x_{i_{j}}\}_{j=1}^\Length}\max_{j=1,\dots,\Length-1} \|x_{i_{j}}-x_{i_{j+1}}\|$ \cite{Little2020path} and is thus driven by the density function $f$.  Outside these extremes, $\ell_{p}$ balances taking a ``short" path and taking one through regions of high density. Note that $\ell_p$ can be defined for $p<1$, but it does not satisfy the triangle inequality and is thus not a metric ($\ell_{p}^{p}$ however {\em is} a metric for all $p > 0$).  This case was studied in \cite{alamgir2012shortest}, where it is shown to have counterintuitive properties that should preclude its use in machine learning and data analysis.  

While (\ref{eqn:DPWSPD}) is defined for finite data, it admits a corresponding continuum formulation.

\begin{defn}\label{defn:CPWSPD}Let $(\mathcal{M},g)$ be a compact, $d$-dimensional Riemannian manifold and $f$ a continuous density function on $\mathcal{M}$ that is lower bounded away from zero (i.e. $f_{\min} := \min_{x\in\mathcal{M}}f(x) > 0$ on $\mathcal{M}$).  For $p \in [1,\infty)$ and $x,y\in\mathcal{M}$, the \emph{(continuum) $p$-weighted shortest path distance} from $x$ to $y$ is: \begin{equation}\label{eqn:CPWSPD}\L_p(x,y) = \left(\inf_{\gamma}\int_0^1 \frac{1}{f(\gamma(t))^{\frac{p-1}{d}}} \sqrt{g\left(\gamma'(t), \gamma'(t)\right)}dt\right)^{\frac{1}{p}},\end{equation}where $\gamma:[0,1]\rightarrow\mathcal{M}$ is a $\mathcal{C}^{1}$ path with $\gamma(0)=x, \gamma(1)=y$.
\end{defn}
Note $\L_1$ is simply the geodesic distance on $\mathcal{M}$. However for $p>1$ and a nonuniform density, the optimal path $\gamma$ is generally not the geodesic distance on $\mathcal{M}$: $\L_p$ favors paths which travel along high-density regions, and detours off the classical $\L_1$ geodesics are thus acceptable. The parameter $p$ controls how large of a detour is optimal; for large $p$, optimal paths may become highly nonlocal and different from classical geodesic paths. 

It is known \cite{Hwang2016_Shortest, Groisman2018nonhomogeneous} that when $f$ is continuous and positive, for $p>1$ and all  $x,y\in \mathcal{M}$,  
\begin{align}
\label{equ:cont_limit_PD}
\displaystyle\lim_{n\rightarrow\infty} n^{\frac{p-1}{pd}}\ell_{p}(x,y) =C_{p,d}\L_{p}(x,y)
\end{align}
for an absolute constant $C_{p,d}$ depending only on $p$ and $d$, i.e. that the discrete PWSPD computed on an i.i.d. sample from $f$ (appropriately rescaled) is a consistent estimator for the continuum PWSPD.  In particular,  \eqref{equ:cont_limit_PD} is established by \cite{Groisman2018nonhomogeneous} for $C^1$, isometrically embedded manifolds and by \cite{Hwang2016_Shortest} for smooth, compact manifolds without boundary and for $\ell_p$ defined using geodesic distance. We thus define the normalized (discrete) path metric 
\begin{align}
\label{equ:normDPWSPD}
\tilde{\ell}_p(x,y) &:= n^{\frac{p-1}{pd}}\ell_{p}(x,y)\, .
\end{align}
The $n^{\frac{p-1}{pd}}$ normalization factor accounts for the fact that for $p>1$, $\ell_{p}$ converges uniformly to 0 as $n\rightarrow \infty$ \cite{mckenzie2019power}.  Note that the $1/p$ exponent in \eqref{eqn:DPWSPD} and \eqref{equ:cont_limit_PD} is necessary to obtain a metric that is homogeneous.  Moreover, as $p\rightarrow\infty$, $\L_{p}$ is constant on regions of constant density, but $\L_{p}^p$ is not.  Indeed, consider a uniform distribution on $[0,1]^{d}$, which has density $f=\mathbbm{1}_{[0,1]^{d}}$. Then for all $x,y\in [0,1]^{d}$ and for all $p$, $\L_{p}^{p}(x,y)=\|x-y\|$.  On the other hand, for all $x,y\in [0,1]^{d}$, $\L_p(x,y)=\|x-y\|^{1/p} \rightarrow 1$ as $p \rightarrow \infty$, i.e. all points are equidistant in the limit $p\rightarrow\infty$. Thus the $1/p$ exponent in \eqref{eqn:DPWSPD} and \eqref{equ:cont_limit_PD} is necessary to obtain an entirely density-based metric for large $p$.  

In practice, it is more efficient to compute PWSPDs in a sparse graph instead of a complete graph.  It is thus natural to define PWSPDs \emph{with respect to a subgraph $\mathcal{H}$ of $\mathcal{G}_{\mathcal{X}}^{p}$}.

\begin{defn}\label{defn:PWSPD_G}Let $\mathcal{H}$ be any subgraph of $\mathcal{G}_{\mathcal{X}}^{p}$. For $x,y\in X$, let $\mathcal{P}_{\mathcal{H}}(x,y)$ be the set of paths connecting $x$ and $y$ in $\mathcal{H}$.  For $p \in [1,\infty)$ and for $x,y\in\mathcal{X}$, the \emph{(discrete) $p$-weighted shortest path distance (PWSPD) with respect to $\mathcal{H}$} from $x$ to $y$ is: \[\displaystyle\ell_p^{\mathcal{H}}(x,y) =\min_{\pi=\{x_{i_{j}}\}_{j=1}^{\Length}\in \mathcal{P}_{\mathcal{H}}(x,y)}\left (\sum_{j=1}^{\Length-1} \|x_{i_{j}} - x_{i_{j+1}}\|^p\right)^{\frac{1}{p}}.\]
\end{defn}
Clearly $\ell_{p}^{\mathcal{G}_{\mathcal{X}}^{p}}(\cdot,\cdot) = \ell_{p}(\cdot,\cdot)$. In order to compute all-pairs PWSPDs in a complete graph with $n$ nodes (i.e. $\ell_{p}(x_{i},x_{j})$ for all $x_{i},x_{j}\in \mathcal{X}$), a direct application of Dijkstra's algorithm has complexity $O(n^{3})$.  Let $\mathcal{G}^{p,k}_{\mathcal{X}}$ denote the {\em $k$NN graph}, constructed from $\mathcal{G}^{p}_{\mathcal{X}}$ by retaining only edges $\{x,y\}$ if  $x$ is amongst the $k$ nearest neighbors of $y$ in $\mathcal{X}$ (we say: ``$x$ is a $k$NN of $y$'' for short) or vice versa. In some cases the PWSPDs with respect to $\mathcal{G}^{p,k}_{\mathcal{X}}$ are known to coincide with those computed in $\mathcal{G}_{\mathcal{X}}^{p}$ \cite{Groisman2018nonhomogeneous,chu2020exact}. If so,  we say the $k$NN graph is a {\em $1$-spanner} of $\mathcal{G}_{\mathcal{X}}^{p}$.  This provides a significant computational advantage, since $k$NN graphs are much sparser, and reduces the complexity of computing all-pairs PWSPD to $O(kn^{2})$ \cite{johnson1977efficient}. 

\subsection{Summary of Contributions}

This article develops new analyses, computational insights, and applications of PWSPDs, which may be summarized in three major contributions. First, we establish that when $\frac{p}{d}$ is not too large, PWSPDs locally are density-rescaled Euclidean distances.  We give precise error bounds that improve over known bounds \cite{Hwang2016_Shortest} and are tight enough to prove the local equivalence of Gaussian kernels constructed with PWSPD and density-rescaled Euclidean distances.  We also develop related theory which clarifies the role of density in machine learning kernels more broadly.  A range of machine learning kernels that normalize in order to mitigate or leverage differences in underlying density are considered and compared to PWSPD.  Relatedly, we analyze how PWSPDs become increasingly influenced by the underlying density as $p\rightarrow\infty$.  We also illustrate the role of density and benefits of PWSPDs on illustrative data sets.

Second, we improve and extend known bounds on $k$ \cite{Groisman2018nonhomogeneous,mckenzie2019power, chu2020exact} guaranteeing that the $k$NN graph is a $1$-spanner of $\mathcal{G}_{\mathcal{X}}^{p}$. Specifically, we show that for any $1<p<\infty$, the $k$NN graph is a $1$-spanner of $\mathcal{G}_{\mathcal{X}}^{p}$ with probability exceeding $1-1/n$ if $k \ge C_{p,d,f,\mathcal{M}}\cdot\log(n)$, for an explicit constant $C_{p,d,f,\mathcal{M}}$ that depends on the density power $p$, intrinsic dimension $d$, underlying density $f$, and the geometry of the manifold $\mathcal{M}$, but is crucially independent of $n$.  These results are proved both in the case that the manifold is isometrically embedded and in the case that the edge lengths are in terms of intrinsic geodesic distance on the manifold.  Our results provide an essential computational tool for the practical use of PWSPDs, and their key dependencies are verified numerically with extensive large-scale experiments.

Third, we bound the convergence rate of PWSPD to its continuum limit using a percolation theory framework, thereby quantifying the \cite{Hwang2016_Shortest, Groisman2018nonhomogeneous} asymptotic convergence result (\ref{equ:normDPWSPD}).  Specifically, we develop bias and variance estimates by relating results on Euclidean first passage percolation (FPP) to the PWSPD setting.  Surprisingly, these results suggest that the variance of PWSPD is essentially independent of $p$, and depends on the intrinsic dimension $d$ in complex ways.  Numerical experiments verify our theoretical analyses and suggest several conjectures related to Euclidean FPP that are of independent interest.  

\subsection{Notation}

We shall use the notation in Table \ref{tab:notation} consistently, though certain specialized notation will be introduced as required. We assume throughout that the data $\mathcal{X}$ is drawn from a compact Riemannian data manifold $(\mathcal{M},g)$, with additional assumptions imposed on $\mathcal{M}$ as needed; we do not rigorously consider the more general case that $\mathcal{X}$ is drawn from a distribution supported \emph{near} $\mathcal{M}$.  If $\mathcal{M}\subset\mathbb{R}^{D}$, we assume that it is isometrically embedded in $\mathbb{R}^{D}$, i.e. $g$ is the unique metric induced by restricting the Euclidean metric on $\mathbb{R}^{D}$ to $\mathcal{M}$, unless otherwise stated.  If an event holds with probability $1-c/n$, where $n = |\mathcal{X}|$ and $c$ is independent of $n$, we say it holds \emph{with high probability (w.h.p.)}.

\begin{table}[htbp!]
\vspace{0.1in}
\begin{center}
\begin{small}
\begin{tabular}{c|c}
\toprule
\textsc{Notation} & \textsc{Definition}  \\
\midrule
$\mathcal{X}$ & $\mathcal{X} = \{x_i\}_{i=1}^{n}\subset\mathbb{R}^{D}$, a finite data set \\
$D$ & ambient dimension of data set $\mathcal{X}$ \\
$d$           & intrinsic dimension of data set $\mathcal{X}$ \\
$\|v\|_{p}$  & $\|v\|_{p}=(\sum_{i=1}^{D}|v_{i}|^{p})^{\frac{1}{p}}$, the Euclidean $p$-norm of $v\in\mathbb{R}^{D}$\\
$\|v\|$ & $\|v\|_{2}$, the Euclidean $2$-norm\\
$|c|$ & the absolute value of $c\in\mathbb{R}$\\
$\mathcal{G}_{\mathcal{X}}^{p}$ & complete graph on $\mathcal{X}$ with edge weight $\|x_{i}-x_{j}\|^{p}$ between $x_{i},x_{j}\in\mathcal{X}$\\
$\{x,y\}$ & edge between nodes $x,y$ in a graph\\
$(\mathcal{M},g)$ & a Riemannian manifold with associated metric $g$\\
$\kappa$ & measure of curvature on $\mathcal{M}$; see Definition \ref{defn:LocalFlatness}\\
$\kappa_{0}$ & measure of regularity on $\mathcal{M}$; see Definition \ref{defn:V}\\
$\zeta$ & reach of a manifold $\mathcal{M}$; see Definition \ref{defn:reach}\\
$f(x)$        & probability density function from which $\mathcal{X}$ is drawn \\
$\fmin$, $\fmax$ & minimum and maximum values of density $f$ defined on compact manifold $\mathcal{M}$\\
$\{\pi_{i}\}_{i=1}^{T}$, $\gamma(t)$ & discrete, continuous path \\
$\ell_{p}(x,y)$ & discrete PWSPD, see \eqref{eqn:DPWSPD} \\
$ \tilde{\ell}_{p}(x,y)$ & rescaled version of $\ell_{p}(x,y)$, see \eqref{equ:normDPWSPD} \\
$\ell_{p}^{\mathcal{H}}(x,y)$ & discrete PWSPD defined on the subgraph $\mathcal{H}\subset\mathcal{G}_{\mathcal{X}}^{p}$; see Definition \ref{defn:PWSPD_G}\\
$\L_{p}(x,y)$ & continuum PWSPD, see \eqref{eqn:CPWSPD} \\
$\D(x,y)$ & geodesic distance on manifold $\mathcal{M}$ \\
$\D_{f,\text{Euc}}(x,y)$ & density-based \emph{stretch} of Euclidean distance with respect to $f$ \\
$W, \Deg, L$ & weight, degree, and Laplacian matrices associated to a graph\\
$\delta(\cdot, \cdot)$ & arbitrary metric \\
$B_{\delta}(x,\epsilon)$ & $\{y \ | \ \delta(x,y)\leq \epsilon\}$,  ball of radius $\epsilon>0$ centered at $x$ with respect to $\delta$\\
$B(x,\epsilon)$ & Euclidean ball of radius $\epsilon>0$ centered at $x$, dimension determined by context\\
$\mathcal{D}_{\alpha, p}(x,y)$ & $p$-elongated set of radius $\alpha$ based at points $x,y$; see Definition \ref{defn:p_elongated_set}\\
$k$ & number of nearest neighbors (NN), sometimes dependent on $n$ (i.e. $k=k(n)$)\\
$\mu,\chi$       & percolation time, fluctuation constants \\
$\lambda$  & intensity parameter in a Poisson point process \\
$\bar{A}$ & complement of the set $A$\\
$\Ex[\xi], \Var[\xi]$ & expectation, variance of a random variable $\xi$\\
$\diam(A)$ & $\sup_{x,y\in A}\|x-y\|$, the Euclidean diameter of a set $A$\\
$\text{vol}(A)$ & volume of a set $A$, with dimension depending on context \\
$\bar{A}$ & complement of a set $A$\\
$\partial A$ & boundary of a set $A$\\
$ a \lesssim b $ &  $a\le C b$ for a constant $C$ independent of the dependencies of $a,b$\\
$ a \propto b $ & quantity $a$ is proportional to quantity $b$, i.e. $a\lesssim b$ and $b\lesssim a$\\
\bottomrule
\end{tabular}
\label{tab:notation}
\end{small}
\end{center}
\vskip -0.1in
\caption{\label{tab:notation} Notation used throughout the paper.} 
\end{table}

\section{Local Analysis: Density and Kernels}
\label{sec:LocalAnalysis}

Density-driven methods are commonly used for unsupervised and semi-supervised learning \cite{Cheng1995_Mean, Ester1996_Density, Coifman2006diffusion, Rinaldo2010_Generalized, Bijral2011_Semi, Azizyan2013_Density, Rodriguez2014_Clustering}.  Despite this popularity, the role of density is not completely clear in this context.  Indeed, some methods seek to leverage variations in density while others mitigate it.  In this section, we explore the role that density plays in popular machine learning kernels, including those used in self-tuning spectral clustering and diffusion maps. We compare with the effect of density in $\ell_p$-based kernels, and illustrate the primary advantages and disadvantages on toy data sets. 

\subsection{Role of Density in Graph Laplacian Kernels}
\label{subsec:RoleDensityOtherKernels}

A large family of algorithms  \cite{Belkin2003laplacian, belkin2007convergence, Shi2000, Ng2002, VonLuxburg2007} view data points as the nodes of a graph, and define the corresponding edge weights via a kernel function.  In general, by kernel we mean a function $\mathcal{K}:\mathbb{R}^{D}\times\mathbb{R}^{D}\rightarrow\mathbb{R}$ that captures a notion of \emph{similarity} between elements of $\mathbb{R}^{D}$.  More precisely, we suppose that $\mathcal{K}$ is of the form $\mathcal{K}(x_{i},x_{j})=h(\delta(x_{i},x_{j}))$ for some metric $\delta$ on $\mathbb{R}^{D}$ and smooth, positive, rapidly decaying (hence integrable) function $h:\mathbb{R}\rightarrow\mathbb{R}$.  Our technical results will pertain exclusively to the Gaussian kernel $\mathcal{K}(x_{i},x_{j})=\exp(-\delta(x_{i},x_{j})^{2}/\epsilon^{2})$ for some metric $\delta$ and scaling parameter $\epsilon>0$, albeit more general kernels have been considered in the literature \cite{antil2021fractional, Damelin2010_Energy, berry2016variable}.  Given $\mathcal{X}\subset\mathbb{R}^{D}$, one first defines a weight matrix $W\in\mathbb{R}^{n\times n}$ by $W_{ij} = \K(x_{i},x_{j})$ for some kernel $\mathcal{K}$, and diagonal degree matrix $\Deg\in\mathbb{R}^{n\times n}$ by $\Deg_{ii} = \sum_{j=1}^{n} W_{ij}$.  A \emph{graph Laplacian} $L$ is then defined using $W, \Deg$.  Then, the $K$ lowest frequency eigenvectors of $L$, denoted $\phi_1,\ldots,\phi_K$, define a $K$-dimensional spectral embedding of the data by $x_{i}\mapsto (\phi_{1}(x_{i}),\phi_{2}(x_{i}),\dots,\phi_{K}(x_{i}))$, where $\phi_{j}(x_{i})=(\phi_{j})_{i}$.  Commonly, a standard clustering algorithm such as $K$-means is then applied to the spectral embedding.  This procedure is known as \textit{spectral clustering} (SC). In unnormalized SC, $L=\Deg-W$, while in normalized SC either the random walk Laplacian $L_{\text{RW}} = \Deg^{-1}L$ or the symmetric normalized Laplacian $L_{\text{SYM}}= \Deg^{-1/2}L\Deg^{-1/2}$ is used.

Many modifications of this general framework have been considered. Although SC is better able to handle irregularly shaped clusters than many traditional algorithms  \cite{Arias2011_Clustering, Schiebinger2015_Geometry}, it is often unstable in the presence of low degree points and sensitive to the choice of scaling parameter $\epsilon$ when using the Gaussian kernel \cite{VonLuxburg2007}. These shortcomings motivated \cite{zelnik2005self} to apply SC with the \emph{self-tuning kernel} $W_{ij} = \exp\left(-\frac{\|x_i-x_j\|^2}{\sigma_{i,k}\sigma_{j,k}} \right),$ where $\sigma_{i,k}$ is Euclidean distance of $x_i$ to its $k^{\text{th}}$ NN.  To clarify how the data density influences this kernel, consider how $\sigma_{i,k}$ relates to the $k$NN density estimator at $x_i$:
\begin{align}
\label{equ:KNN_dens_est}
f_n(x_i) &:= \frac{k}{n \text{vol}(B(0,1))\sigma_{i, k}^d}\,.
\end{align}
It is known \cite{loftsgaarden1965nonparametric} that if $k=k(n)$ is such that $k(n) \rightarrow \infty$ while $k(n)/n \rightarrow 0$, then $f_n(x_i)$ is a consistent estimator of $f(x_i)$, as long as $f$ is continuous and positive.
Furthermore, if $f$ is uniformly continuous and  $k(n)/\log n \rightarrow \infty$ while $k(n)/n \rightarrow 0$, then $\sup_{i} |f_n(x_i)-f(x_i)| \rightarrow 0$ with probability 1  \cite{devroye1977strong}. Although these results assume the density $f$ is supported in $\mathbb{R}^d$, the density estimator (\ref{equ:KNN_dens_est}) is consistent in the general case when $f$ is supported on a $d$-dimensional Riemannian manifold $\mathcal{M} \subseteq \mathbb{R}^D$ for $\log n \ll k(n) \ll n$ \cite{farahmand2007manifold}.  For such $k(n)$, $\sigma_{i,k} \rightarrow \epsilon_{n,d}f(x_i)^{-\frac{1}{d}}$ for some constant $\epsilon_{n,d}$ depending on $n, d$. Thus, for $n$ large the kernel for self-tuning spectral clustering is approximately:
\begin{align}
\label{equ:self_tune_kernel_asym}
W_{ij} &\approx \exp\left(-f(x_i)^{\frac{1}{d}}f(x_j)^{\frac{1}{d}}\frac{\|x_i-x_j\|^2}{\epsilon_{n,d}^2} \right).
\end{align}
Relative to a standard SC kernel, (\ref{equ:self_tune_kernel_asym}) weakens connections in high density regions and strengthens connections in low density regions.

Diffusion maps \cite{Coifman2005_Geometric, Coifman2006diffusion} is a more general framework which reduces to SC for certain parameter choices.  More specifically, \cite{Coifman2006diffusion} considered the family of kernels
\begin{align}
\label{equ:DMkernel}
W_{ij} &= \frac{\exp\left(-\|x_i-x_j\|^2 / \epsilon^2\right)}{\degg(x_i)^{a}\degg(x_j)^{a}} \quad, \quad \degg(x_i) = \sum_{j=1}^{n} \exp\left(-\|x_i-x_j\|^2 / \epsilon^2\right)
\end{align}parametrized by $a\in [0,1]$, which determines the degree of density normalization. 
Since $\degg(x_i) \propto f(x_i) + O(\epsilon^2)$, $\degg(x_i)$ is a kernel density estimator of the density $f(x_i)$ \cite{berry2016local} and, up to higher order terms,
\begin{align}
\label{equ:DMkernel_density_form}
W_{ij}&\propto \frac{\exp\left(-\|x_i-x_j\|^2 / \epsilon^2\right)}{f(x_i)^{a}f(x_j)^{a}}.
\end{align}
Note that $f$ has an effect on the kernel similar to the self-tuning kernel (\ref{equ:self_tune_kernel_asym}): connections in high density regions are weakened, and connections in low density regions are strengthened. 
Let $L_{\text{RW}}^{a, \epsilon}$ denote the discrete random walk Laplacian using the weights $W_{ij}$ given in (\ref{equ:DMkernel}).  The discrete operator $-L_{\text{RW}}^{a,\epsilon}/\epsilon^2$ converges to the continuum Kolmogorov operator $\mathscr{L} \psi = \Delta\psi+(2-2a)\nabla\psi\cdot\frac{\nabla f}{f}$
as $n\rightarrow \infty, \epsilon\rightarrow 0^{+}$ for Laplacian operator $\Delta$ and gradient $\nabla$, both taken with respect to the Riemannian metric inherited from the ambient space  \cite{Belkin2003laplacian, Coifman2006diffusion, berry2016local}. When $a=0$, we recover standard spectral clustering; there is no density renormalization in the kernel but the limiting operator is density dependent. When $a=1$, $-L_{\text{RW}}^{1,\epsilon}/\epsilon^2 \rightarrow \Delta$; in this case the discrete operator is density dependent but the limiting operator is purely geometric, since the density term is eliminated.  We note that Laplacians and diffusion maps with various metrics and norms have been considered in a range of settings \cite{Xu2010clustering, Boninsegna2015_Investigating, Van2018recovering, Kileel2020_Manifold}.

\subsection{Local Characterization of PWSPD-Based Kernels}

While the kernels discussed in Section \ref{subsec:RoleDensityOtherKernels} compensate for discrepancies in density, PWSPD-based kernels strengthen connections through high-density regions and weaken connections through low-density regions. To illustrate more clearly the role of density in PWSPD-based kernels, we first show that locally the continuum PWSPD $\L_p^p$ is well-approximated by the density-based \emph{stretch} of Euclidean distance $\displaystyle \D_{f,\text{Euc}}(x,y) = \frac{\|x-y\|}{\left(f(x)f(y)\right)^{\frac{p-1}{2d}}},$ as long as $f$ does not vary too rapidly and $\mathcal{M}$ does not curve too quickly.  This is quantified in Lemma \ref{lem:local_equivalence},  which is then used to prove Theorem \ref{thm:metric_difference_euc}, which bounds the local deviation of $\L_p$ from $\D_{f,\text{Euc}}^{1/p}$. Finally, Corollary \ref{cor:kernel_equivalence} establishes that Gaussian kernels constructed with $\L_p$ and $\D_{f,\text{Euc}}^{1/p}$ are locally similar. 
Throughout this section we assume $\mathcal{M}\in S(d,\kappa,\epsilon_0)$ as defined below. 
\begin{defn}\label{defn:LocalFlatness}
	An isometrically embedded Riemannian manifold $\mathcal{M}\subset\mathbb{R}^{D}$ is an element of $S(d,\kappa,\epsilon_0)$ if it is compact with dimension $d$, $\text{vol}(\mathcal{M})=1$, and $\D(x,y) \leq \|x-y\|(1+\kappa\|x-y\|^2)$ for all $x,y\in \mathcal{M}$ such that $\D(x,y) \leq \epsilon_0$, where $\D(\cdot,\cdot)$ is geodesic distance on $\mathcal{M}$.
\end{defn}	
The condition $\D(x,y) \leq \|x-y\|(1+\kappa\|x-y\|^2)$ for all $x,y\in \mathcal{M}$ such that $\D(x,y) \leq \epsilon_0$ is equivalent to an upper bound on the second fundamental form: $\|II_{x}\| \leq \kappa$ for all $x\in\mathcal{M}$ \cite{antil2021fractional, Malik2019connecting}.  Note that this is also equivalent to a positive lower bound on the \emph{reach} \cite{Federer1959_Curvature} of $\mathcal{M}$ (e.g. Proposition 6.1 in \cite{Niyogi2008_Finding} and Proposition A.1 in \cite{Aamari2019estimating}); see Definition \ref{defn:reach}.

Let $B_{\L_p^p}(x,\epsilon)$ and $B_{\D}(x,\epsilon)$ denote, respectively, the (closed) $\L_p^p$ and geodesic balls centered at $x$ of radius $\epsilon$.  Let $\fmax = \max_y \{f(y):y\in \mathcal{M}\}$, $\fmin = \min_y \{f(y):y\in \mathcal{M}\}$ be the global density maximum and minimum.  Define the following local quantities:
\begin{align*}
\fmin(x,\epsilon) &=\min_y \left\{f(y) :  y \in B_{\D}(x,\epsilon(1+\kappa\epsilon^2)) \right\}, \\
\fmax(x,\epsilon) &=\max_y \{f(y) : y \in B_{\L_p^p}(x,\epsilon(1+\kappa\epsilon^2)/\fmin(x,\epsilon)^{\frac{p-1}{d}}) \}.
\end{align*}
Let $\rho_{x,\epsilon}  = \fmax(x,\epsilon)/\fmin(x,\epsilon)$, which characterizes the local discrepancy in density in a ball of radius $O(\epsilon)$ around the point $x$. 

The following Lemma establishes that $\L_{p}^p$ and $\D_{f,\text{Euc}}$ are locally equivalent, and that discrepancies depend on $(\rho_{x,\epsilon})^{\frac{p-1}{d}}$  and the curvature constant $\kappa$. We note similar estimates appear in \cite{alamgir2012shortest} for the special case $p=0$.  The proof appears in Appendix \ref{app:Proofs_for_LocalAnalysis}.

\begin{lem}
	\label{lem:local_equivalence}
	Let $\mathcal{M} \in S(d,\kappa,\epsilon_0)$. Then for all $y \in \mathcal{M}$ with $\D(x,y)\leq \epsilon_0$ and $\|x-y\|\leq \epsilon$,
	\begin{align}
	\label{equ:local_equivalence}
	\frac{1}{(\rho_{x,\epsilon})^{\frac{p-1}{d}}} \D_{f,\text{Euc}}(x,y) &\leq \L_p^p(x,y) \leq (\rho_{x,\epsilon})^{\frac{p-1}{d}} (1+\kappa \epsilon^2)\D_{f,\text{Euc}}(x,y) \, .
	\end{align}
\end{lem}	

Note that corresponding bounds in terms of geodesic distance follow easily from the definition of $\L_{p}$: $f_{\max}(x,\epsilon)^{-\frac{p-1}{d}}\D(x,y) \leq \L^{p}_{p}(x,y) \leq f_{\min}(x,\epsilon)^{-\frac{p-1}{d}}\D(x,y)$.  Lemma \ref{lem:local_equivalence} thus establishes that the metrics $\L_p^p$ and $\D_{f,\text{Euc}}$ are locally equivalent when (i) $\rho_{x,\epsilon}$ is close to 1, (ii) $\frac{p-1}{d}$ is not too large, and (iii) $\kappa$ is not too large.  However, when $\frac{p-1}{d} \gg 1$, $\L_p^p$ balls may become highly nonlocal in terms of geodesics. 

The following Theorem establishes the local equivalence of $\L_p$ and $\D_{f,\text{Euc}}^{1/p}$ (and thus kernels constructed using these metrics). Assuming the density does not vary too quickly, Lemma \ref{lem:local_equivalence} can be used to show that locally the difference between $\D_{f,\text{Euc}}^{1/p}$ and $\L_p$ is small. Variations in density are controlled by requiring that $f$ is $\Lip$-Lipschitz with respect to geodesic distance, i.e. $|f(x)-f(y)|\leq \Lip \D(x,y)$.  This Lipschitz assumption allows us to establish a higher-order equivalence compared to existing results (e.g. Corollary 9 in \cite{Hwang2016_Shortest}), which we leverage to obtain the local kernel equivalence stated in Corollary \ref{cor:kernel_equivalence}.  The following analysis also establishes explicit dependencies of the equivalence on $d,p,\Lip,\kappa$.

\begin{thm}
	\label{thm:metric_difference_euc}
	Assume $\mathcal{M}\in S(d,\kappa,\epsilon_0)$ and that $f$ is a bounded  $\Lip$-Lipschitz density function on $\mathcal{M}$ with $\fmin>0$.  Let $\epsilon>0$ and let \[\rho = \displaystyle\max_{x\in\mathcal{M}} \rho_{x,\epsilon}, \ C_1 =  \frac{\Lip(\rho^{\frac{p-1}{d}}+1)(p-1)}{\fmin^{1+\frac{p-1}{pd}}pd}, \ C_2 = \frac{\kappa}{\fmin^{\frac{p-1}{pd}}p}.\]  Then for all $x, y \in \mathcal{M}$ such that $\D(x,y)\leq \epsilon_0$ and $\|x-y\| \leq \epsilon$, \[|\L_p(x,y)-\D_{f,\text{Euc}}^{1/p}(x,y)| \leq C_1\epsilon^{1+\frac{1}{p}}+C_2\epsilon^{2+\frac{1}{p}}+O(\epsilon^{3+\frac{1}{p}}).\]

\end{thm}
\begin{proof}
	We first show that $\rho_{x,\epsilon}$ is close to 1. Let $y_1\in B_{\L_p^p}(x,\epsilon(1+\kappa\epsilon^2)/\fmin(x,\epsilon)^{\frac{p-1}{d}}) $ satisfy $f(y_1) = \fmax(x,\epsilon)$ and $y_2 \in B_{\D}(x,\epsilon(1+\kappa\epsilon^2))$ satisfy $f(y_2) = \fmin(x,\epsilon)$ (since these sets are compact, these points must exist). Then by the Lipschitz condition:
	\begin{align*}
	|\rho_{x,\epsilon}-1| &= \frac{|f(y_1) - f(y_2)|}{f(y_2)} 
	\leq \frac{\Lip \D(y_1,y_2)}{f(y_2)} 
	\leq \frac{\Lip \D(x,y_1)+\Lip\D(x,y_2)}{f(y_2)} \, .
	\end{align*}
	Let $\gamma_2(t)$ be a path achieving $\L_p^p(x,y_1)$. Note that
	\begin{align*}
	\frac{\D(x,y_1)}{\fmax(x,\epsilon)^{\frac{p-1}{d}}} &\leq \int_0^1 \frac{1}{f(\gamma_2(t))^{\frac{p-1}{d}}} |\gamma_2'(t)|\ dt = \L_p^p(x,y_1) \leq \frac{\epsilon(1+\kappa\epsilon^2)}{\fmin(x,\epsilon)^{\frac{p-1}{d}}}
	\end{align*}  
	so that $\D(x,y_1) \leq \rho_{x,\epsilon}^{\frac{p-1}{d}} \epsilon(1+\kappa\epsilon^2), \quad \D(x,y_2) \leq \epsilon(1+\kappa\epsilon^2).$
	We thus obtain
	\begin{align}
	\label{equ:rho_close_to_one}
	\rho_{x,\epsilon}&\leq 1+\Lip\left(\frac{\rho_{x,\epsilon}^{\frac{p-1}{d}}+1}{\fmin(x,\epsilon)}\right) \epsilon(1+\kappa\epsilon^2) \, .
	\end{align}
	Letting $C_{x,\epsilon} = \Lip(\rho_{x,\epsilon}^{\frac{p-1}{d}}+1)/\fmin(x,\epsilon)$, Taylor expanding around $\epsilon=0$ and (\ref{equ:rho_close_to_one}) give
	$\rho_{x,\epsilon}^{\frac{p-1}{pd}} \leq (1+C_{x,\epsilon}\epsilon(1+\kappa\epsilon^2))^{\frac{p-1}{pd}} = 1 +C_{x,\epsilon}\frac{(p-1)}{pd}\epsilon + O(\epsilon^3)
	$.
	Applying Lemma \ref{lem:local_equivalence} yields $(\rho_{x,\epsilon})^{-\frac{p-1}{pd}} \D_{f,\text{Euc}}^{1/p}(x,y) \leq \L_p(x,y) \leq (\rho_{x,\epsilon})^{\frac{p-1}{pd}} (1+\kappa \epsilon^2)^{\frac{1}{p}}\D_{f,\text{Euc}}^{1/p}(x,y)$, which gives
	\begin{align*}
	&\frac{\D_{f,\text{Euc}}^{1/p}(x,y)}{\left(1 +C_{x,\epsilon}\frac{(p-1)}{pd}\epsilon + O(\epsilon^3)\right)} \leq \L_p(x,y) \leq \left(1 +C_{x,\epsilon}\frac{(p-1)}{pd}\epsilon +\frac{\kappa}{p} \epsilon^2+O(\epsilon^3)\right)\D_{f,\text{Euc}}^{1/p}(x,y).
	\end{align*}
	Rewriting the above yields:\begin{small} 
	\begin{align*}
	\left(1 -C_{x,\epsilon}\frac{(p-1)}{pd}\epsilon -\frac{\kappa}{p} \epsilon^2+O(\epsilon^3)\right)\L_p(x,y) &\leq \D_{f,\text{Euc}}^{1/p}(x,y) \leq \L_p(x,y)\left(1 +C_{x,\epsilon}\frac{(p-1)}{pd}\epsilon + O(\epsilon^3)\right).
	\end{align*}
	\end{small}
	We thus obtain
	\begin{align*}
	\left| \L_p(x,y) - \D_{f,\text{Euc}}^{1/p}(x,y)\right| &\leq \left(C_{x,\epsilon}\frac{(p-1)}{pd}\epsilon +\frac{\kappa}{p} \epsilon^2+O(\epsilon^3)\right)\L_p(x,y)\\
	&\leq \left(C_{x,\epsilon}\frac{(p-1)}{pd}\epsilon +\frac{\kappa}{p} \epsilon^2+O(\epsilon^3)\right)\frac{\epsilon^{\frac{1}{p}}(1+\kappa\epsilon^2)^{\frac{1}{p}}}{\fmin(x,\epsilon)^{\frac{p-1}{pd}}} \\
	&= \left(\frac{C_{x,\epsilon}}{\fmin(x,\epsilon)^{\frac{p-1}{pd}}}\frac{(p-1)}{pd}\epsilon^{1+{\frac{1}{p}}} +\frac{\kappa}{p\fmin(x,\epsilon)^{\frac{p-1}{pd}}} \epsilon^{2+{\frac{1}{p}}}+O(\epsilon^{3+{\frac{1}{p}}})\right) \, .
	\end{align*}
\end{proof}	

Note the coefficient $C_1$ increases exponentially in $p$; thus the equivalence between $\L_p$ and $\D_{f,\text{Euc}}^{1/p}$ is weaker for large $p$. We also emphasize that in a Euclidean ball of radius $\epsilon$, the metric $\L_p$ scales like $\epsilon^{\frac{1}{p}}$; Theorem \ref{thm:metric_difference_euc} thus guarantees that the relative error of approximating $\L_p$ with $\D_{f,\text{Euc}}^{1/p}$ is $O(\epsilon)$. 

When $\L_p$ is locally well-approximated by $\D_{f,\text{Euc}}^{1/p}$, the kernels constructed from these two metrics are also locally similar. The following Corollary leverages the error term in Theorem \ref{thm:metric_difference_euc} to make this precise for Gaussian kernels. It is a direct consequence of Theorem \ref{thm:metric_difference_euc} and Taylor expanding the Gaussian kernel, and its proof is given in Appendix~\ref{app:Proof_of_Cor}. Let $h_a(x)=\exp(-x^{2a})$ so that $h_1\left(\frac{\delta(\cdot,\cdot)}{\epsilon}\right)$ is the Gaussian kernel with metric $\delta(\cdot,\cdot)$ and scaling parameter $\epsilon>0$.  Note $h_1\left(\frac{\L_p}{\epsilon^{1/p}}\right) = h_{\frac{1}{p}}\left(\frac{\L_p^p}{\epsilon}\right)$.

\begin{cor}
	\label{cor:kernel_equivalence}
	Under the assumptions and notation of Theorem \ref{thm:metric_difference_euc}, for $\tilde{C}_i = C_i/\fmin^{\frac{p-1}{pd}}$,
	\begin{align*}
	 \frac{\left|h_{\frac{1}{p}}\left(\L_p^p(x,y)/\epsilon\right) - h_{\frac{1}{p}}\left(\D_{f,\text{Euc}}(x,y)/\epsilon\right) \right|}{h_{\frac{1}{p}}\left(\L_p^p(x,y)/\epsilon\right)} &\leq \tilde{C}_1 \epsilon + \left(\tilde{C}_2 +\frac{1}{2}\tilde{C}_1^2\right)\epsilon^2 + O(\epsilon^3) \, .
	\end{align*}
\end{cor}
When $p-1$ is not too large relative to $d$, 
a kernel constructed with $\L_p$ is locally well-approximated by a kernel constructed with $\D_{f,\text{Euc}}^{1/p}$.  Thus, in a Euclidean ball of radius $\epsilon$, we may think of the Gaussian $\L_p$ kernel as:
\begin{align*}
h_1\left(\frac{\L_p(x_i,x_j)}{\epsilon^{1/p}}\right)  &
\approx h_\frac{1}{p}\left(\frac{\|x_i-x_j\|}{\epsilon (f(x_i)f(x_j))^{\frac{p-1}{2d}}}\right).
\end{align*}
Density plays a different role in this kernel compared with those of Section \ref{subsec:RoleDensityOtherKernels}. This kernel strengthens connections in high density regions and weakens them in low density regions. 

We note that the $\frac{1}{p}$-power in Definition \ref{defn:CPWSPD} has a large impact, in that $\L_p$-based and $\L_p^p$-based kernels have very different properties. More specifically, $h_1(\L_p^p/\epsilon)$ is a local kernel as defined in \cite{berry2016local}, so it is sufficient to analyze the kernel locally. However $h_1(\L_p/\epsilon^{1/p})$ is a non-local kernel, so that non-trivial connections between distant points are possible. The analysis in this Section thus establishes the global equivalence of  $h_1(\L_p^p/\epsilon)$ and $h_1(\D_{f,\text{Euc}}/\epsilon)$ (when $p$ is not too large relative to $d$) but only the local equivalence of $h_{\frac{1}{p}}(\L_p^p/\epsilon)$ and $h_{\frac{1}{p}}(\D_{f,\text{Euc}}/\epsilon)$.


\subsection{The Role of $p$: Examples}
This subsection illustrates the useful properties of PWSPDs and the role of $p$ on three synthetic data sets in $\mathbb{R}^{2}$: (1) \emph{Two Rings} data, consisting of two non-convex clusters that are well-separated by a low-density region; (2) \emph{Long Bottleneck} data, consisting of two isotropic clusters each with a density gap connected by a long, thin bottleneck; (3) \emph{Short Bottleneck} data, where two elongated clusters are connected by a short bottleneck.  The data sets are shown in Figures \ref{fig:TwoRings},  \ref{fig:LongBottleneck}, and \ref{fig:ShortBottleneck}, respectively.  We also show the PWSPD spectral embedding (denoted PWSPD SE) for various $p$, computed from a symmetric normalized Laplacian constructed with PWSPD. The scaling parameter $\epsilon$ for each data set is chosen as the $15^{\text{th}}$ percentile of pairwise PWSPD distances.

Different aspects of the data are emphasized in the low-dimensional PWSPD embedding as $p$ varies.  Indeed, in Figure \ref{fig:TwoRings}, we see the PWSPD embedding separates the rings for large $p$ but not for small $p$.  In Figure \ref{fig:LongBottleneck}, we see separation across the bottleneck for $p$ small, while for $p$ large there is separation with respect to the density gradients that appear in the two bells of the dumbbell.  Interestingly, separation with respect to both density and geometry is observed for $p=2$ (see Figure \ref{fig:LongBottleneck_PWSPD_p2}).  In Figure \ref{fig:ShortBottleneck}, the clusters are both elongated and lack robust density separation, but the PWSPD embedding well-separates the two clusters for moderate $p$.  In general, $p$ close to 1 emphasizes the geometry of the data, large $p$ emphasizes the density structure of the data, and moderate $p$ defines a metric balancing these two considerations. 
 
 \begin{figure}[htbp!]
	\centering
	\begin{subfigure}[t]{0.235\textwidth}
		\centering
		\includegraphics[width=\textwidth]{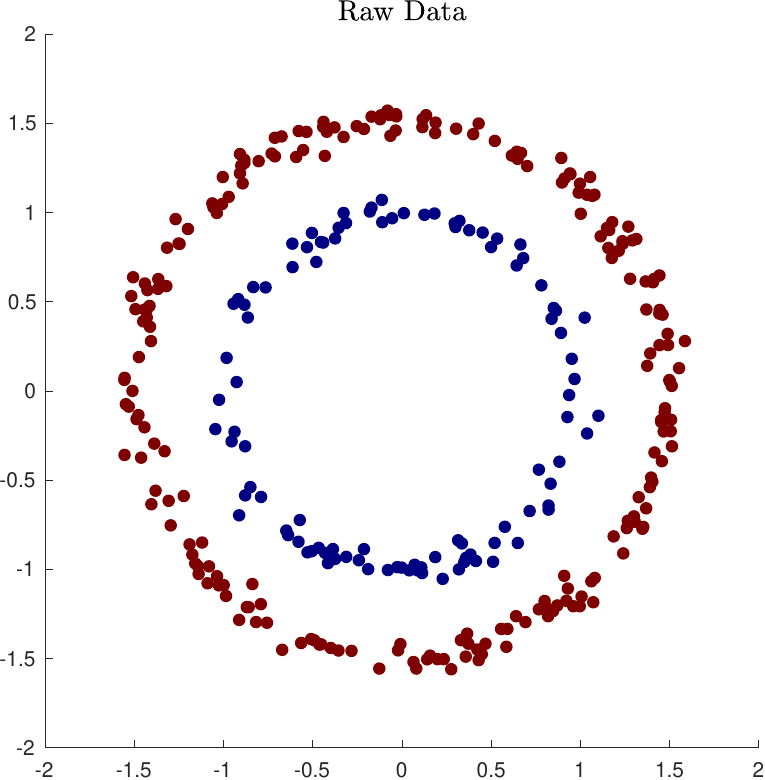}
		\caption{\tiny{Two Rings}}
	\end{subfigure}
	\hspace{.05cm}
	\begin{subfigure}[t]{0.235\textwidth}
		\centering
		\includegraphics[width=\textwidth]{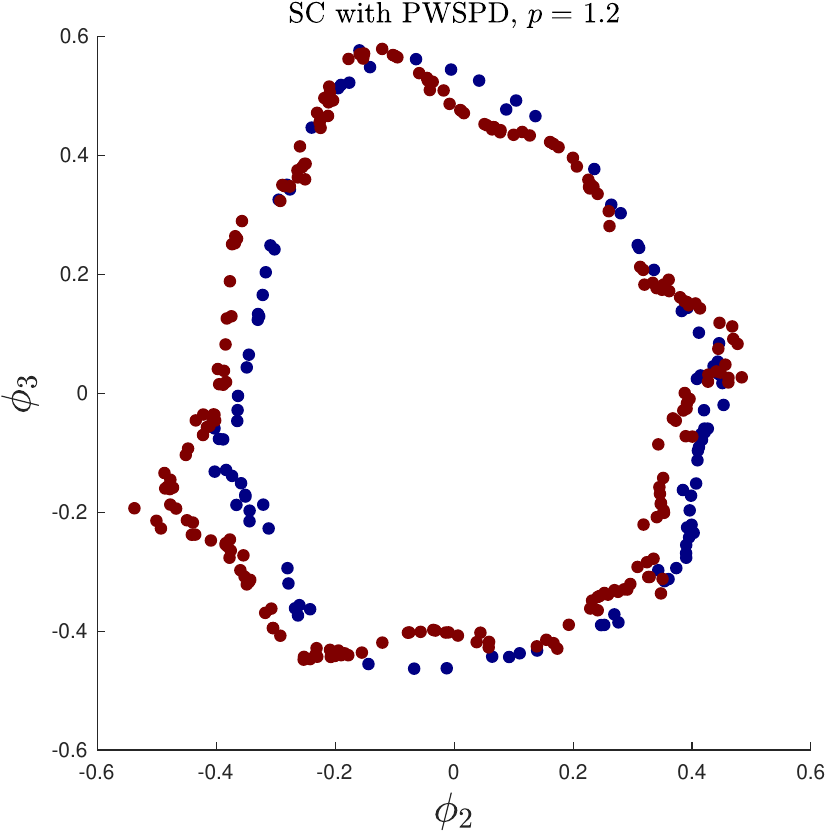}
		\caption{\tiny{PWSPD SE, $p=1.2$}}
	\end{subfigure}
	\hspace{.05cm}
	\begin{subfigure}[t]{0.235\textwidth}
		\centering
		\includegraphics[width=\textwidth]{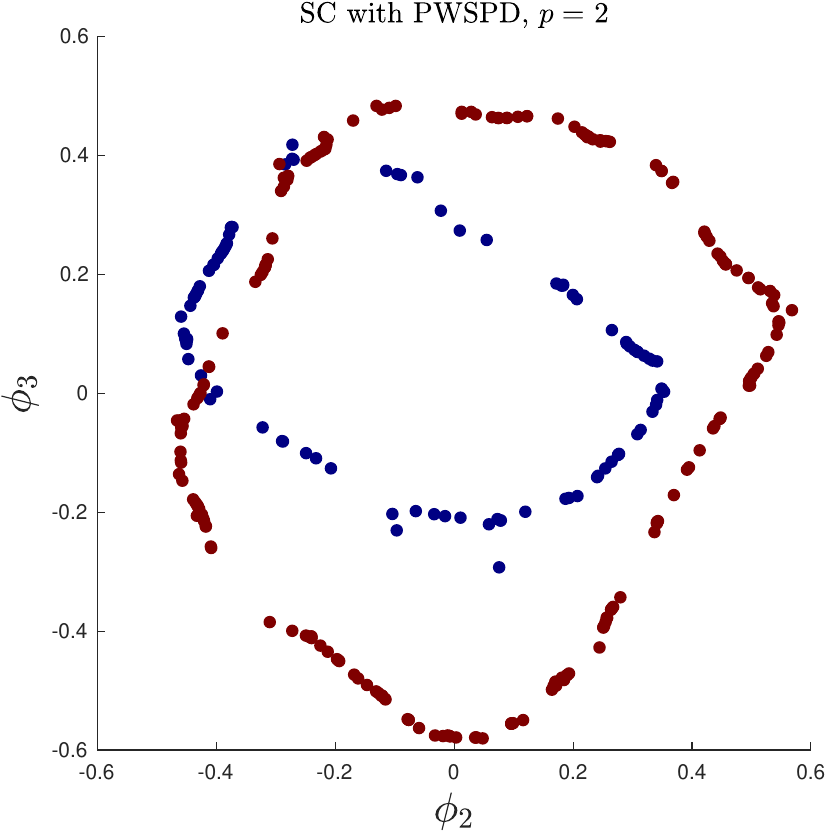}
		\caption{\tiny{PWSPD SE, $p=2$}}
	\end{subfigure}
	\hspace{.05cm}
	\begin{subfigure}[t]{0.235\textwidth}
		\centering
		\includegraphics[width=\textwidth]{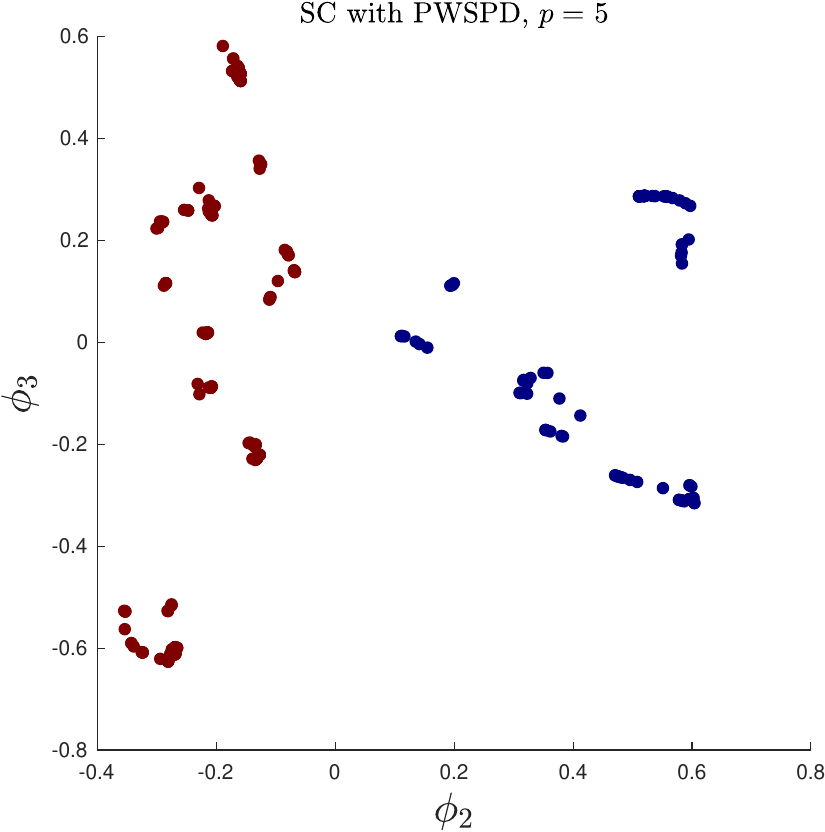}
		\caption{\tiny{PWSPD SE, $p=5$}}
		\label{fig:TwoRings_PWSPD_Embedding_p=5}
	\end{subfigure}
	\begin{subfigure}[t]{0.235\textwidth}
		\centering
		\includegraphics[width=\textwidth]{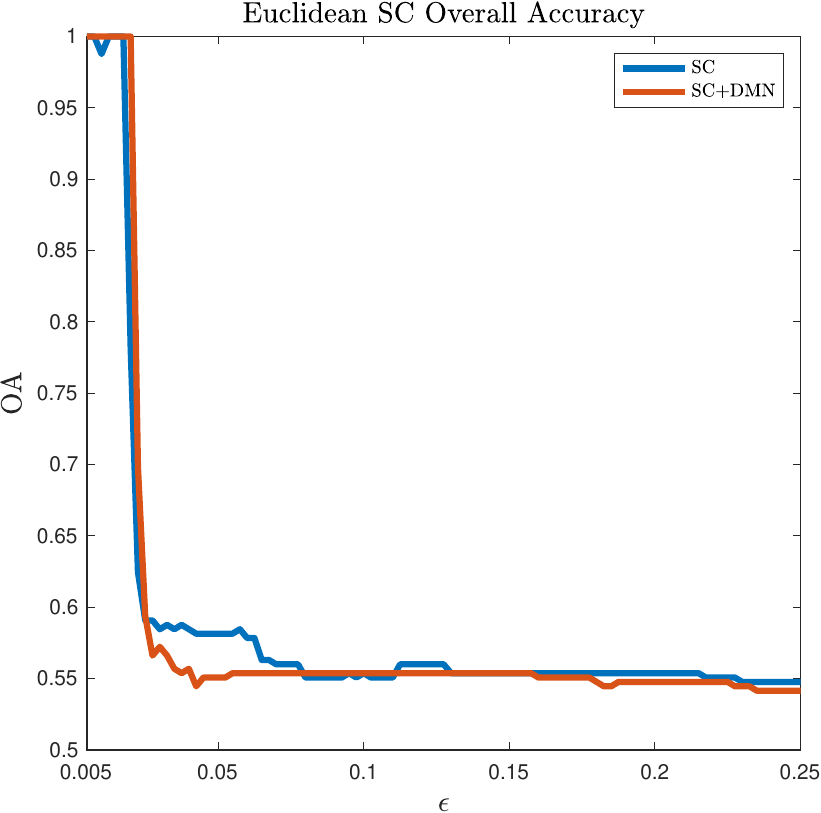}
		\subcaption{\tiny{OA, Euc. SC}}
		\label{fig:TwoRingsAccuracyPlots_OA_SC}
	\end{subfigure}
	\hspace{.05cm}
	\begin{subfigure}[t]{0.235\textwidth}
		\centering
		\includegraphics[width=1.1\textwidth]{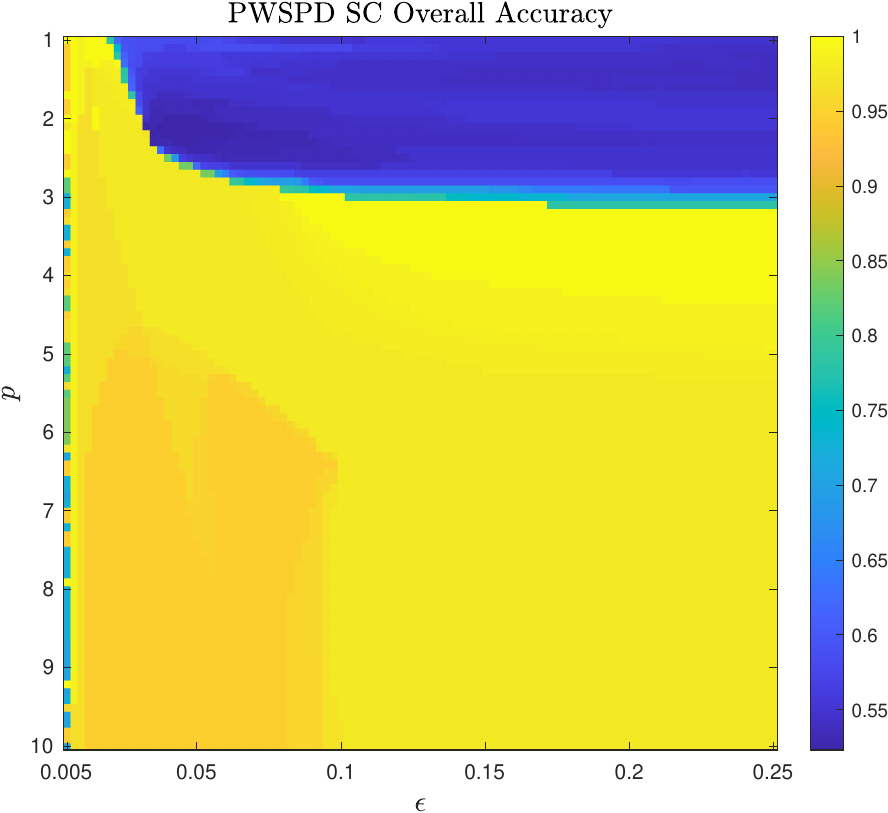}
		\subcaption{\tiny{OA, PWSPD SC}}
		\label{fig:TwoRingsAccuracyPlots_OA_SC_PWSPD}
	\end{subfigure}	
	\hspace{.05cm}
	\begin{subfigure}[t]{0.235\textwidth}
		\centering
		\includegraphics[width=\textwidth]{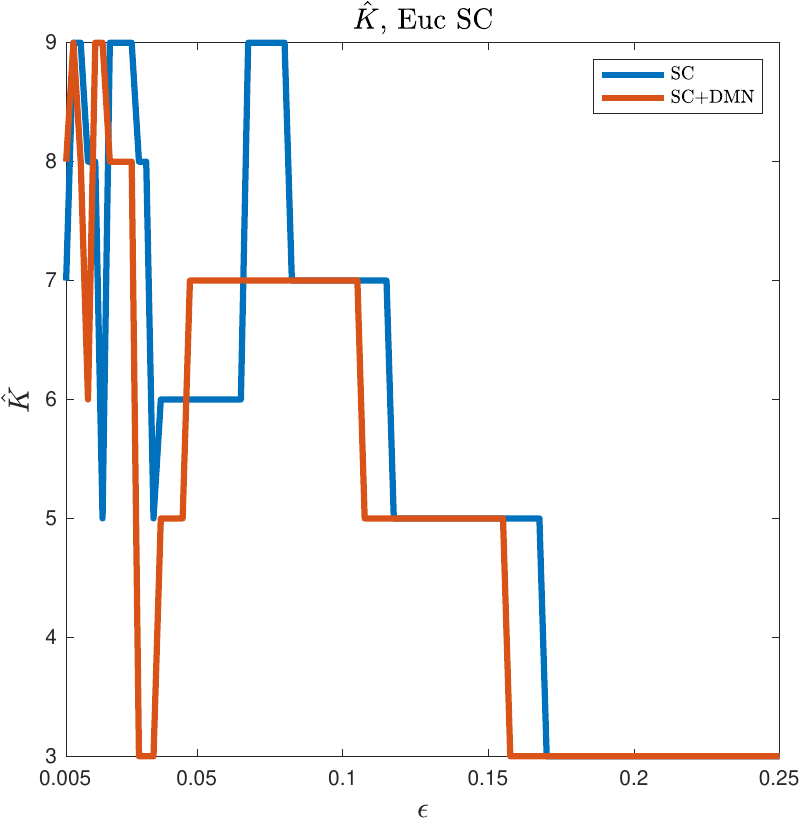}
		\subcaption{\tiny{$\hat{K}$, Euc. SC}}
		\label{fig:TwoRingsAccuracyPlots_K_SC}
	\end{subfigure}
	\hspace{.05cm}
	\begin{subfigure}[t]{0.235\textwidth}
		\centering
		\includegraphics[width=\textwidth]{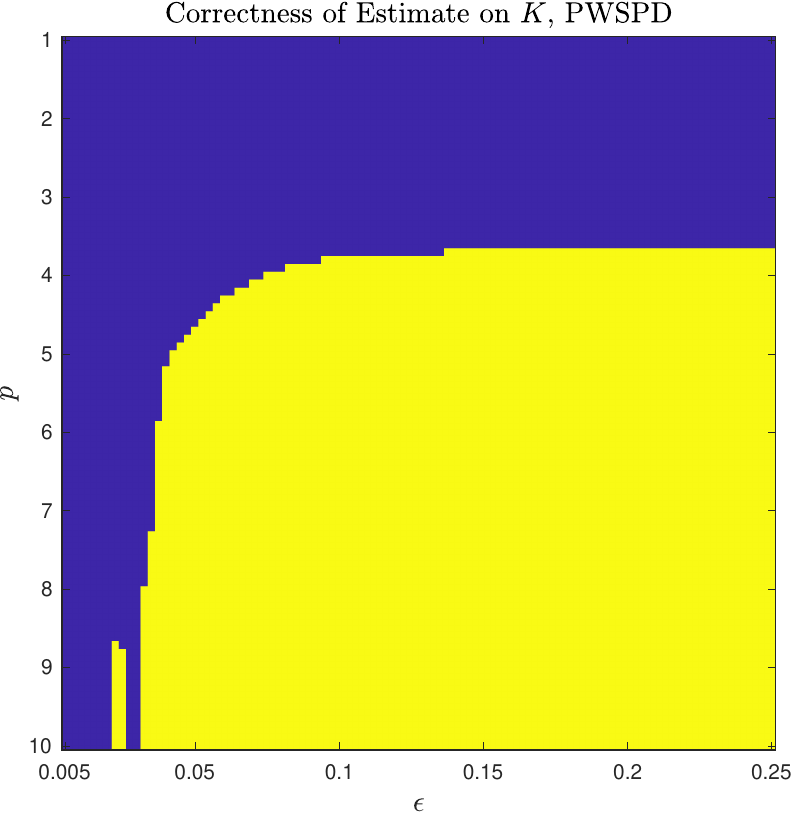}
		\subcaption{\tiny{$\hat{K}$, PWSPD SC (binarized)}}
		\label{fig:TwoRingsAccuracyPlots_K_SC_PWSPD_Binarized}
	\end{subfigure}
		\caption{\emph{Two Rings dataset}.
		Because the underlying cluster structure is density-driven, the PWSPD SE separates the clusters for large $p$ (Figure \ref{fig:TwoRings_PWSPD_Embedding_p=5}).  While taking $\epsilon$ small in Euclidean spectral clustering can allow for good clustering accuracy (see Figure \ref{fig:TwoRingsAccuracyPlots_OA_SC}), the range is narrow and does not permit accurate estimation of $K$ via the eigengap (see Figure \ref{fig:TwoRingsAccuracyPlots_K_SC}).  On the other hand, PWSPD consistently clusters well and correctly captures $K=2$ for a wide range of $(\epsilon,p)$ pairs (see Figures \ref{fig:TwoRingsAccuracyPlots_K_SC}, \ref{fig:TwoRingsAccuracyPlots_K_SC_PWSPD_Binarized}).  Generally, PWSPD allows for fully unsupervised clustering as long as $p$ is sufficiently large and $\epsilon$ not too small.}
	\label{fig:TwoRings}
\end{figure}

\begin{figure}[htbp!]
	\centering
	\begin{subfigure}[t]{0.185\textwidth}
		\centering
		\includegraphics[width=\textwidth]{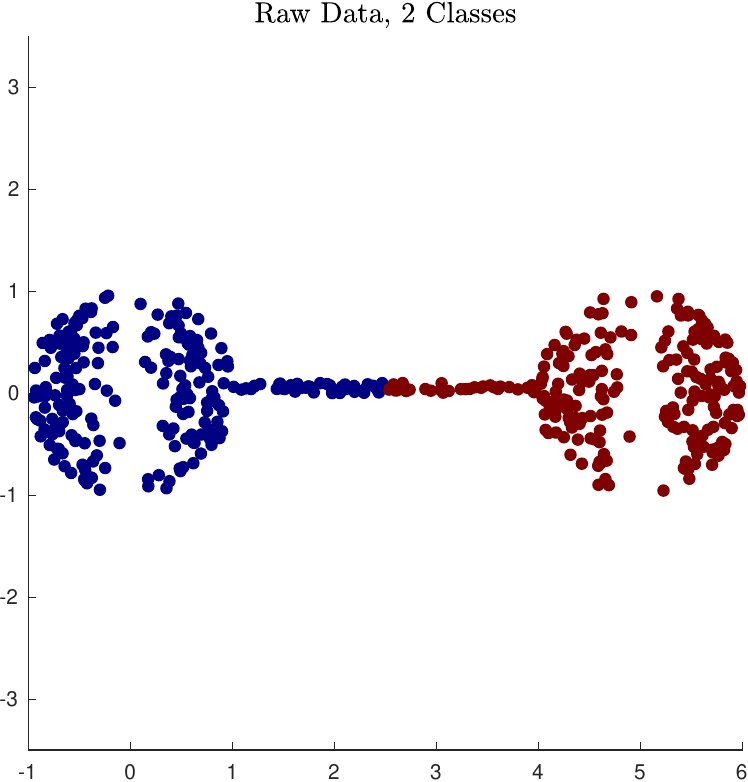}
		\subcaption{\tiny{L. Bottleneck, $K=2$}}
		\label{fig:LongBottleneck_K_2}
	\end{subfigure}
	\hspace{.05cm}
	\begin{subfigure}[t]{0.185\textwidth}
		\centering
		\includegraphics[width=\textwidth]{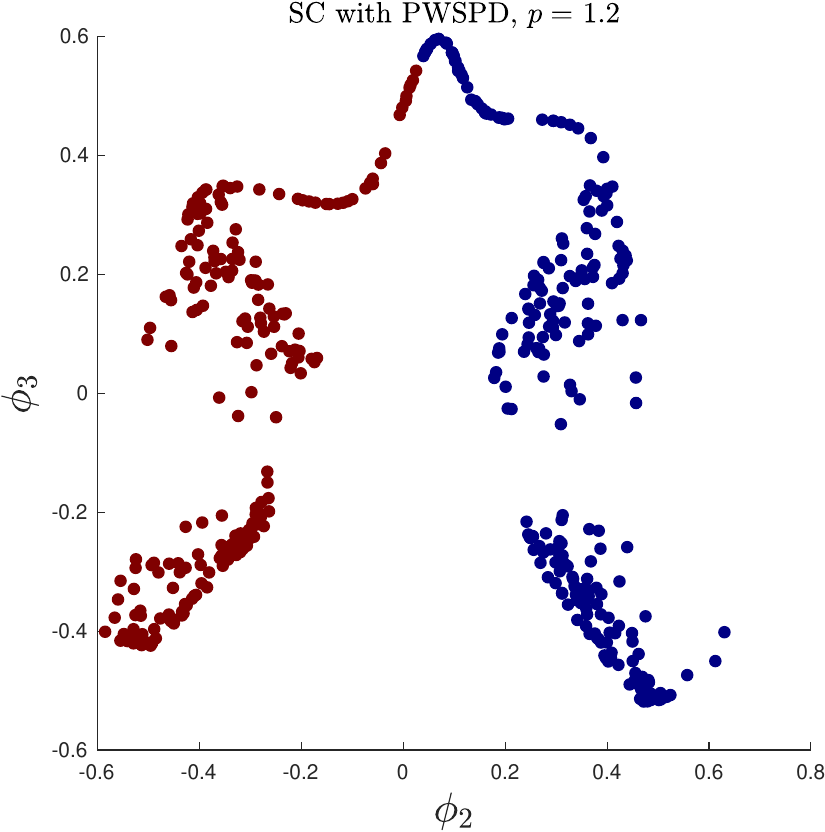}
		\subcaption{\tiny{PWSPD SE., $p=1.2$}}
		\label{fig:LongBottleneck_PWSPD_p12}
	\end{subfigure}
	\hspace{.05cm}
	\begin{subfigure}[t]{0.185\textwidth}
		\centering
		\includegraphics[width=1.1\textwidth]{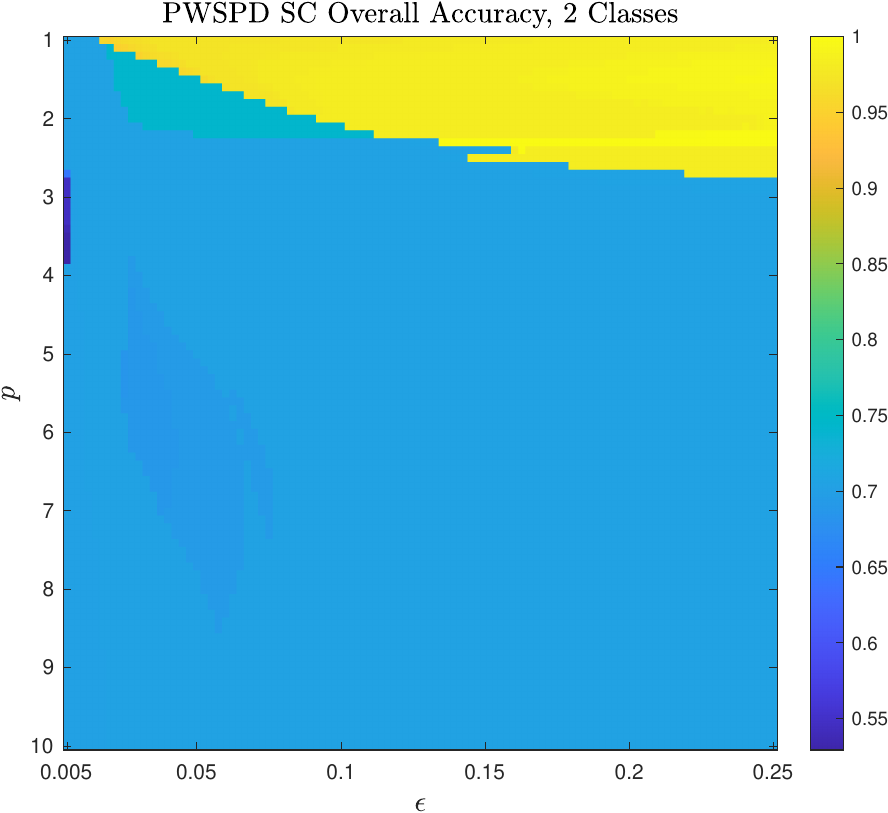}
		\caption{\tiny{OA, PWSPD SC}}
		\label{fig:LongBottleneckAccuracyPlot_SC_PWSPD_OA_2}
	\end{subfigure}
	\hspace{.05cm}
	\begin{subfigure}[t]{0.185\textwidth}
		\centering
		\includegraphics[width=\textwidth]{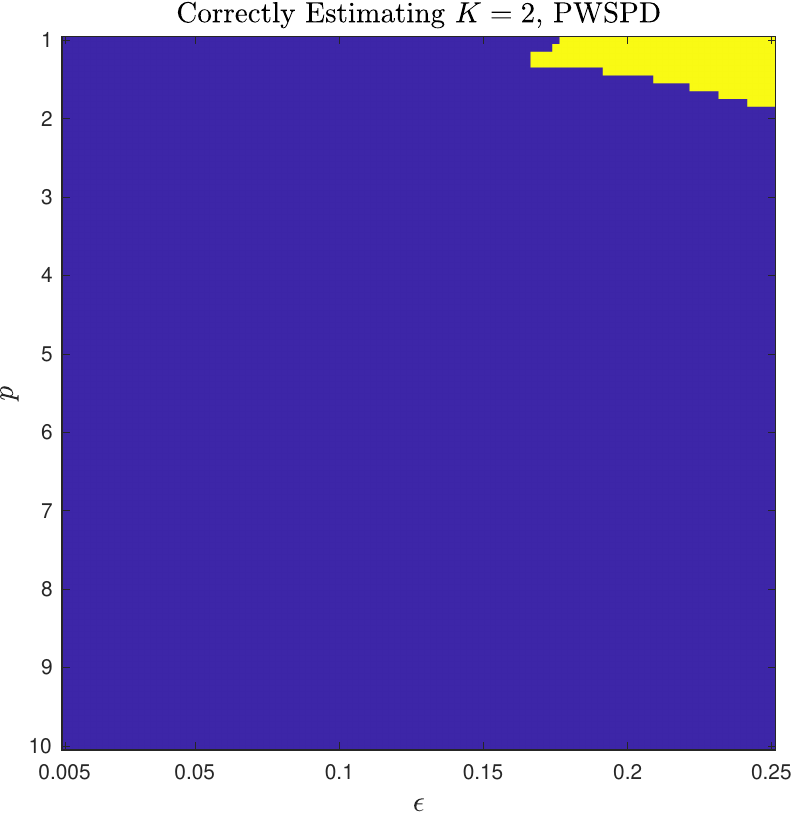}
		\caption{\tiny{$\hat{K}$, PWSPD SC}}
		\label{fig:LongBottleneckAccuracyPlot_SC_PWSPD_K_2}
	\end{subfigure}
	\hspace{.05cm}
	\begin{subfigure}[t]{0.185\textwidth}
		\centering
		\includegraphics[width=\textwidth]{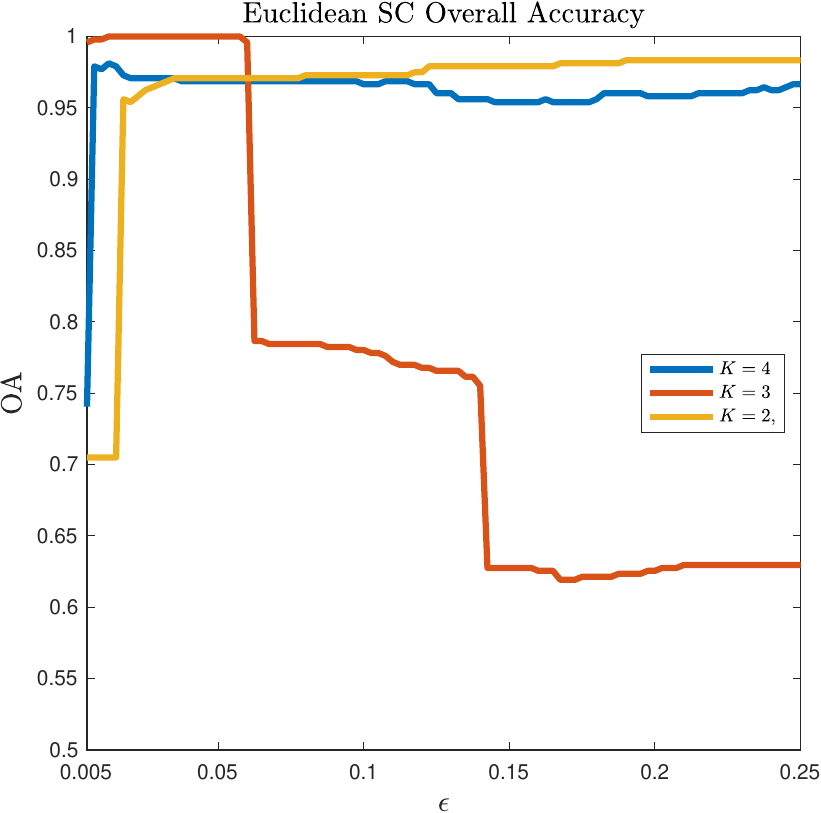}
		\caption{\tiny{OA, Euc. SC}}
		\label{fig:LongBottleneckAccuracyPlot_SC_OA}
	\end{subfigure}
	\begin{subfigure}[t]{0.185\textwidth}
		\centering
		\includegraphics[width=\textwidth]{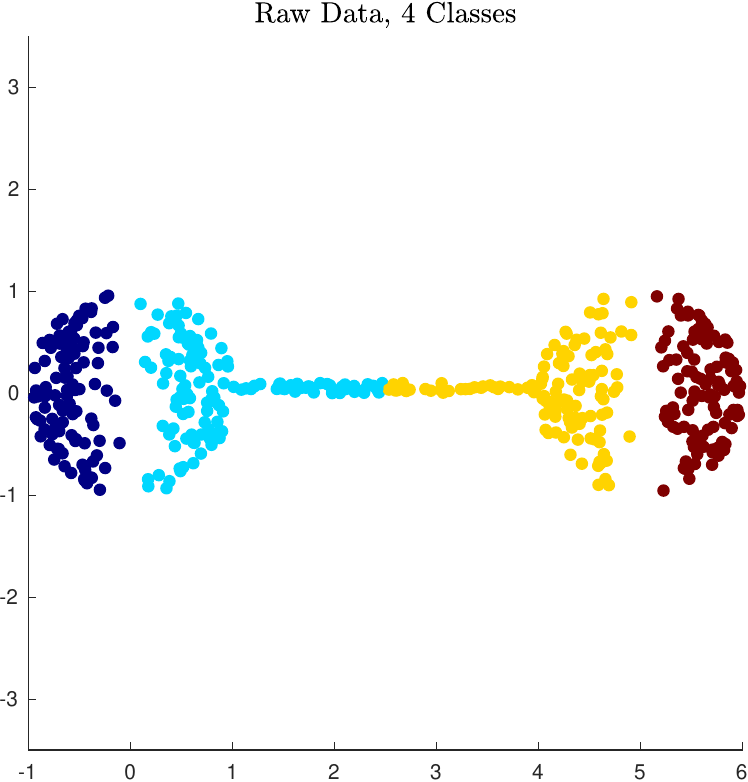}
		\subcaption{\tiny{L. Bottleneck, $K=4$}}
		\label{fig:LongBottleneck_K_4}
	\end{subfigure}
	\hspace{.05cm}
	\begin{subfigure}[t]{0.185\textwidth}
		\centering
		\includegraphics[width=\textwidth]{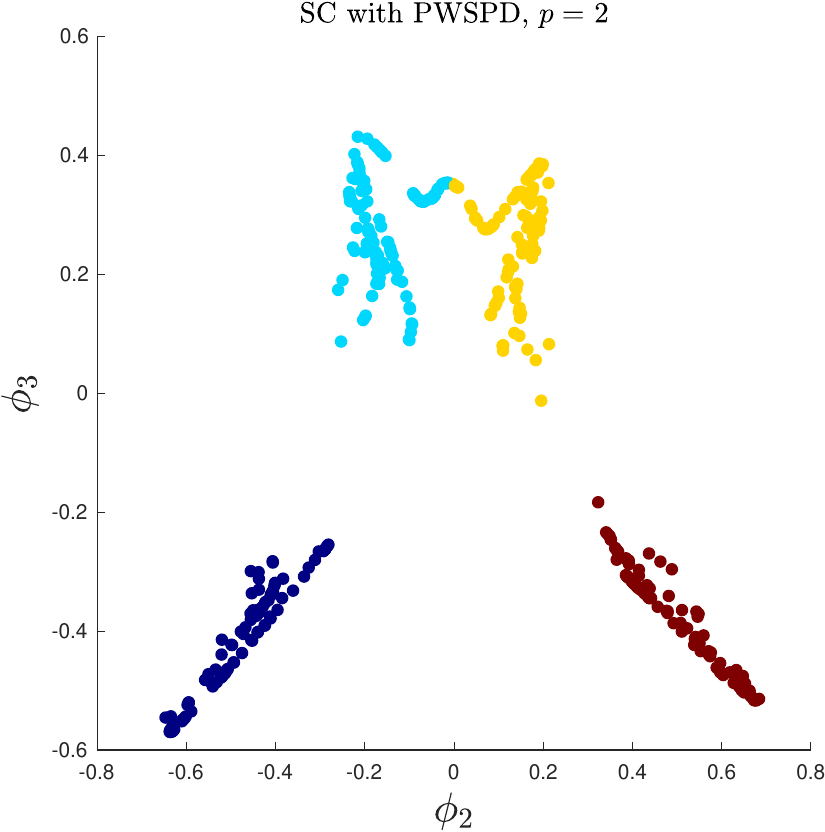}
		\subcaption{\tiny{PWSPD SE, $p=2$}}
		\label{fig:LongBottleneck_PWSPD_p2}
	\end{subfigure}
	\hspace{.05cm}
	\begin{subfigure}[t]{0.185\textwidth}
		\centering
		\includegraphics[width=1.1\textwidth]{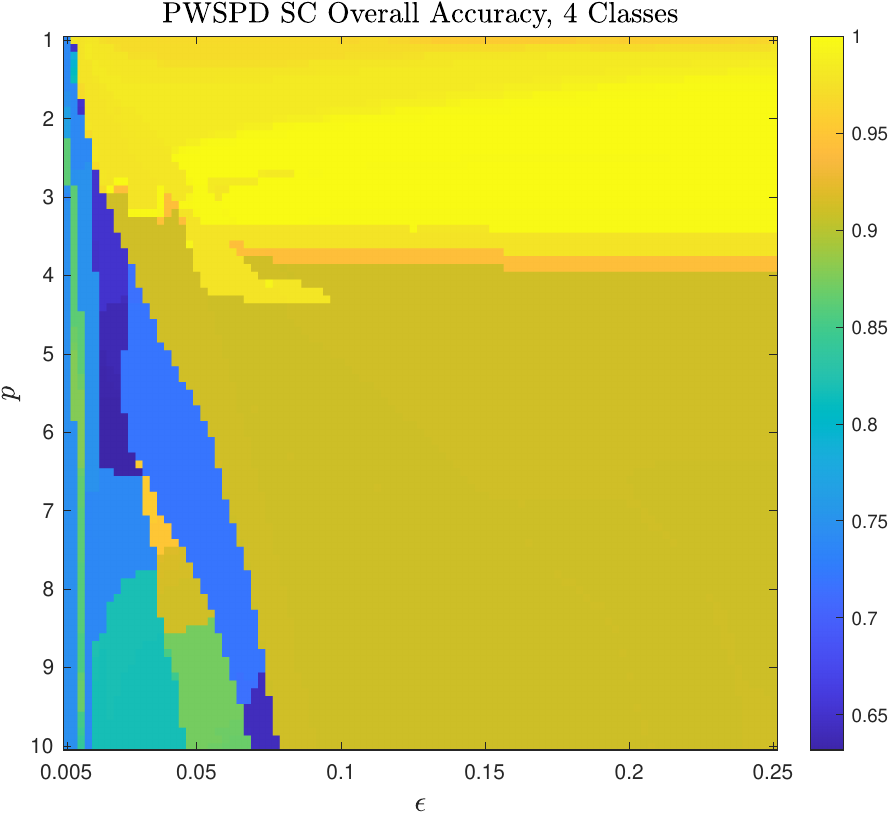}
		\caption{\tiny{OA, PWSPD SC}}
		\label{fig:LongBottleneckAccuracyPlot_SC_PWSPD_OA_4}
	\end{subfigure}
	\hspace{.05cm}
	\begin{subfigure}[t]{0.185\textwidth}
		\centering
		\includegraphics[width=\textwidth]{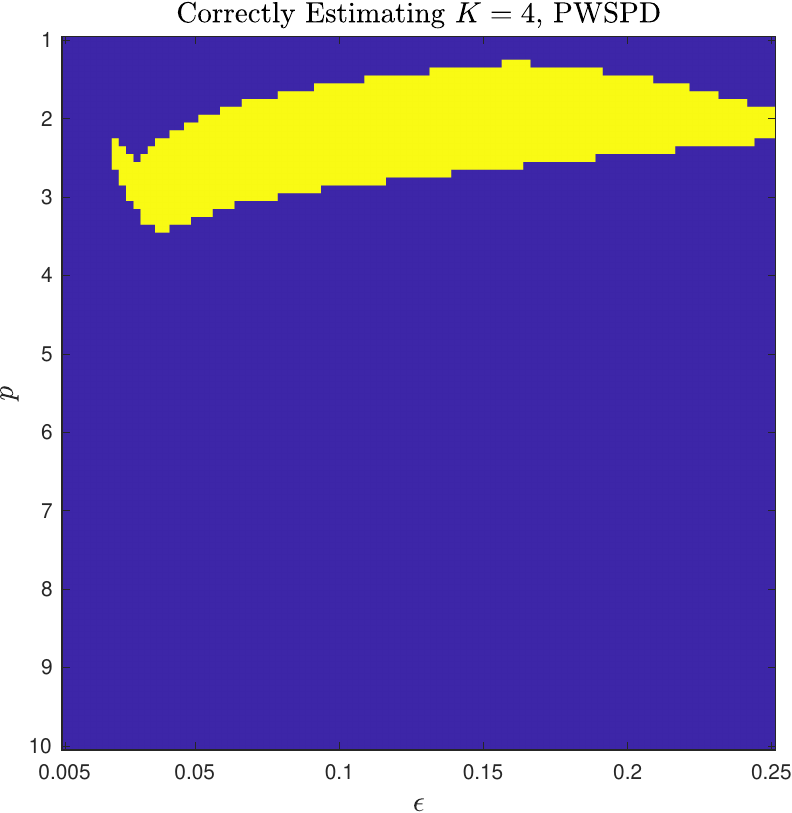}
		\caption{\tiny{$\hat{K}$, PWSPD SC}}
		\label{fig:LongBottleneckAccuracyPlot_SC_PWSPD_K_4}
	\end{subfigure}
	\hspace{.05cm}
	\begin{subfigure}[t]{0.185\textwidth}
		\centering
		\includegraphics[width=\textwidth]{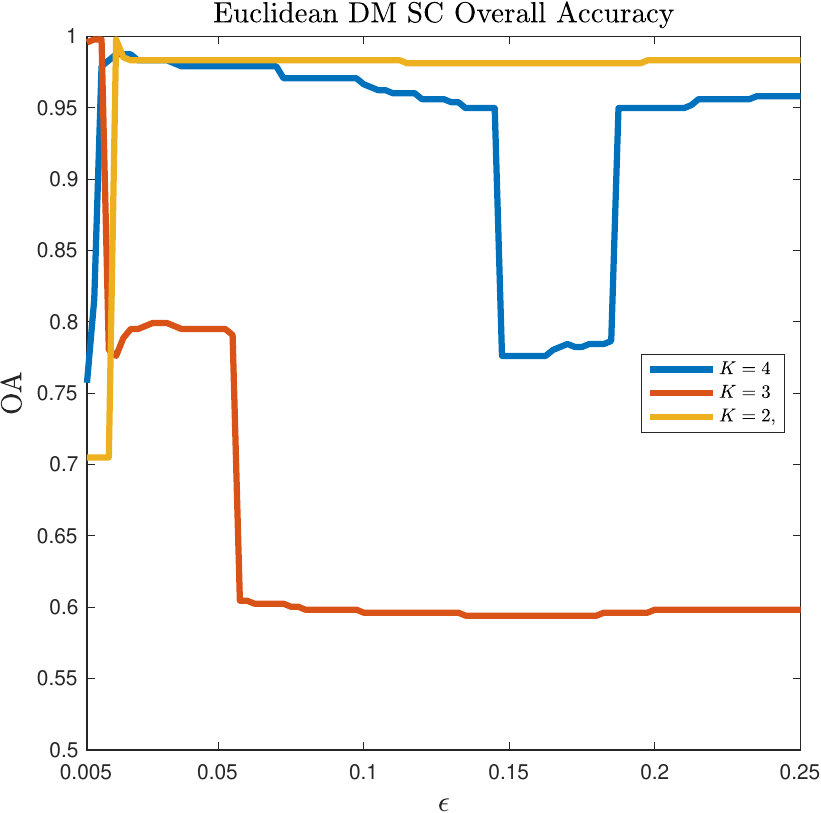}
		\caption{\tiny{OA, Euc SC+DMN}}
		\label{fig:LongBottleneckAccuracyPlot_SC_DMN_OA}
	\end{subfigure}	
	\begin{subfigure}[t]{0.185\textwidth}
		\centering
		\includegraphics[width=\textwidth]{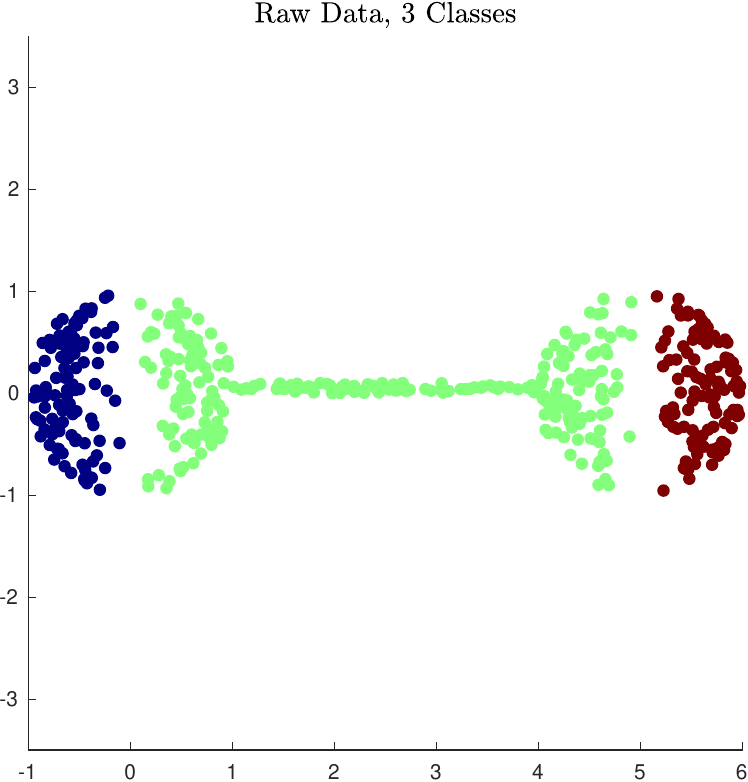}
		\subcaption{\tiny{L. Bottleneck, $K=3$}}
		\label{fig:LongBottleneck_K_3}
	\end{subfigure}
	\hspace{.05cm}
	\begin{subfigure}[t]{0.185\textwidth}
		\centering
		\includegraphics[width=\textwidth]{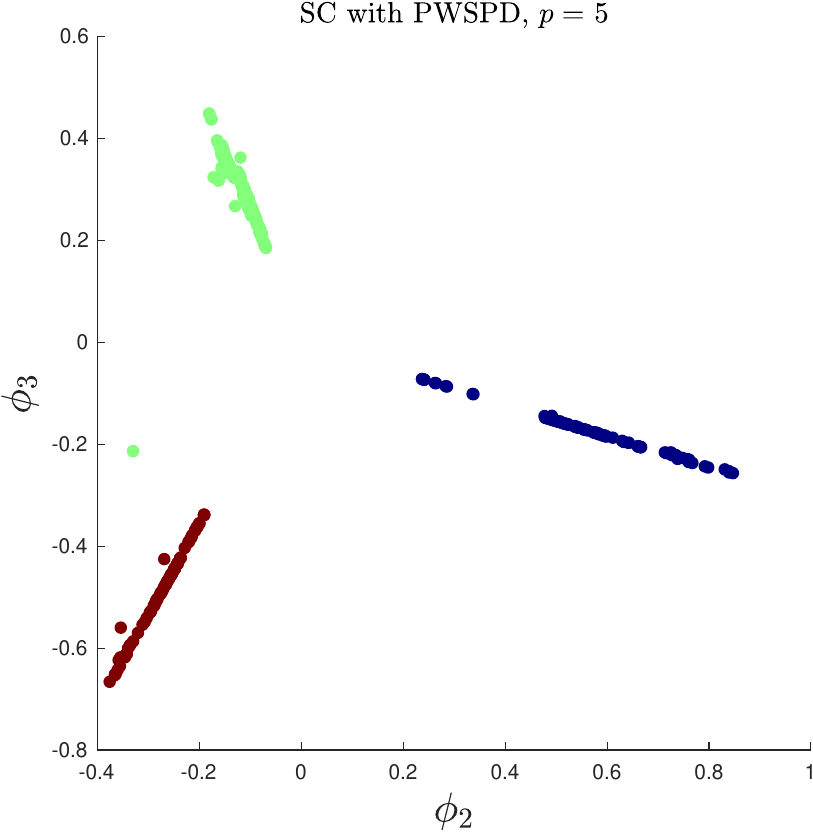}
		\subcaption{\tiny{PWSPD SE, $p=5$}}
		\label{fig:LongBottleneck_PWSPD_p5}
	\end{subfigure}
	\hspace{.05cm}
	\begin{subfigure}[t]{0.185\textwidth}
		\centering
		\includegraphics[width=1.1\textwidth]{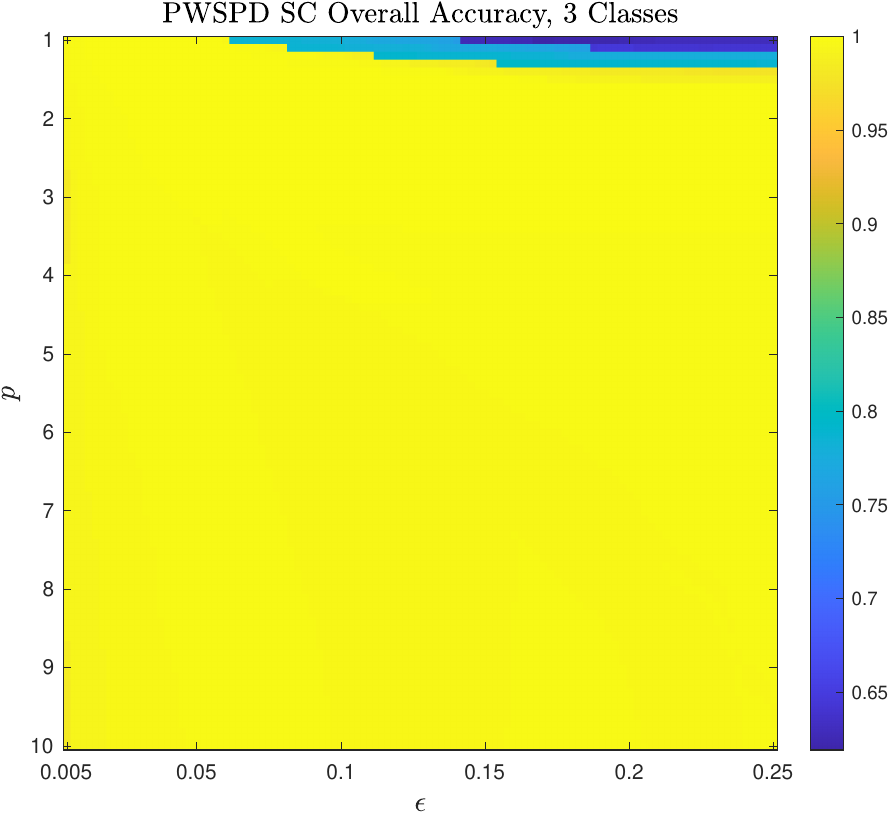}
		\caption{\tiny{OA, PWSPD SC}}
		\label{fig:LongBottleneckAccuracyPlot_SC_PWSPD_OA_3}
	\end{subfigure}
	\hspace{.05cm}
	\begin{subfigure}[t]{0.185\textwidth}
		\centering
		\includegraphics[width=\textwidth]{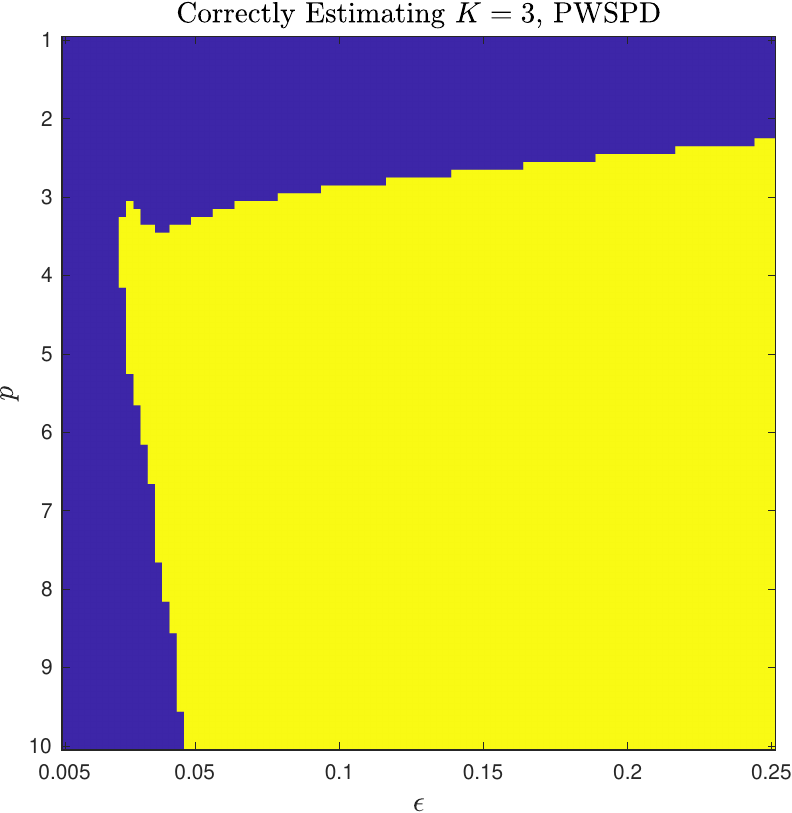}
		\caption{\tiny{$\hat{K}$, PWSPD SC}}
		\label{fig:LongBottleneckAccuracyPlot_SC_PWSPD_K_3}
	\end{subfigure}
	\hspace{.05cm}
	\begin{subfigure}[t]{0.185\textwidth}
		\centering
		\includegraphics[width=\textwidth]{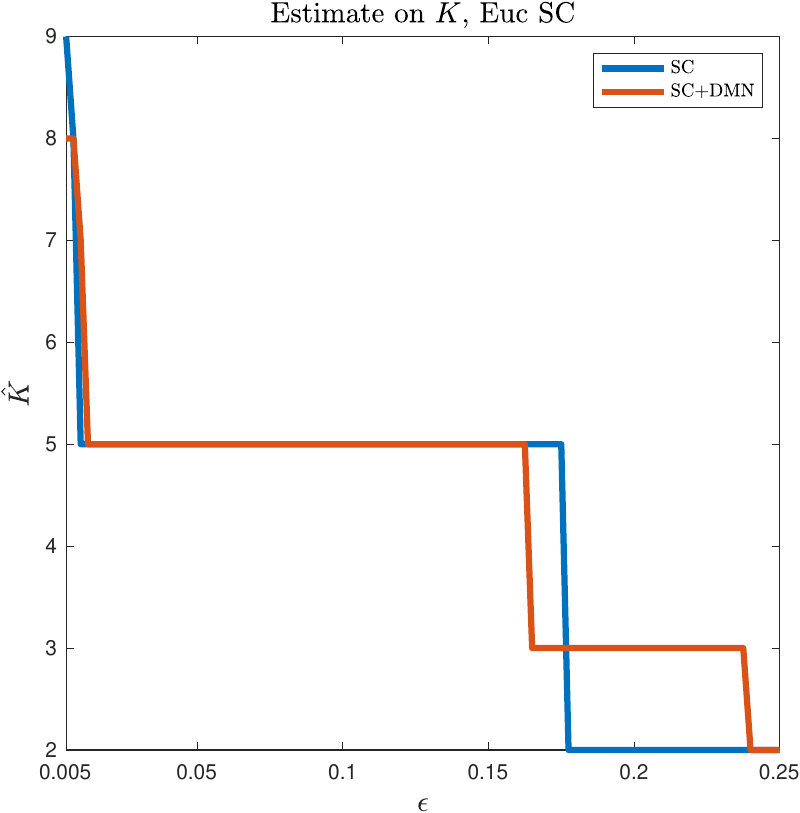}
		\caption{\tiny{$\hat{K}$, Euc. SC} }
		\label{fig:LongBottleneckAccuracyPlot_SC_K}
	\end{subfigure}
		\caption{\emph{Long Bottleneck dataset}.
		Different latent cluster structures exist in this data, driven by geometry (Figure \ref{fig:LongBottleneck_K_2}, $K=2$), density (Figure \ref{fig:LongBottleneck_K_3}, $K=3$), and a combination of geometry and density (Figure \ref{fig:LongBottleneck_K_4}, $K=4$).  When varying $p$, the PWSPD SE separates by geometry (Figure \ref{fig:LongBottleneck_PWSPD_p12}) for $p$ near 1, before separating by density for $p\gg1$ (Figure \ref{fig:LongBottleneck_PWSPD_p5}).  Given the correct choice of $\epsilon$ and a priori knowledge of $K$, any of the three natural clusterings can be learned by Euclidean SC (Figure \ref{fig:LongBottleneckAccuracyPlot_SC_OA}, \ref{fig:LongBottleneckAccuracyPlot_SC_DMN_OA}).  However, in the Euclidean SC case, correct estimation of $K$ fails to coincide with parameters that give good clustering results (Figure \ref{fig:LongBottleneckAccuracyPlot_SC_K}).  On the other hand, PWSPD SC is able to correctly estimate each of $K=2,3,4$ for some choice of $(\epsilon,p)$ parameters in the same region that such parameters yield high clustering accuracy (Figures \ref{fig:LongBottleneckAccuracyPlot_SC_PWSPD_OA_2}, \ref{fig:LongBottleneckAccuracyPlot_SC_PWSPD_K_2} for $K=2$; Figures \ref{fig:LongBottleneckAccuracyPlot_SC_PWSPD_OA_3}, \ref{fig:LongBottleneckAccuracyPlot_SC_PWSPD_K_3} for $K=3$; Figures \ref{fig:LongBottleneckAccuracyPlot_SC_PWSPD_OA_4}, \ref{fig:LongBottleneckAccuracyPlot_SC_PWSPD_K_4} for $K=4$).}
	\label{fig:LongBottleneck}
\end{figure}

\begin{figure}[h]
	\centering
	\begin{subfigure}[t]{0.235\textwidth}
		\centering
		\includegraphics[width=\textwidth]{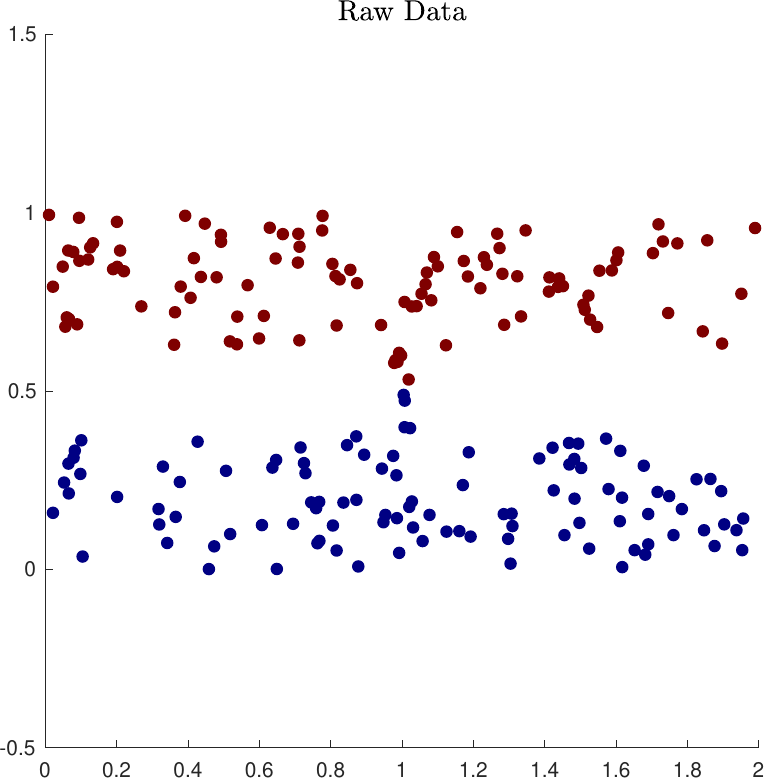}
		\subcaption{\tiny{Short Bottleneck}}
	\end{subfigure}
	\hspace{.05cm}
	\begin{subfigure}[t]{0.235\textwidth}
		\centering
		\includegraphics[width=\textwidth]{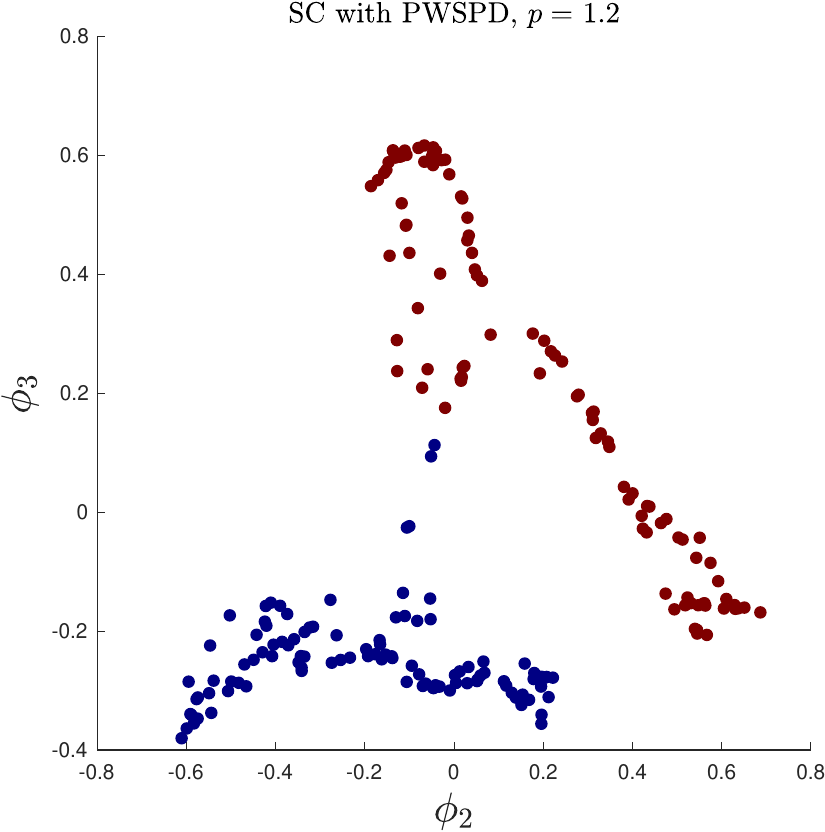}
		\caption{\tiny{PWSPD SE, $p=1.2$}}
	\end{subfigure}
	\hspace{.05cm}
	\begin{subfigure}[t]{0.235\textwidth}
		\centering
		\includegraphics[width=\textwidth]{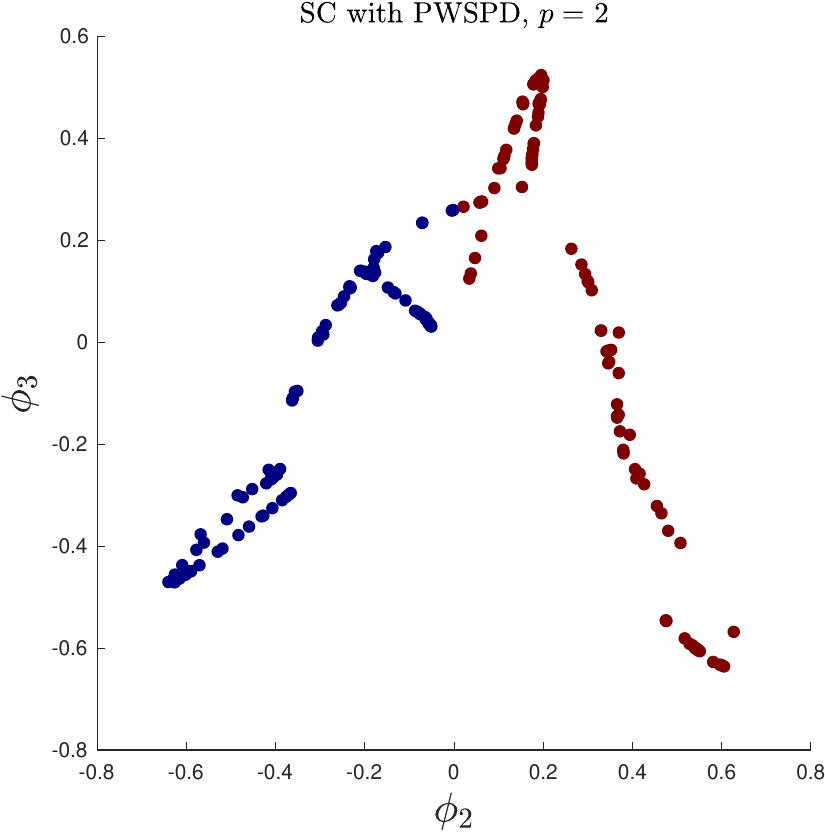}
		\caption{\tiny{PWSPD SE, $p=2$}}
		\label{fig:short_bottleneck_p2}
	\end{subfigure}
	\hspace{.05cm}
	\begin{subfigure}[t]{0.235\textwidth}
		\centering
		\includegraphics[width=\textwidth]{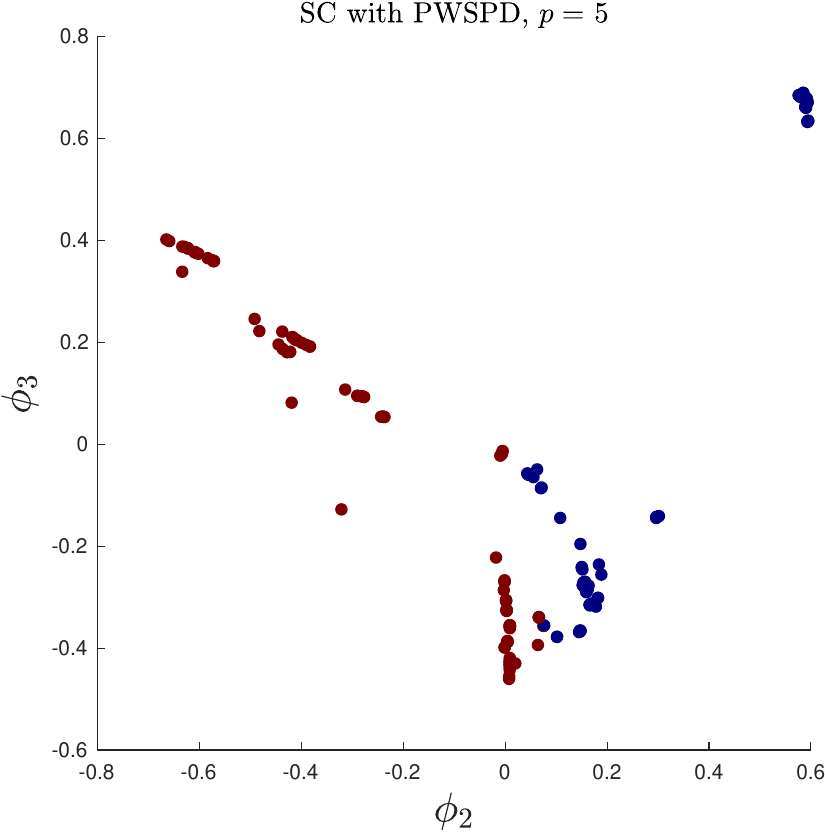}
		\caption{\tiny{PWSPD SE, $p=5$}}
	\end{subfigure}
	\begin{subfigure}[t]{0.235\textwidth}
		\centering
		\includegraphics[width=\textwidth]{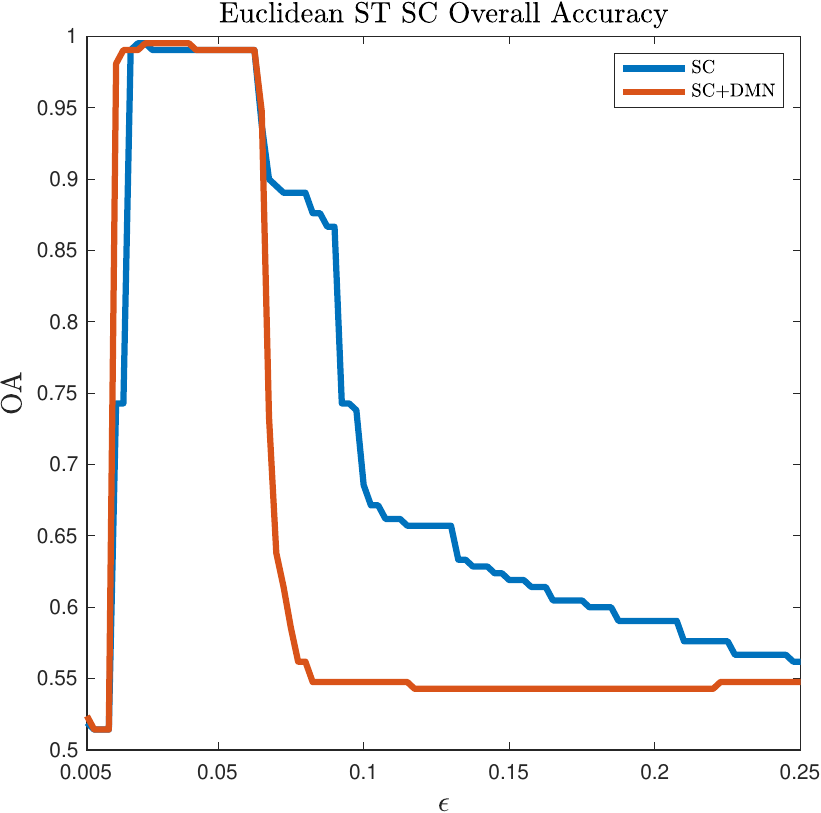}
		\subcaption{\tiny{OA, Euc. SC}}
		\label{fig:ShortBottleneckAccuracyPlots_OA_EucSC}
	\end{subfigure}
	\hspace{.05cm}
	\begin{subfigure}[t]{0.235\textwidth}
		\centering
		\includegraphics[width=1.08\textwidth]{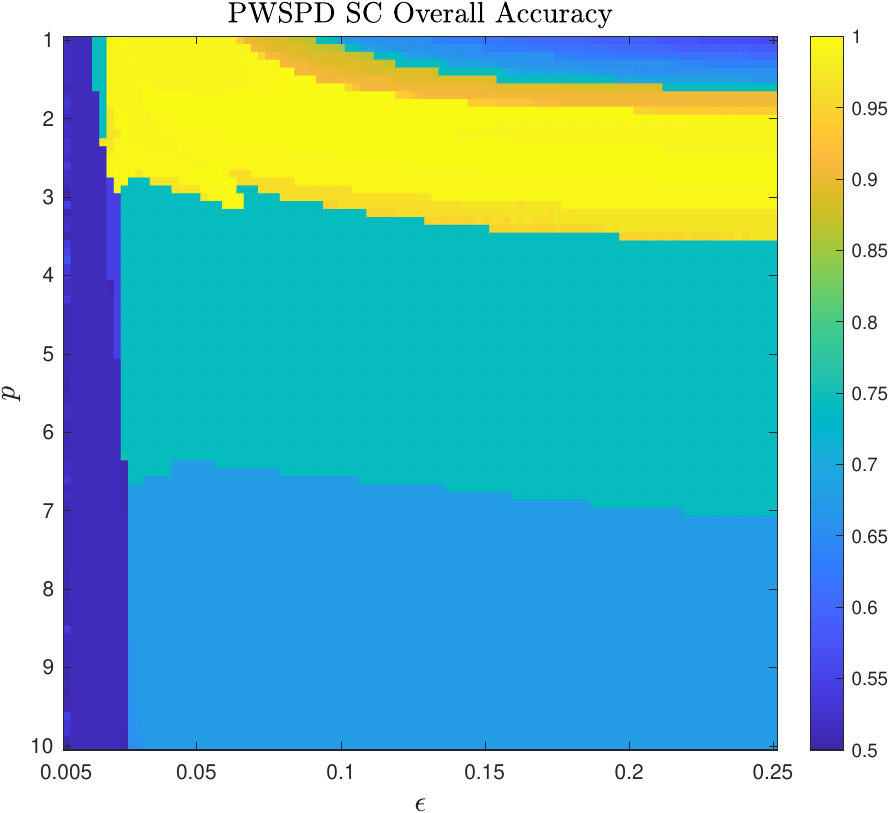}
		\subcaption{\tiny{OA, PWSPD SC}}
		\label{fig:ShortBottleneckAccuracyPlots_OA_SC_PWSPD}
	\end{subfigure}	 
	\hspace{.05cm}
	\begin{subfigure}[t]{0.235\textwidth}
		\centering
		\includegraphics[width=\textwidth]{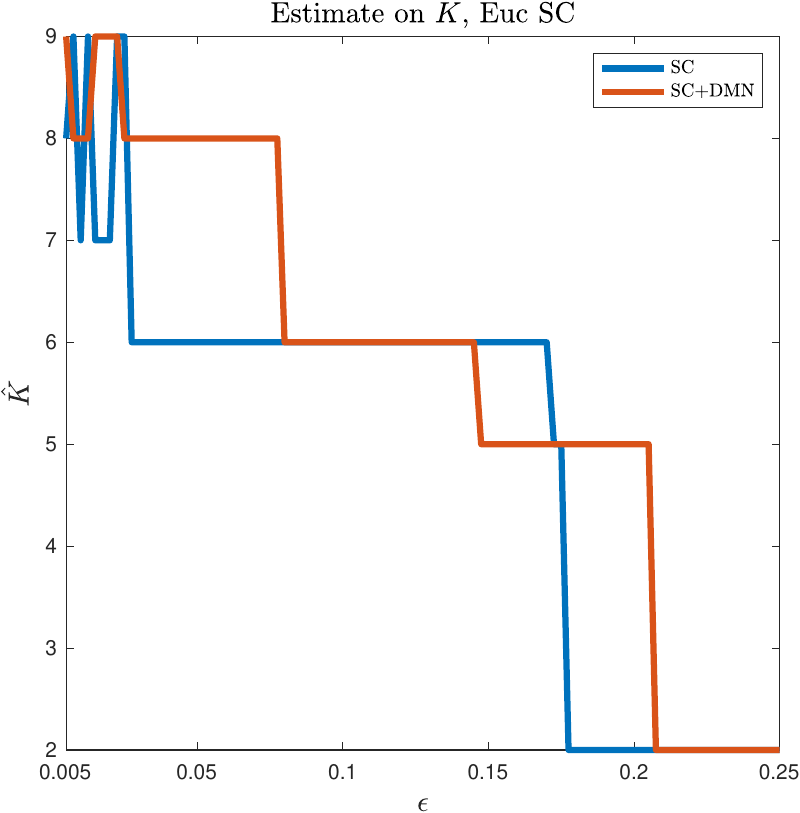}
		\subcaption{\tiny{$\hat{K}$, Euc. SC}}
		\label{fig:ShortBottleneckAccuracyPlots_K_SC}
	\end{subfigure}
	\hspace{.05cm}
	\begin{subfigure}[t]{0.235\textwidth}
		\centering
		\includegraphics[width=\textwidth]{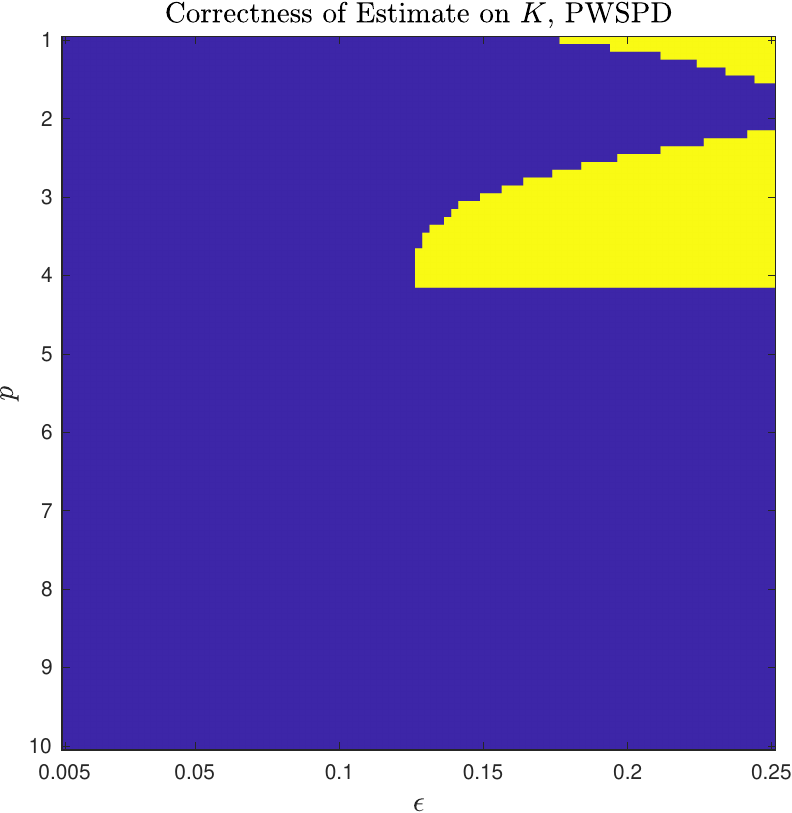}
		\subcaption{\tiny{$\hat{K}$, PWSPD SC (binarized)}}
		\label{fig:ShortBottleneckAccuracyPlots_K_SC_PWSPD_Binarized}
	\end{subfigure}
		\caption{\emph{Short Bottleneck dataset}.
		Because the underlying cluster structure is not driven entirely by geometry or density, the PWSPD SE separates the clusters for moderate $p$ (see Figure \ref{fig:short_bottleneck_p2}).  We note PWSPD is able to correctly learn $K$ and cluster accurately for $\epsilon$ somewhat large and $p$ between 2 and 3 (see Figures \ref{fig:ShortBottleneckAccuracyPlots_OA_SC_PWSPD}, \ref{fig:ShortBottleneckAccuracyPlots_K_SC_PWSPD_Binarized}), while Euclidean SC cannot simultaneously learn $K$ and cluster accurately (see Figures \ref{fig:ShortBottleneckAccuracyPlots_OA_EucSC}, \ref{fig:ShortBottleneckAccuracyPlots_K_SC}).}
	\label{fig:ShortBottleneck}
\end{figure}

\subsubsection{Comparison with Euclidean Spectral Clustering}

To evaluate how $p$ impacts the clusterability of the PWSPD spectral embedding, we consider experiments in which we run spectral clustering under various graph constructions.  We run $K$-means for a range of parameters on the spectral embedding $x_{i}\mapsto (\phi_{2}(x_{i}),\dots,\phi_{K}(x_{i}))$, where $\phi_{k}$ is the $k^{\text{th}}$ lowest frequency eigenvector of the Laplacian.  We construct the symmetric normalized Laplacian using PWSPD (denoted PWSPD SC) and also using Euclidean distances (denoted SC) and the Laplacian with diffusion maps normalization $a=1$ (denoted SC+DMN).  We vary $\epsilon$ in the SC and SC+DMN methods, and both $\epsilon$ and $p$ in the PWSPD SC method.  Results for selt-tuning SC, in which the $k$NN used to compute the local scaling parameter varies, are in Appendix~\ref{app:Additional_Clustering_Results}.  To allow for comparisons across figures, $\epsilon$ is varied across the percentiles of the pairwise distances in the underlying data, up to the $25^{th}$ percentile.  We measure two outputs of the clustering experiments:

\vspace{2pt}

\begin{enumerate}[(i)]
\item  The \emph{overall accuracy (OA)}, namely the proportion of data points correctly labeled after alignment when $K$ is known a priori.  For $K=2$, similar results were observed when thresholding $\phi_{2}$ at 0 instead of running $K$-means; see Appendix~\ref{app:Additional_Clustering_Results}.  

\item The \emph{eigengap estimate} of the number of latent clusters: $\hat{K}=\argmax_{k\ge 2}\lambda_{k+1}-\lambda_{k}$, where $0=\lambda_{1}\le\lambda_{2}\le\dots\le \lambda_{n}$ are the eigenvalues of the corresponding graph Laplacian.  We note that experiments estimating $K$ by considering the ratio of consecutive eigenvalues were also performed, with similar results.  In the case of PWSPD SC, we plot heatmaps of where $K$ is correctly estimated, with yellow corresponding to success ($\hat{K}=K$) and blue corresponding to failure ($\hat{K}\neq K$).  

\end{enumerate}

\vspace{2pt}

The results in terms of OA and $\hat{K}$ as a function of $\epsilon $ and $p$ are in Figures \ref{fig:TwoRings}, \ref{fig:LongBottleneck}, \ref{fig:ShortBottleneck}. 
We see that when density separates the data clearly, as in the Two Rings data, PWSPD SC with large $p$ gives accurate clustering results, while small $p$ may fail.  In this dataset, $\epsilon$ very small allows for the data to be correctly clustered with SC and SC+DMN when $K$ is known a priori.  However, the regime of $\epsilon$ is so small that the eigenvalues become unhelpful for estimating the number of latent clusters.  Unlike Euclidean spectral clustering, PWSPD SC correctly estimates $\hat{K}=2$ for a range of parameters, and achieves near-perfect clustering results for those parameters as well.  Indeed, as shown by Figures \ref{fig:TwoRingsAccuracyPlots_OA_SC_PWSPD}, \ref{fig:TwoRingsAccuracyPlots_K_SC_PWSPD_Binarized}, PWSPD SC with $p$ large is able to do fully unsupervised clustering on the Two Rings data.  

In the case of the Long Bottleneck dataset, there are three reasonable latent clusterings, depending on whether geometry, density, or both matter (see Figure \ref{fig:LongBottleneck_K_2}, \ref{fig:LongBottleneck_K_3}, \ref{fig:LongBottleneck_K_4}).  PWSPD is able to balance between the geometry and density-driven cluster structure in the data.  Indeed, all of the cluster configurations shown in Figure \ref{fig:LongBottleneck_K_2}, \ref{fig:LongBottleneck_K_3}, \ref{fig:LongBottleneck_K_4} are learnable without supervision for some choice of parameters $(\epsilon,p)$.  To capture the density cluster structure ($K=3$), $p$ should be taken large, as suggested in Figure \ref{fig:LongBottleneckAccuracyPlot_SC_PWSPD_OA_3}, \ref{fig:LongBottleneckAccuracyPlot_SC_PWSPD_K_3}.  To capture the geometry cluster structure ($K=2$), $p$ should be taken small and $\epsilon$ large, as suggested by Figures \ref{fig:LongBottleneckAccuracyPlot_SC_PWSPD_OA_2}, \ref{fig:LongBottleneckAccuracyPlot_SC_PWSPD_K_2}.  Interestingly, both cluster and geometry ($K=4$) can be captured by choosing $p$ moderate, as in Figure \ref{fig:LongBottleneckAccuracyPlot_SC_PWSPD_OA_4}, \ref{fig:LongBottleneckAccuracyPlot_SC_PWSPD_K_4}.  For Euclidean SC, varying $\epsilon$ is insufficient to capture the rich structure of this data.

In the case of the Short Bottleneck, taking $\epsilon$ large allows for the Euclidean methods to correctly estimate the number of clusters.  But, in this $\epsilon$ regime, the methods do not cluster accurately.  On the other hand, taking $p$ between 2 and 3 and $\epsilon$ large allows PWSPD to correctly estimate $K$ and also cluster accurately.  

Overall, this suggests that varying $p$ in PWSPD SC has a different impact than varying the scaling parameter $\epsilon$, and can allow for richer cluster structures to be learned when compared to SC with Euclidean distances.  In addition, PWSPDs generally allow for the underlying cluster structures to be learned in a \emph{fully unsupervised manner}, while Euclidean methods may struggle to simultaneously cluster well and estimate $K$ accurately.

\section{Spanners for PWSPD}
\label{sec:Spanners}
Let $\mathcal{H}\subset \mathcal{G}_{\mathcal{X}}^{p}$ denote a subgraph and recall the definition of $\ell^{\mathcal{H}}_{p}(\cdot,\cdot)$ given in Definition \ref{defn:PWSPD_G}.
\begin{defn}
\label{def:Spanner}
 For $t\geq 1$, $\mathcal{H}\subset \mathcal{G}_{\mathcal{X}}^{p}$ is a \emph{$t$-spanner} if $\ell^{\mathcal{H}}_{p}(x,y) \leq t\ell_{p}(x,y)$ for all $x,y\in\mathcal{X}$.
\end{defn}
Clearly $\ell_{p}(x,y) \leq \ell^{\mathcal{H}}_{p}(x,y)$ always, as any path in $\mathcal{H}$ is a path in $\mathcal{G}^{p}_{\mathcal{X}}$. Hence if $\mathcal{H}$ is a {\em $1$-spanner} we have equality: $\ell^{\mathcal{H}}_{p}(x,y) = \ell_{p}(x,y)$. Define the $k$NN graph, $\mathcal{G}^{p,k}_{\mathcal{X}}$, by retaining only edges $\{x,y\}$ if  $x$ is a $k$NN of $y$ or vice versa.  For appropriate $k,p$ and $\mathcal{M}$ it is known that $\mathcal{G}^{p,k}_{\mathcal{X}}$ is a $1$-spanner of $\mathcal{G}^{p}_{\mathcal{X}}$ w.h.p. Specifically, \cite{Groisman2018nonhomogeneous} shows this when $\mathcal{M}$ is an open connected set with $C^{1}$ boundary, $1 < p <\infty$ and $k = O(c_{p,d}\log(n))$ for a constant $c_{p,d}$ depending on $p,d$. One can deduce $c_{p,d} \geq 2^{d+1}3^{d}d^{d/2}$, while the dependence on $p$ is more obscure.  A different approach is used in \cite{chu2020exact} to show this for arbitrary smooth, closed, isometrically embedded $\mathcal{M}$, $2 \leq p < \infty$ and $k = O(2^{d}\log(n))$, where $O$ hides constants depending on the geometry of $\mathcal{M}$. In both cases $f$ must be continuous and bounded away from zero. 

Under these assumptions, we prove $\mathcal{G}^{p,k}_{\mathcal{X}}$ is a $1$-spanner w.h.p., for any smooth, closed, isometrically embedded $\mathcal{M}$ with mild restrictions on its curvature. Our results hold generally for $1<p<\infty$ and enjoy improved dependence of $k$ on $d$ and explicit dependence of $k$ on $p$ and the geometry of $\mathcal{M}$ compared to \cite{Groisman2018nonhomogeneous, chu2020exact}. We also consider an {\em intrinsic} version of PWSPD, \[\displaystyle \ell_{\mathcal{M},p}(x,y) = \left(\min_{\pi=\{x_{i_{j}}\}_{j=1}^\Length} \sum_{j=1}^{\Length-1} \D(x_{i_{j}},x_{i_{j+1}})^p\right)^{1/p},\] where $\D(\cdot,\cdot)$ is assumed known, which is not typically the case in data science. However this situation can occur when $\mathcal{X}$ is presented as a subset of $\mathbb{R}^{D}$, but one wishes to analyze $\mathcal{X}$ with an exotic metric (i.e. not $\|\cdot\|$). For example, if each $x_i\in\mathcal{X}$ is an image, a Wasserstein metric may be more appropriate than $\|\cdot\|$.  As this case closely mirrors the statement and proof of Theorem~\ref{theorem:EuclideanCase} we leave it to Appendix~\ref{subsec:PWSPD_Manifold}. Before proceeding we introduce some further terminology:
\begin{defn}
The edge $\{x,y\}$ is \emph{critical} if it is in the shortest path from $x$ to $y$ in $\mathcal{G}_{\mathcal{X}}^{p}$.  
\end{defn}
\begin{lem}\cite{chu2020exact}
\label{lemma:Critical_edges_one_spanners}
$\mathcal{H}\subset \mathcal{G}_{\mathcal{X}}^{p}$ is a $1$-spanner if it contains every critical edge of $\mathcal{G}_{\mathcal{X}}^{p}$.
\end{lem}

\subsection{Nearest Neighbors and PWSPD Spanners}\label{subsec:PWSPD_Euclidean}

A key proof ingredient is the following definition, which generalizes the role of spheres in the proof of Theorem 1.3 in \cite{chu2020exact}. 

\begin{defn}\label{defn:p_elongated_set}
For any $x,y\in\mathbb{R}^{d}$ and $\alpha\in (0,1]$, the \emph{$p$-elongated set} associated to $x,y$ is \[\mathcal{D}_{\alpha,p}(x,y) = \left\{ z\in \mathbb{R}^{d}: \ \|x - z\|^{p} + \|y-z\|^{p} \leq \alpha\|x - y\|^{p}\right\}.\]
\end{defn}Visualizations of $\mathcal{D}_{1,p}(x,y)\subset\mathbb{R}^{2}$ are shown in Figure \ref{fig:p_elongated}. $\mathcal{D}_{1,p}(x,y)$ is the set of points $z$ such that the two-hop path, $x\to z\to y$, is $\ell_{p}$-shorter than the one-hop path, $x\to y$. Hence:
\begin{figure}
\centering
	\includegraphics[width=.32\textwidth]{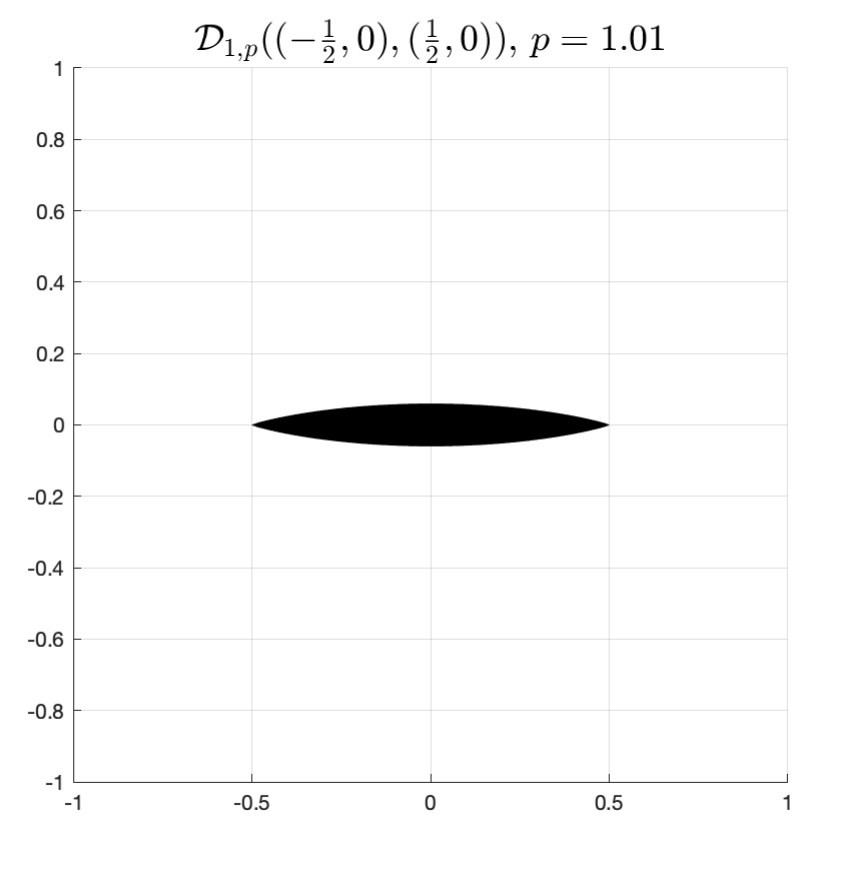}
	\includegraphics[width=.32\textwidth]{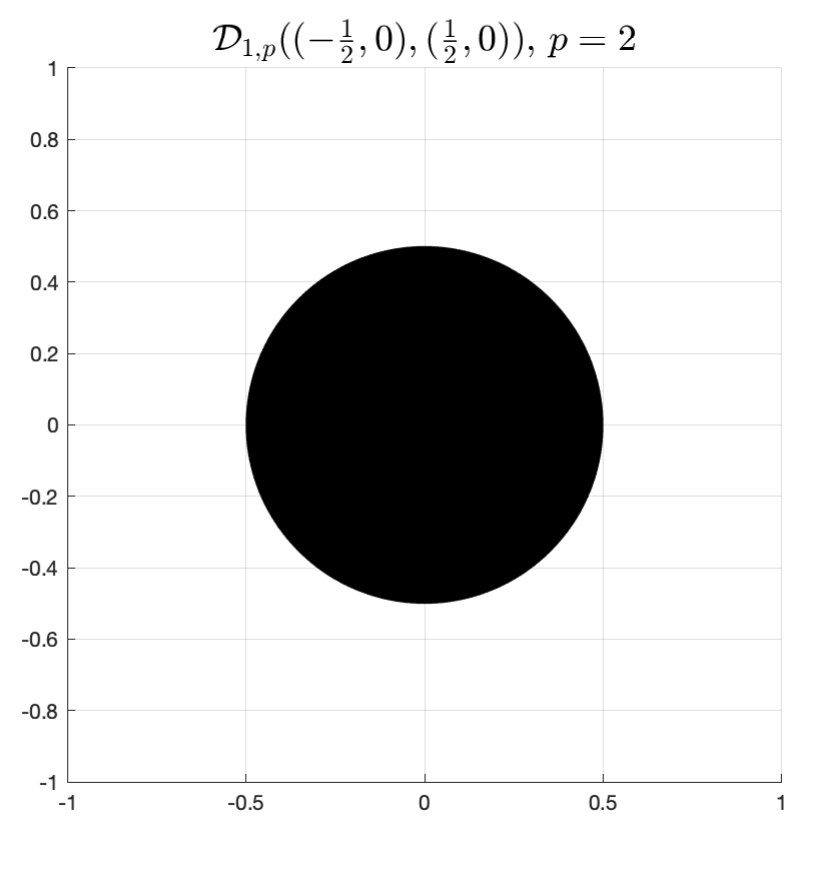}
	\includegraphics[width=.32\textwidth]{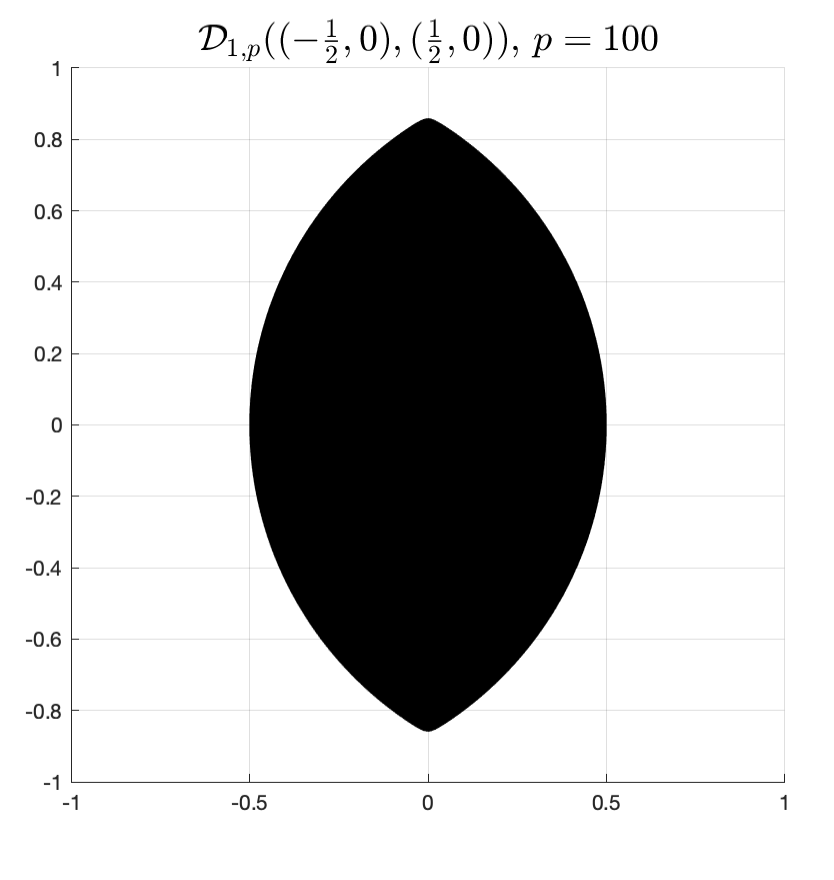}
\caption{Plots of $\mathcal{D}_{1,p}((-\frac{1}{2},0),(\frac{1}{2},0))$ for $p=1.01,2,100$.  We see that for smaller $p$, the set becomes quite small, converging to a line segment as $p\rightarrow 1^{+}$.  For $p=2$, the $p$-elongated set is a circle.  As $p$ increases, $\mathcal{D}_{1,p}((-\frac{1}{2},0),(\frac{1}{2},0))$ converges to a set resembling a vertically-oriented American football.}\label{fig:p_elongated}
\end{figure}

\begin{lem}
If there exists $z \in \mathcal{D}_{1,p}(x,y) \cap \mathcal{X}$ then the edge $\{x,y\}$ is not critical.
\label{lem:Edge_Criticality} 
\end{lem}

We defer the proof of the following technical Lemma to Appendix~\ref{app:Proofs_for_Spanners}.

\begin{lem}\label{lemma:Volume_of_Region}
Let $r := \|x-y\|$, $x_{M} = \frac{x+y}{2}$, and $r^{\star}:= r\sqrt{\frac{\alpha^{2/p}}{4^{1/p}} - \frac{1}{4}}$ for $\alpha>2^{1-p}$. Then:
\[B(x_{M},r^{\star}) \subset \mathcal{D}_{\alpha,p}(x,y) \subset B(x,r).\]
\end{lem}

For $\alpha =1$, \cite{Groisman2018nonhomogeneous} makes a similar claim but crucially does not quantify the dependence of the radius of this ball on $p$. Before proceeding, we introduce two regularity assumptions:

\begin{defn}
\label{defn:V}
 $\mathcal{M} \subset\mathbb{R}^{D}$ is in \emph{$V(d,\kappa_{0},\epsilon_0)$} for $\kappa_{0} \geq 1$ and $\epsilon_{0} > 0$ if it is connected and for all $x\in \mathcal{M},\ \epsilon \in (0,\epsilon_{0})$ we have: $\kappa_{0}^{-1}\epsilon^{d} \leq \text{vol}(\mathcal{M}\cap B(x,\epsilon))/\text{vol}(B(0,1)) \leq \kappa_{0}\epsilon^{d}$. 
\end{defn}

\begin{defn}
\label{defn:reach}
A compact manifold $\mathcal{M}\subset \mathbb{R}^{D}$ has \emph{reach} $\zeta > 0$ if every $x \in \mathbb{R}^{D}$ satisfying $\text{dist}(x,\mathcal{M}) := \min_{y\in \mathcal{M}}\|x - y\| < \zeta$ has a unique projection onto $\mathcal{M}$.
\end{defn}

\begin{thm}
	\label{thm:required_kNN}
Let $\mathcal{M}\in V(d,\kappa_{0},\epsilon_0)$ be a compact manifold with reach $\zeta > 0$. Let $\mathcal{X} = \{x_i\}_{i=1}^{n}$ be drawn i.i.d. from $\mathcal{M}$ according to a probability distribution with continuous density $f$ satisfying $0 < f_{\min} \leq f(x) \leq f_{\max}$ for all $x \in \mathcal{M}$. For $p>1$ and $n$ sufficiently large, $\mathcal{G}^{p,k}_{\mathcal{X}}$ is a $1$-spanner of $\mathcal{G}^{p}_{\mathcal{X}}$ with probability at least $1 -1/n$ if
\begin{equation}
 	k \geq 4\kappa_{0}^{2}\left[\frac{f_{\max}}{f_{\min}}\right]\left[\frac{4}{4^{1-1/p}-1}\right]^{d/2}\log(n).
 	\label{eq:Main_k_bound}
\end{equation}
\label{theorem:EuclideanCase}
\end{thm}

\begin{proof}
In light of Lemma \ref{lemma:Critical_edges_one_spanners} we prove that, with probability at least $1-1/n$, $\mathcal{G}^{p,k}_{\mathcal{X}}$ contains every critical edge of $\mathcal{G}^{p}_{\mathcal{X}}$. Equivalently, we show every edge of $\mathcal{G}^{p}_{\mathcal{X}}$ not contained in $\mathcal{G}^{p,k}_{\mathcal{X}}$ is not critical. 


For any $c,\epsilon>0$, $\Prob\left[\displaystyle\max_{x,y\in\X}\ell_{p}(x,y)\le\epsilon\right]\ge 1-c/n$ for $n$ sufficiently large \cite{mckenzie2019power}.  So, let $n$ be sufficiently large so that $\Prob\left[\ell_{p}(x,y)\le \min\left\{\epsilon_{0},\frac{\zeta}{d}\sqrt{\frac{1}{4^{1/p}}-\frac{1}{4}}\right\} \text{ for all } x,y\in \mathcal{X}\right]\ge  \left(1-\frac{1}{2n}\right)$.  Pick any $x,y\in\mathcal{X}$ which are not $k$NNs and let $r:=\|x-y\|$.  If $r> \min\left\{\epsilon_{0},\frac{\zeta}{d}\sqrt{\frac{1}{4^{1/p}}-\frac{1}{4}}\right\}$, then $\ell_{p}(x,y)< \|x-y\|$ and thus the edge $\{x,y\}$ is not critical.  So, suppose without loss of generality in what follows that $r\le \min\left\{\epsilon_{0},\frac{\zeta}{d}\sqrt{\frac{1}{4^{1/p}}-\frac{1}{4}}\right\}$.  

Define $r_{1}^{\star}:= r\sqrt{\frac{1}{4^{1/p}} - \frac{1}{4}}$ and $r^{\star}_2 := r\left(\sqrt{\frac{1}{4^{1/p}} - \frac{1}{4}}- \frac{r}{4\zeta}\right)$; note that $r_{2}^{\star}>0$ by the assumption $r\le \frac{\zeta}{d}\sqrt{\frac{1}{4^{1/p}}-\frac{1}{4}}$.  Let $x_{M} := \frac{x+y}{2}$ and let $\tilde{x}_{M} := \argmin_{z\in\mathcal{M}}\|x_{M} - z\|$ be the projection of $x_{M}$ onto $\mathcal{M}$, which is unique because $r<\zeta$.  By Lemma~\ref{lemma:Volume_of_Region}, $B(x_{M},r_{1}^{\star}) \subset \mathcal{D}_{1,p}(x,y)\subset B(x,r)$.  By Lemma~\ref{lem:Ball_radius_reach}, $B(\tilde{x}_{M},r_{2}^{\star})\subset B(x_{M},r_{1}^{\star})$. Let $x_{i_1},\ldots, x_{i_k}$ denote the $k$NNs of $x$, ordered randomly. Because $y$ is not a $k$NN of $x$, $\|x-x_{i_j}\| \leq \|x-y\| = r$ for $j=1,\ldots,k$. Thus, $x_{i_j} \in B(x,r)$ and so by Lemma~\ref{lem:r_star_vol_lower_bound} we bound for fixed $j$
\begin{align}
	\Prob\left[x_{i_j}\in \mathcal{D}_{1,p}(x,y)  \ | \ x_{i_j}\in B(x,r)\right] & \geq \Prob\left[x_{i_j}\in B(\tilde{x}_{M},r_{2}^{\star})  \ | \ x_{i_j}\in B(x,r)\right] \\
		& \geq \frac{3}{4}\kappa_{0}^{-2}\frac{f_{\min}}{f_{\max}}\left(\frac{1}{4^{1/p}} - \frac{1}{4}\right)^{d/2} =: \varepsilon_{\mathcal{M},p,f}.
\end{align}
Because the $x_{i_j}$ are all independently drawn:
\begin{align*}
\Prob\left[ \not\exists j \text{ with } x_{i_j}\in \mathcal{D}_{1,p}(x,y)\right] = \prod_{j=1}^{k}\Prob\left[ x_{i_j}\notin \mathcal{D}_{1,p}(x,y)\ | \ x_{i_j}\in B(x,r) \right] 
     \leq \left(1 - \varepsilon_{\mathcal{M},p,f}\right)^{k}.
\end{align*}
A routine calculations reveals that for $k \geq \frac{3\log n}{-\log(1-\varepsilon_{\mathcal{M},p,f})}$,
\begin{align}
\Prob\left[\exists j \text{ with } x_{i_j}\in \mathcal{D}_{1,p}(x,y)\right] = 1 - \Prob\left[ \not\exists j \text{ with } x_{i_j}\in \mathcal{D}_{1,p}(x,y)\right] \ge 1 - \frac{1}{n^3}.
\label{eq:Apply_Union_Bound}
\end{align}
By Lemma~\ref{lem:Edge_Criticality} we conclude the edge $\{x,y\}$ is not critical with probability exceeding $1 - \frac{1}{n^3}$.  There are fewer than $n(n-1)/2$ such non-$k$NN pairs $x,y\in \mathcal{X}$. These edges $\{x,y\}$ are precisely those contained in $\mathcal{G}^{p}_{\mathcal{X}}$ but not in $\mathcal{G}^{p,k}_{\mathcal{X}}$. By the union bound and \eqref{eq:Apply_Union_Bound} we conclude that none of these are critical with probability greater than $1 - \frac{n(n-1)}{2}\frac{1}{n^3} \geq 1 - \frac{1}{2n}$.  This was conditioned on $\ell_{p}(x,y)\le \min\left\{\epsilon_{0},\frac{\zeta}{d}\sqrt{\frac{1}{4^{1/p}}-\frac{1}{4}}\right\}$ for all $x,y\in \mathcal{X}$, which holds with probability exceeding $1-\frac{1}{2n}$.  Thus, all critical edges are contained in $\mathcal{G}_{p,k}^{\mathcal{X}}$ with probability exceeding $1-\left(\frac{1}{2n}+\frac{1}{2n}\right)=1-\frac{1}{n}$.  Unpacking $\varepsilon_{\mathcal{M},p,f}$ yields the claimed lower bound on $k$.
\end{proof}

In \eqref{eq:Main_k_bound}, the explicit dependence of $k$ on $\kappa_{0}, p$, and $d$ are shown.  The $4\kappa_{0}^{2}$ factor corresponds to the geometry of $\mathcal{M}$. The numerical constant 4, which is not tight, stems from accounting for the reach of $\mathcal{M}$. If $\mathcal{M}$ is convex (i.e. $\zeta = \infty$) then it can be replaced with $3$.  The second factor in \eqref{eq:Main_k_bound} is controlled by the probability distribution while the third corresponds to $p$ and $d$.  For $p=2$ and ignoring geometric and density factors we attain $k = O(2^d\log(n))$ as in \cite{chu2020exact}. For large $p$
 we get $k \approx O\left(\left(\frac{4}{3}\right)^{d/2}\log(n)\right)$, thus improving the dependence of $k$ on $d$ given in \cite{Groisman2018nonhomogeneous, chu2020exact}. Finally, using Corollary 4.4 of \cite{mckenzie2019power} we can sharpen the qualitative requirement that $n$ be ``sufficiently large" 
 to the quantitative lower bound $n \geq C\max\left\{\left[\frac{d}{\zeta}\right]^{\frac{pd}{p-1}}\left[\frac{4}{4^{1-1/p}-1}\right]^{\frac{pd}{2(p-1)}}, \left[\frac{1}{\epsilon_0}\right]^{\frac{pd}{p-1}}\right\}$ for a constant $C$ depending on the geometry of $\mathcal{M}$.  So, when $\mathcal{M}$ is high-dimensional, has small reach, or when $p$ is close to 1, $n$ may need to be quite large for $k$ as in \eqref{eq:Main_k_bound} to yield a 1-spanner.  

\subsection{Numerical Experiments}
\label{subsec:PWSPD_Experiments}

We verify the claimed dependence of $k$ on $n,p$ and $d$ ensures that $\mathcal{G}^{p,k}_{\mathcal{X}}$ is a $1$-spanner of $\mathcal{G}^{p}_{\mathcal{X}}$ numerically. To generate Figures \ref{fig:HeatMap1}--\ref{fig:HeatMap_Ridge} we: 

\begin{enumerate}[(1)]
	\item Fix $p,d,\mathcal{M},$ and $f$, then generate a sequence of $(n,k)$ pairs.
	\item For each $(n,k)$, do:
	\begin{enumerate}[(i)]
		\item Generate $\mathcal{X} = \{x_i\}_{i=1}^{n}$ by sampling i.i.d. from $f$ on $\mathcal{M}$. 
		\item For all pairs $\{x_i,x_j\}$ compute $\ell_{p}(x_i,x_j)$ and $\ell_{p}^{\mathcal{G}^{p,k}_{\mathcal{X}}}(x_i,x_j)$.
		\item If $\displaystyle\max_{1\leq i < j \leq n}\left|  \ell_{p}(x_i,x_j) - \ell_{p}^{\mathcal{G}^{p,k}_{\mathcal{X}}}(x_i,x_j)\right| > 10^{-10}$ record ``failure''; else, record ``success''.
	\end{enumerate}
	\item Repeat step 2 twenty times and compute the proportion of successes.
\end{enumerate} 

As can be seen from Figure~\ref{fig:Heatmaps}, there is a sharp transition between an ``all failures" and an ``all successes" regime. The transition line is roughly linear when viewed using semi-log-x axes, i.e. $k \propto \log(n)$. Moreover the slope of the line-of-best-fit to this transition line decreases with increasing $p$ (compare Figure \ref{fig:HeatMap1}-\ref{fig:HeatMap3}) and depends on intrinsic, not extrinsic dimension (compare Figure \ref{fig:HeatMap2} and \ref{fig:HeatMap4}), as predicted by Theorem~\ref{theorem:EuclideanCase}. Intriguingly, there is little difference between Figure \ref{fig:HeatMap2} (uniform distribution) and Figure \ref{fig:HeatMap5} (Gaussian distribution), suggesting that perhaps the assumption $f_{\min} > 0$ in Theorem~\ref{theorem:EuclideanCase} is unnecessary. Finally, we observe that the constant of proportionality (i.e. $C$ such that $k = C\log n$) predicted by Theorem \ref{theorem:EuclideanCase} appears pessimistic. For Figure \ref{fig:HeatMap1}-\ref{fig:HeatMap3}, Theorem \ref{theorem:EuclideanCase} predicts $C = 484.03,128$ and $21.76$ respectively (taking $\kappa_0=1$ due to the flat domain), while empirically the slope of the line-of-best-fit is $43.43$, $25.33$ and $0.29$ respectively.  

In Figure \ref{fig:HeatMap_Ridge}, we consider an intrinsically 4-dimensional set corrupted with Gaussian noise (standard deviation $0.1$) in the fifth dimension.  Interestingly, the scaling with $k$ is more efficient than as shown in Figure \ref{fig:HeatMap1} for the intrinsically 5-dimensional data.  This suggests that measures which concentrate near low-dimensional sets benefit from that low-dimensionality, even if they are not supported exactly on it.

We also consider relaxing the success condition (2.iii). We define $\mathcal{H}$ to be a \emph{$(t,\omega)$-spanner} if $\ell^{\mathcal{H}}_p(x,y) \leq t \ell_p(x,y)$ for $\omega\in (0,1]$ proportion of the edges, so that Theorem \ref{thm:required_kNN} pertains to (1,1)-spanners.  Figures \ref{fig:HeatMap6} and \ref{fig:HeatMap7} show the minimal $\omega$ (averaged across simulations) for which $\mathcal{G}^{p,k}_{\mathcal{X}}$ is a $(1.1,\omega)$-spanner and a $(1.01,\omega)$-spanner respectively; the red lines trace out the requirements for $\mathcal{G}^{p,k}_{\mathcal{X}}$ to be a $(1.1,1)$-spanner and $(1.01,1)$-spanner respectively. Comparing with Figure \ref{fig:HeatMap2}, we see that the required scaling for $\mathcal{G}^{p,k}_{\mathcal{X}}$ to be a $(1+\epsilon,1)$-spanner is similar to the required scaling to be a $(1,1)$-spanner, at least for $\epsilon>0$ small. However, the required scaling for $(1+\epsilon, \omega)$-spanners ($\omega<1$) is quite different and much less restrictive, even for $\omega$ very close to 1; for example the requirement for $\mathcal{G}^{p,k}_{\mathcal{X}}$ to be a $(1.01, 0.95)$-spanner appears sublinear in the $\log_2(n)$ versus $k$ plot (see Figure \ref{fig:HeatMap7}). If this notion of approximation is acceptable, our empirical results suggest one can enjoy much greater sparsity.  
Finally, in Figure~\ref{fig:HeatMap8} we compute the minimal $t \geq 1$ such that $\mathcal{G}^{p,k}_{\mathcal{X}}$ is a $(t,1)$-spanner of $\mathcal{G}^{p}_{\mathcal{X}}$; again the overall transition patterns for $(t,1)$-spanners are similar to the $(1,1)$-spanner case in Figure \ref{fig:HeatMap2} when $t$ is close to 1. Overall we see that greater sparsity is permissible in these relaxed cases, and analyzing such notions rigorously is a topic of ongoing research.

\begin{figure}[htbt!]
	\centering
	\begin{subfigure}[b]{0.32\textwidth}
		\centering
		\includegraphics[width= .8\textwidth]{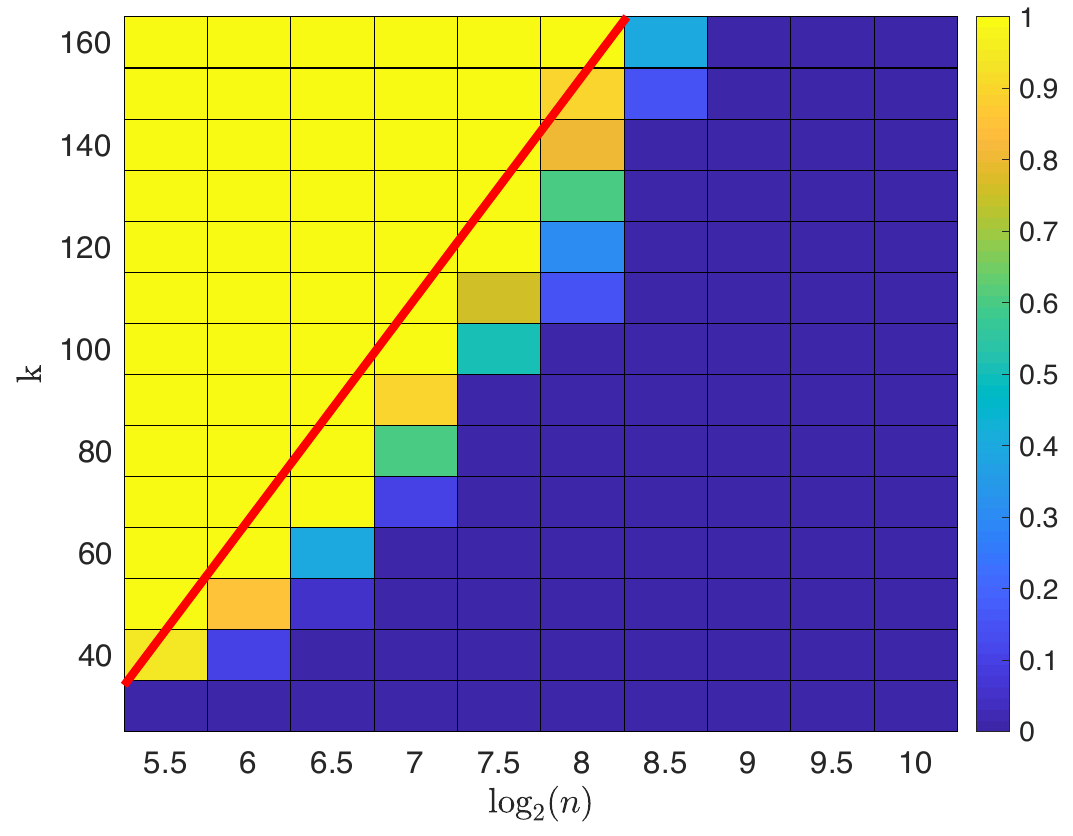}
		\caption{\tiny{$p = 1.5, \mathcal{M} = [0,1]^{5}$. Uniform.}}
		\label{fig:HeatMap1}
	\end{subfigure}
	\begin{subfigure}[b]{0.32\textwidth}
		\centering
		\includegraphics[width=.8\textwidth]{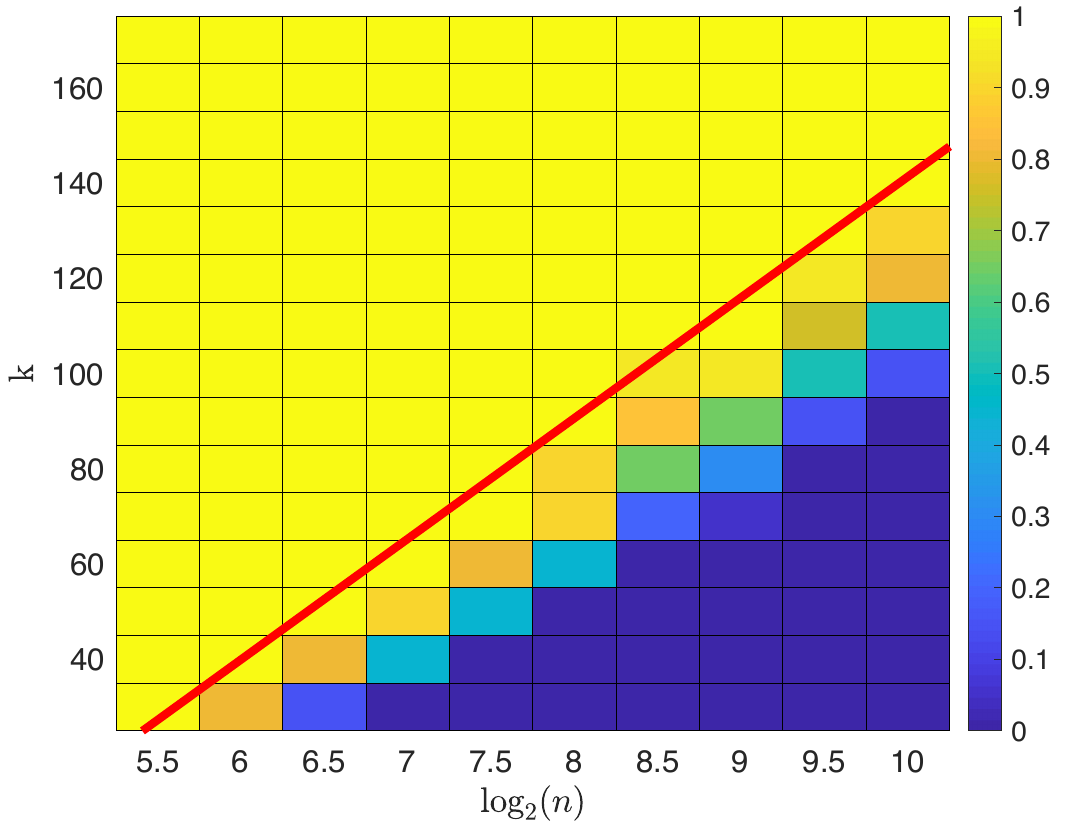}
		\caption{\tiny{$p = 2, \mathcal{M} = [0,1]^{5}$. Uniform.}}
		\label{fig:HeatMap2}
	\end{subfigure}
	\begin{subfigure}[b]{0.32\textwidth}
		\centering
		\includegraphics[width= .8\textwidth]{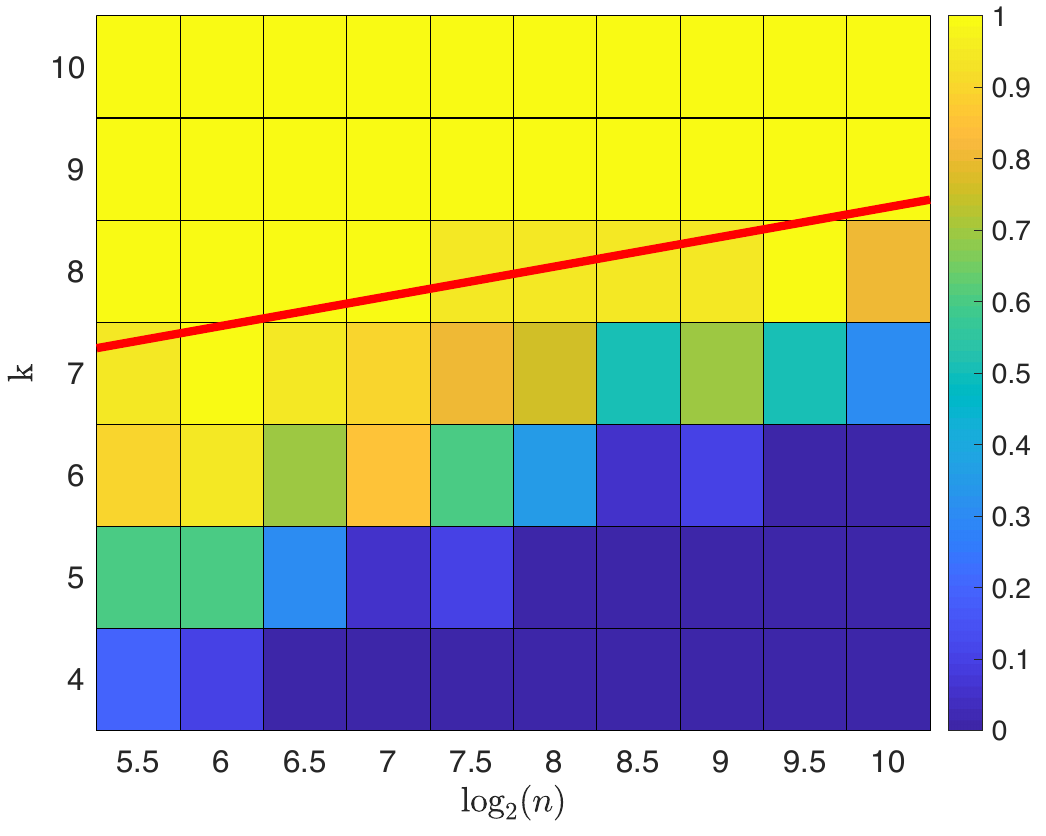}		
		\caption{\tiny{$p = 10, \mathcal{M} = [0,1]^{5}$. Uniform.}}
		\label{fig:HeatMap3}
	\end{subfigure}
	\begin{subfigure}[b]{0.32\textwidth}
		\centering
		\includegraphics[width= .8\textwidth]{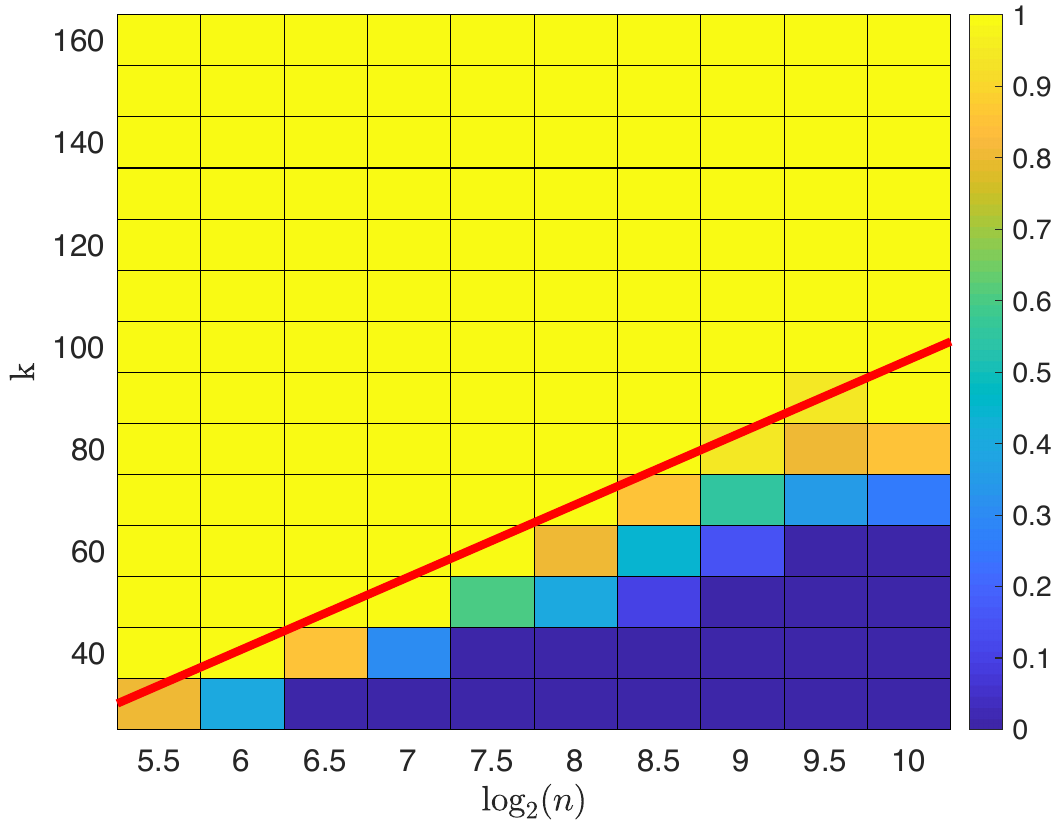}
		\caption{\tiny{$p = 2,\mathcal{M}  = \mathbb{S}^{4}\subset\mathbb{R}^{5}$. Uniform.}}
		\label{fig:HeatMap4}
	\end{subfigure}
	\begin{subfigure}[b]{0.32\textwidth}
		\centering
		\includegraphics[width= .8\textwidth]{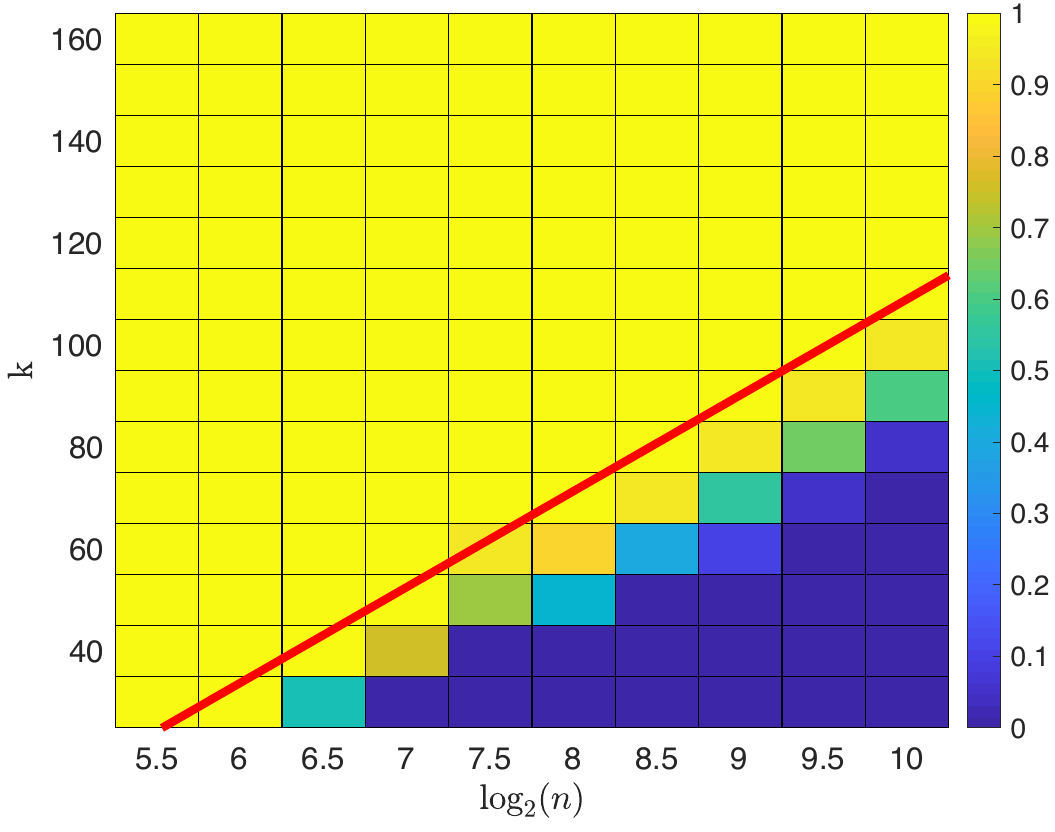}
		\caption{\tiny{$p = 2, \mathcal{M}  = [0,1]^5$. Gaussian dist.}}
		\label{fig:HeatMap5}
	\end{subfigure}
	\begin{subfigure}[b]{0.32\textwidth}
		\centering
		\includegraphics[width= .8\textwidth]{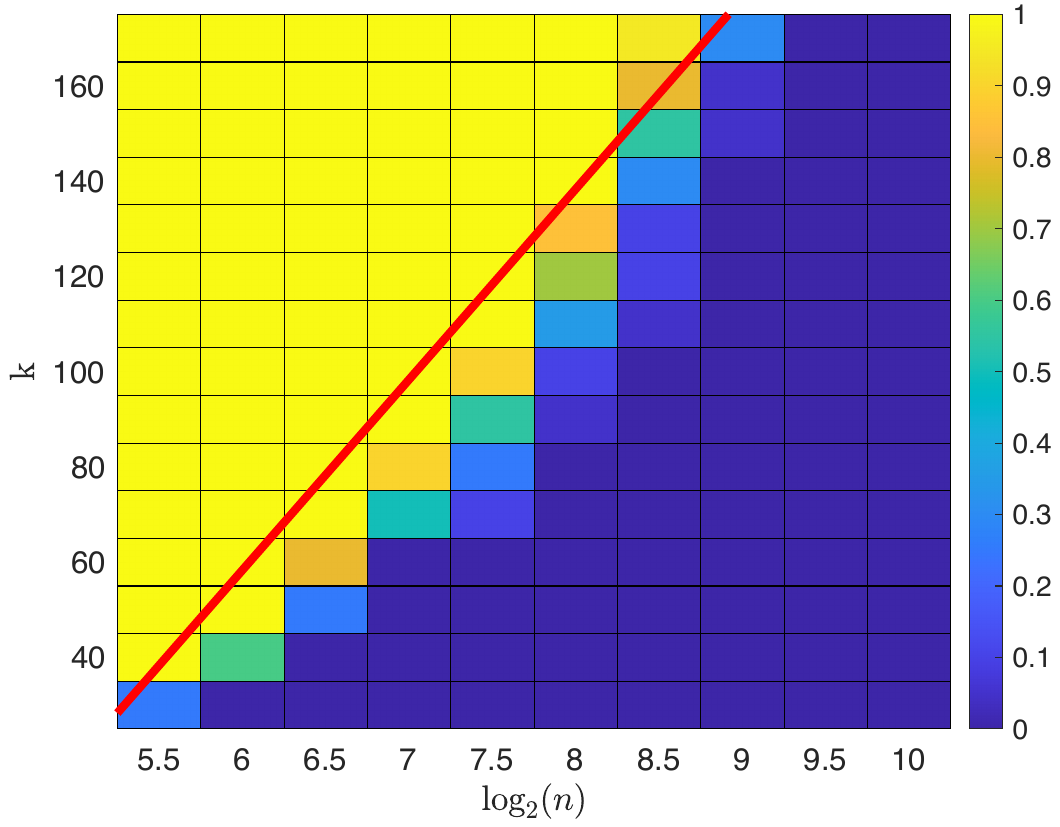}
		\caption{\tiny{$p = 1.5, \mathcal{M}  = [0,1]^4$, Uniform+noise.}}
		\label{fig:HeatMap_Ridge}
	\end{subfigure}
	\begin{subfigure}[b]{0.32\textwidth}
		\centering
		\includegraphics[width= .8\textwidth]{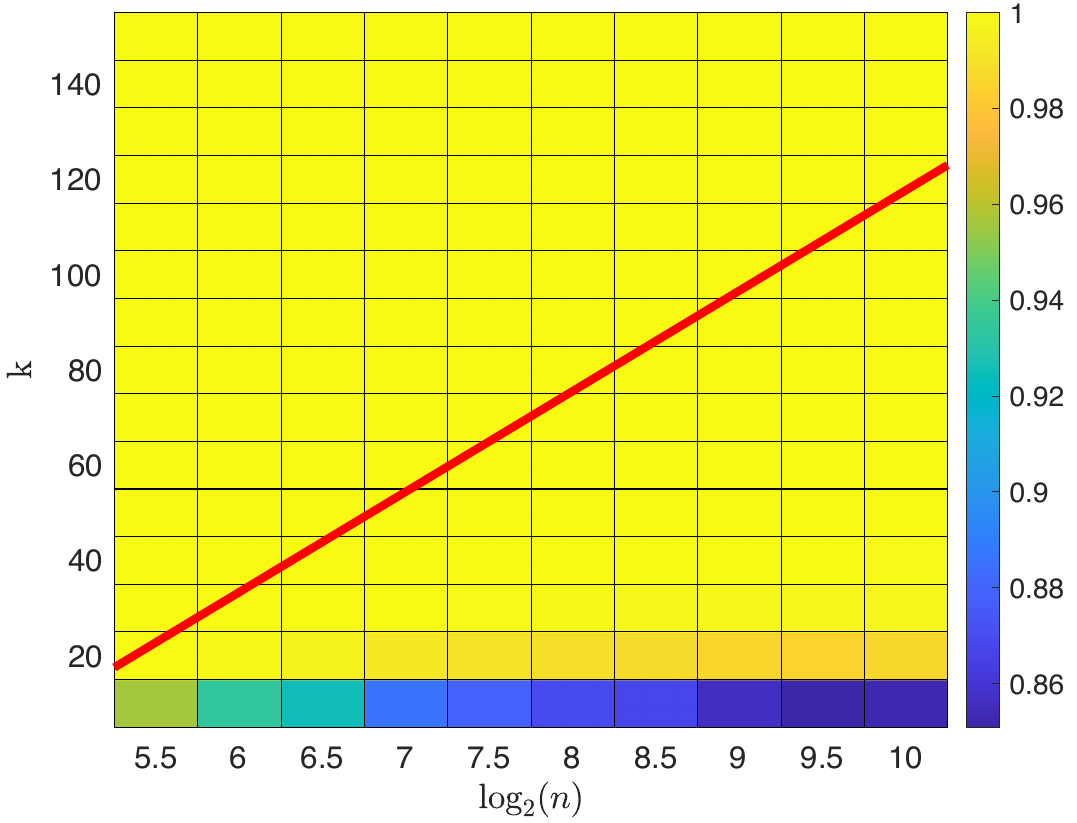}
		\caption{\tiny{$p = 2, \mathcal{M}  = [0,1]^5$. Uniform.}}
		\label{fig:HeatMap6}	
	\end{subfigure}
	\begin{subfigure}[b]{0.32\textwidth}
		\centering
		\includegraphics[width= .8\textwidth]{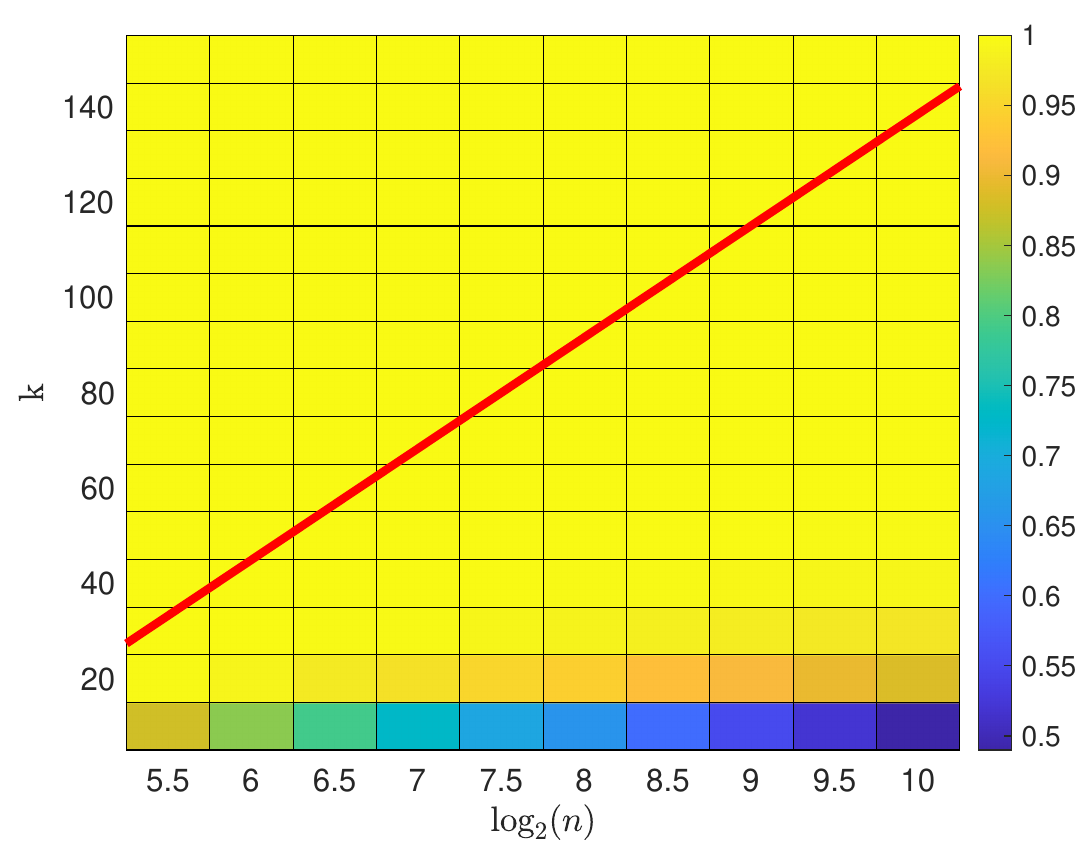}
		\caption{\tiny{$p = 2, \mathcal{M}  = [0,1]^5$. Uniform.}}
		\label{fig:HeatMap7}
	\end{subfigure}
	\begin{subfigure}[b]{0.32\textwidth}
		\centering
		\includegraphics[width= .8\textwidth]{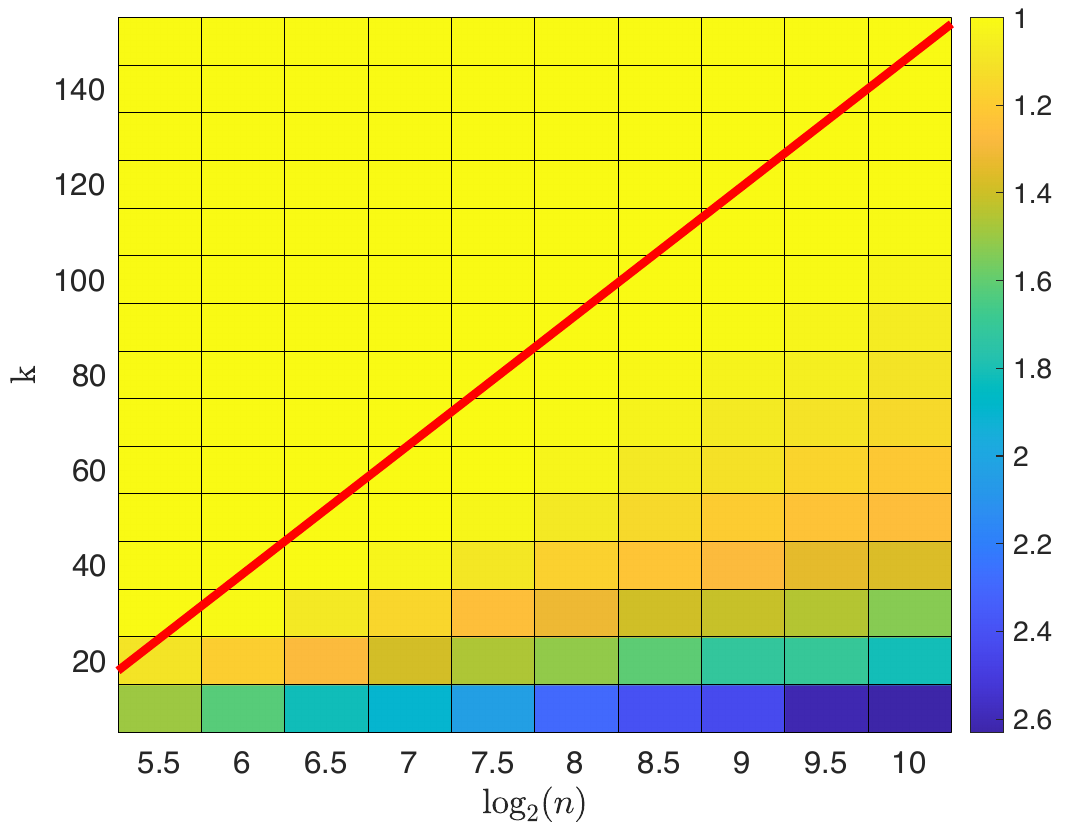}
		\caption{\tiny{$p = 2, \mathcal{M}  = [0,1]^5$. Uniform.}}
		\label{fig:HeatMap8}
	\end{subfigure}
	\caption{\label{fig:Heatmaps}  Figures~\ref{fig:HeatMap1}--\ref{fig:HeatMap_Ridge} show the proportion of randomly generated data sets for which $\mathcal{G}^{p,k}_{\mathcal{X}}$ is a $1$-spanner of $\mathcal{G}^{p}_{\mathcal{X}}$. The red line is the line of best fit through the cells representing the first value of $k$, for each value of $n$, for which all trials were successful, i.e. it is the line ensuring $\mathcal{G}^{p,k}_{\mathcal{X}}$ is a $1$-spanner. The slopes of Figure \ref{fig:HeatMap1}--\ref{fig:HeatMap_Ridge} are, respectively, $43.43$, $25.33$, $0.29$, $14.18$, $18.79$, and $40.0$. Figures \ref{fig:HeatMap6} and \ref{fig:HeatMap7} show the minimal $\omega$ (averaged across simulations) for which $\mathcal{G}^{p,k}_{\mathcal{X}}$ is a $(1.1,\omega)$-spanner and a $(1.01,\omega)$-spanner respectively; the red lines trace out the requirements for $\mathcal{G}^{p,k}_{\mathcal{X}}$ to be a $(1.1,1)$-spanner and $(1.01,1)$-spanner respectively. Figure~\ref{fig:HeatMap8} shows the minimal $t\ge 1$ such that $\mathcal{G}^{p,k}_{\mathcal{X}}$ is a $(t,1)$-spanner of $\mathcal{G}^{p}_{\mathcal{X}}$, and the red line traces out the $(1,1)$-spanner requirement.}
\end{figure}

\section{Global Analysis: Statistics on PWSPD and Percolation}
\label{sec:Statistics}

We recall that after a suitable normalization, $\ell_{p}$ is a consistent estimator for $\L_{p}$.  Indeed, \cite{Hwang2016_Shortest, Groisman2018nonhomogeneous} prove that for any $d\ge 1$, $p>1$, there exists a constant $C_{p,d}$ independent of $n$ such that $\displaystyle\lim_{n\rightarrow\infty} \tilde{\ell}_p(x,y)= C_{p,d} \L_p(x,y)$. The important question then arises: how quickly does $\widetilde{\ell}_p$ converge? How large does $n$ need to be to guarantee the error incurred by approximating $\L_p$ with $\widetilde{\ell}_p$ is small? To answer this question we turn to results from Euclidean first passage percolation (FPP) \cite{howard1997euclidean, howard2001geodesics, auffinger2015_70, damron2016entropy}. For any discrete set $\mathcal{X}$, we let $\ell_p(x,y,\mathcal{X})$ denote the PWSPD computed in the set $\mathcal{X}\cup\{x\}\cup\{y\}$.

\subsection{Overview of Euclidean First Passage Percolation}
\label{subsec:EucFPP}

Euclidean FPP analyzes $\ell_p^p(0,z,H_1)$, where $H_1$ is a homogeneous, unit intensity Poisson point process (PPP) on $\mathbb{R}^{d}$. 

\begin{defn}  A \emph{(homogeneous) Poisson point process (PPP)} on $\mathbb{R}^{d}$ is a point process such that for any bounded subset $A\subset\mathbb{R}^{d}$, $n_{A}$ (the number of points in $A$) is a random variable with distribution $\Prob[n_{A}=m]=\frac{1}{m!}(\lambda |A|)^{m}e^{-\lambda|A|}$; $\lambda$ is the \emph{intensity} of the PPP.
\end{defn}
It is known that
\begin{align}
\label{eqn:mu}
\lim_{ \|z\| \rightarrow \infty}\frac{\ell_p^p(0,z,H_1)}{\|z\|} &= \mu \, ,
\end{align}
where $\mu=\mu_{p,d}$ is a constant depending only on $p,d$ known as the \textit{time constant}. The convergence of $\ell_p^p(0,z,H_1)$ is studied by decomposing the error into random and deterministic fluctuations, i.e.
\begin{align*}
\ell_p^p(0,z,H_1) - \mu\|z\| &=  \underbrace{\ell_p^p(0,z,H_1) - \Ex[\ell_p^p(0,z,H_1)]}_{\text{random}} + \underbrace{\Ex[\ell_p^p(0,z,H_1)] - \mu\|z\|}_{\text{deterministic}} \, .
\end{align*} 
In terms of mean squared error (MSE), one has the standard bias-variance decomposition: \[\Ex\left[\left(\ell_p^p(0,z,H_1) - \mu\|z\|\right)^2\right] = \left(\Ex[\ell_p^p(0,z,H_1)] - \mu\|z\|\right)^2+ \Var\left[\ell_p^p(0,z,H_1)\right].\]
The following Proposition is well known in the Euclidean FPP literature.
\begin{prop} 
	\label{prop:MSE_PPP}
	Let $d\geq 2$ and $p>1$. Then $\Ex\left[\left(\ell_p^p(0,z,H_1) - \mu\|z\|\right)^2\right] \leq C\|z\| \log^2(\|z\|)$ for a constant $C$ depending only on $p,d$.
\end{prop}
\begin{proof}
	By Theorem 2.1 in \cite{howard2001geodesics}, $ \Var\left[\ell_p^p(0,z,H_1)\right] \leq C\|z\|$. By Theorem 2.1 in \cite{alexander1993note}, $\left(\Ex\left[\mu\|z\|\right] - \mu\|z\|\right)^2 \leq C \|z\| \log^2(\|z\|)$.
\end{proof}
Although $\Var\left[\ell_p^p(0,z,H_1)\right] \leq C\|z\|$ is the best bound which has been proved, the fluctuation rate is known to in fact depend on the dimension, i.e. $\Var\left[\ell_p^p(0,z,H_1)\right] \sim \|z\|^{2\chi}$ for some exponent $\chi=\chi(d)\leq \frac{1}{2}$.  Strong evidence is provided in \cite{damron2016entropy} that the bias can be bounded by the variance, so the exponent $\chi$ very likely controls the total convergence rate.  

The following tail bound is also known \cite{howard2001geodesics}.

\begin{prop}
	\label{prop:TailBoundPPP}
	Let $d\geq 2, p>1, \kappaOne = \min\{1,d/p\}$, and $\kappaTwo=1/(4p+3)$. For any $\epsilon\in(0,\kappaTwo)$, there exist constants $C_0$ and $C_1$ (depending on $\epsilon$) such that for $\|z\|>0$ and $\|z\|^{\frac{1}{2}+\epsilon} \leq t \leq \|z\|^{\frac{1}{2}+\kappaTwo-\epsilon}$, $\Prob\left[\left|\ell_p^p(0,z,H_1) - \mu\|z\|\right| \geq t \right] \leq C_1\exp\left(-C_0(t /\sqrt{\|z\|})^{\kappaOne}\right)$.
	
\end{prop}

\subsection{Convergence Rates for PWSPD}
\label{subsec:ConvRates}

We wish to utilize the results in Section \ref{subsec:EucFPP} to obtain convergence rates for PWSPD. However, we are interested in PWSPD computed on a compact set with boundary $M$ and the convergence rate of $\ell_p$ rather than $\ell_p^p$. To simplify the analysis, we restrict our attention to the following idealized model.
\begin{assumption}
	\label{assump:uniformM}
	Let $M \subseteq\mathbb{R}^d$ be a convex, compact, $d$-dimensional set of unit volume containing the origin. Assume we sample $n$ points independently and uniformly from $M$, i.e. $f=1_{M}$, to obtain the discrete set $\X_n$. Let $M_\tau$ denote the points in $M$ which are at least distance $\tau$ from the boundary of $M$, i.e. $M_\tau := \{x \in M: \min_{y\in\partial M} \| x - y\| > \tau\}$.
\end{assumption}
We establish three things: (i) Euclidean FPP results apply away from $\partial M$; (ii) the time constant $\mu$ equals the constant $C_{p,d}$ in \eqref{equ:cont_limit_PD}; (iii) $\ell_p$ has the same convergence rate as $\ell_p^p$.

To establish (i), we let $H_n$ denote a homogeneous PPP with rate $\lambda=n$, and let $\ell_p(0,y,H_n)$ denote the length of the shortest path connecting $0$ and $y$ in $H_n$. We also let $\X_N = H_n \cap M$ and $\ell_p(0,y,\X_N)$ denote the PWSPD in $\X_N$; note $\Ex[|\X_N|]=n$. To apply percolation results to our setting, the statistical equivalence of $\ell_p(0,y,\X_n),$ $\ell_p(0,y,\X_N)$, and $\ell_p(0,y,H_n)$ must be established. For $n$ large, the equivalence of $\ell_p(0,y,\X_n)$ and $\ell_p(0,y,\X_N)$ is standard and we omit any analysis. The equivalence of $\ell_p(0,y,\X_N)$ and $\ell_p(0,y,H_n)$ is less clear. In particular, how far away from $\partial M$ do $0,y$ need to be to ensure these metrics are the same?  The following Proposition is a direct consequence of Theorem 2.4 from \cite{howard2001geodesics}, and essentially guarantees the equivalence of the metrics as long as $0$ and $y$ are at least distance $O(n^{-\frac{1}{4d}})$ from $\partial M$.

\begin{prop}
	\label{prop:boundary_effects}
	Let $d \geq 2$, $p > 1$, $\kappaOne = \min\{1, \frac{d}{p}\}$, $\kappaTwo = 1/(4p+3)$, and $\epsilon \in (0, \frac{\kappaTwo}{2})$, and $\tau=n^{-\frac{1}{4d} +\frac{\epsilon}{d}}\diam(M)^{\frac{3}{4}+\epsilon}$.
	Then for constants $C_{0},C_{1}$ (depending on $\epsilon$), for all $0,y \in M_\tau$, the geodesics connecting $0,y$ in $\X_N$ and $H_n$ are equal with probability at least $1-C_{1}\exp(-C_{0}(n^{\frac{1}{d}}\|y\|)^{\frac{3}{4}\epsilon\kappaOne})$, so that $\ell_p(0,y,\X_N)=\ell_p(0,y,H_n)$. 
\end{prop}

Next we establish the equivalence of $\mu_{p,d}$ (percolation time constant) and $C_{p,d}$ (PWSPD discrete-to-continuum normalization constant). 

\begin{prop}
	\label{prop:EquivalenceOfMu}
	Let $\mu_{p,d}$ be as in (\ref{eqn:mu}) and $C_{p,d}$ as in (\ref{equ:cont_limit_PD}).  Then $\mu_{p,d}^{1/p}=C_{p,d}$.
\end{prop}
\begin{proof}
	Suppose Assumption \ref{assump:uniformM} holds and choose $y\in M$ with $\|y\|=1$ and let $M$ be such that $0,y$ are not on the boundary. By Proposition \ref{prop:boundary_effects}, $\displaystyle\lim_{n\rightarrow\infty} \ell_p(0,y,\X_n) = \lim_{n\rightarrow\infty} \ell_p(0,y,H_n)$.  Let $H_1$ be the unit intensity PPP obtained from $H_n$ by rescaling each axis by $n^{1/d}$, so that $\ell_p(0,y,H_n)=n^{-\frac{1}{d}}\ell_p(0,n^{\frac{1}{d}}y,H_1)$.  For notational convenience, let $z=n^{\frac{1}{d}}y$. Then
	\begin{align*}
	\lim_{n\rightarrow\infty} \widetilde{\ell}_p(0,y,\X_n) &= \lim_{n\rightarrow\infty} \widetilde{\ell}_p(0,y,H_n) \\
	&= \lim_{n\rightarrow\infty} n^{\frac{p-1}{pd}} \ell_p(0,y,H_n) \\
	&= \lim_{n\rightarrow\infty} n^{\frac{p-1}{pd}} n^{-\frac{1}{d}}\ell_p(0,n^{\frac{1}{d}}y,H_1) \\
	&= \lim_{\|z\|\rightarrow\infty} \|z\|^{\frac{p-1}{p}} \|z\|^{-1}\ell_p(0,z,H_1) \\
	&=  \lim_{\|z\|\rightarrow\infty} \frac{ \ell_p(0,z,H_1) }{\|z\|^{1/p}}.
	\end{align*}		
	Thus, $\displaystyle C_{p,d} = C_{p,d}\L_p(0,y) =\lim_{n\rightarrow\infty} \widetilde{\ell}_p(0,y,\X_n) =\lim_{\|z\|\rightarrow\infty} \frac{ \ell_p(0,z,H_1) }{\|z\|^{1/p}} = \mu_{p,d}^{1/p}$.
\end{proof}
Finally, we bound our real quantity of interest: the convergence rate of $\widetilde{\ell}_p$ to $C_{p,d}\L_p$.

\begin{thm}
	\label{thm:MSE_of_PWSPD}
	Assume Assumption \ref{assump:uniformM}, $d\geq 2$, $\kappaTwo =1/(4p+3)$, $\tau=n^{-\frac{(1-\kappaTwo)}{4d}}\diam(M)^{\frac{3+\kappaTwo}{4}}$, $p>1$, and $0,y \in M_\tau$. Then for $n$ large enough, $\Ex\left[\left(\tilde{\ell}_{p}(0,y,\X_n)-C_{p,d}\L_p(0,y)\right)^{2}\right] \lesssim n^{-\frac{1}{d}} \log^2(n)$.
\end{thm}
\begin{proof}
To simplify notation throughout the proof we denote $\L_p(0,y)$ simply by $\L_p$. By Proposition \ref{prop:EquivalenceOfMu} and for $n$ large enough,	
\begin{align*}
\Ex\left[\left(\tilde{\ell}_{p}(0,y,\X_n)-C_{p,d}\L_p\right)^{2}\right] 
&\lesssim \Ex\left[\left(\tilde{\ell}_{p}(0,y,\X_N)-\mu^{1/p}\L_p\right)^{2}\right] =: (I) \, ,
\end{align*}
where $\X_N = H_n \cap M$ and $H_n$ is a homogeneous PPP with rate $n$.
Let $A$ be the event that the geodesics from $0$ to $y$ in $\X_N$ and $H_n$ are equal. Since we assume $\tau=n^{-\frac{(1-\kappaTwo)}{4d}}\diam(M)^{\frac{3+\kappaTwo}{4}}$, we may apply Proposition \ref{prop:boundary_effects} with $\epsilon=\kappaTwo/4$ to conclude $\Prob[A] \geq 1 - C_1\exp\left(-C_0\|y\|^{\frac{\nu}{d}}n^{\nu}\right)$ for $\nu=\frac{3\kappaTwo}{16}\min\{1,\frac{d}{p}\}$.
Conditioning on $A$, and observing  $\tilde{\ell}_{p}(0,y,\X_N)=n^{\frac{p-1}{pd}}\ell_{p}(0,y,\X_N) \leq n^{\frac{p-1}{pd}} \|y\|$, we obtain
\begin{align*}
(I) &= \Ex\left[\left(\tilde{\ell}_{p}(0,y,\X_N)-\mu^{1/p}\L_p\right)^{2}\ \big|\ A \right]\Prob[A] + \Ex\left[\left(\tilde{\ell}_{p}(0,y,\X_N)-\mu^{1/p}\L_p\right)^{2}\ \big|\ \bar{A} \right]\Prob[\bar{A}] \\
&\leq \Ex\left[\left(\tilde{\ell}_{p}(0,y,H_n)-\mu^{1/p}\L_p\right)^{2}\ \big|\ A \right] + \left(n^{\frac{2(p-1)}{pd}} \|y\|^2 + \mu^{2/p}\L_p^2\right)C_1\exp\left(-C_0\|y\|^{\frac{\nu}{d}}n^{\nu}\right) \\
&\leq \Ex\left[\left(\tilde{\ell}_{p}(0,y,H_n)-\mu^{1/p}\L_p\right)^{2}\right] + q_1
\end{align*}
where $q_1$ decays exponentially in $n$ (for the last line note that conditioning on $A$ means conditioning on the geodesics being local, which can only decrease the expected error). 

A Lipschitz analysis applied to the function $g(x)=x^{1/p}$ yields:
\begin{align*}
\left(\widetilde{\ell}_{p}(0,y,H_n)-\mu^{1/p}\L_p\right)^{2} &\leq p^{-2}\widetilde{\ell}_{p}(0,y,H_n)^{2(1-p)/p}\cdot \left(\widetilde{\ell}^p_{p}(0,y,H_n) - \mu\L_p^p\right)^2 \  .
\end{align*}
By Proposition \ref{prop:TailBoundPPP},
\begin{align}
\label{equ:tail_bound_PWSPD}
\widetilde{\ell}^p_{p}(0,y,H_n) &\geq \mu\L_p^p - \|y\|^{\frac{1}{2}+\epsilon}/n^{\frac{1}{d}\left(\frac{1}{2}-\epsilon\right)}
\end{align}
with probability at least $1-C_1 \exp\left(-C_0 \|y\|^{\epsilon\kappaOne} n^{\frac{\epsilon\kappaOne}{d}}\right)$ for any $\epsilon\in(0,\kappaTwo)$, where $\kappaOne=\min\{1,d/p\}$.  Fix $\epsilon\in(0,\kappaTwo)$ and let $B$ be the event that \eqref{equ:tail_bound_PWSPD} is satisfied.  On $B$,
\begin{align*}
\widetilde{\ell}_{p}(0,y,H_n)^\frac{2(1-p)}{p} 
&\leq (\mu^{1/p}\L_p)^\frac{2(1-p)}{p}\left(1 - \frac{\|y\|^{\frac{1}{2}+\epsilon}}{\mu\L_p^pn^{\frac{1}{d}\left(\frac{1}{2}-\epsilon\right)}}\right)^{\frac{2(1-p)}{p^2}} \\
&\leq (\mu^{1/p}\L_p)^\frac{2(1-p)}{p} \left( 1 + \frac{2(p-1)\|y\|^{\frac{1}{2}+\epsilon}}{p^2\mu\L_p^pn^{\frac{1}{d}\left(\frac{1}{2}-\epsilon\right)}} + \text{higher order terms} \right) \\
&\leq 2(\mu^{1/p}\L_p)^\frac{2(1-p)}{p},
\end{align*}
for $n$ large enough.  Note also that
\begin{small}
\begin{align*}
\Ex\left[ \left(\widetilde{\ell}_{p}(0,y,H_n)-\mu^{1/p}\L_p\right)^{2}\, |\, \bar{B} \right]\Prob[\bar{B}] &\leq \left(n^{\frac{2(p-1)}{pd}} \|y\|^2 + \mu^{2/p}\L_p^2\right) \exp\left(-C_0 \|y\|^{\epsilon\kappaOne} n^{\frac{\epsilon\kappaOne}{d}}\right) := q_2
\end{align*}
\end{small}
and $q_2$ decreases exponentially in $n$.
 We thus obtain
\begin{align*}
\Ex\left[ \left(\widetilde{\ell}_{p}(0,y,H_n)-\mu^{1/p}\L_p\right)^{2} \right]
&\leq \Ex\left[ \left(\widetilde{\ell}_{p}(0,y,H_n)-\mu^{1/p}\L_p\right)^{2}\ \big|\ B \right]\Prob[B] + q_2\\
&\leq \frac{2}{p^2}(\mu^{1/p}\L_p)^\frac{2(1-p)}{p}\Ex\left[ \left(\widetilde{\ell}^p_{p}(0,y,H_n) - \mu\L_p^p\right)^2\ \big|\ B \right] + q_2 \\
&= C \Ex\left[ \left(\widetilde{\ell}^p_{p}(0,y,H_n) - \mu\L_p^p\right)^2\right] + q_2
\end{align*}
where $C$ is a constant depending on $p,d,\|y\|$, and the last line follows since once again the expected error is lower conditioned on $B$ than unconditionally. We have thus established
\begin{align*}
\Ex\left[\left(\tilde{\ell}_{p}(0,y,\X_n)-C_{p,d}\L_p\right)^{2}\right] &\lesssim  \Ex\left[ \left(\widetilde{\ell}^p_{p}(0,y,H_n) - \mu\L_p^p\right)^2\right] + q_1 + q_2
\end{align*}
for $q_1, q_2$ exponentially small in $n$. Finally let $H_1$ be the unit intensity homogeneous PPP obtained from $H_n$ by multiplying each axis by $n^{1/d}$. 
By Proposition \ref{prop:MSE_PPP},
\begin{align*}
\Ex\left[\left(\ell_p^p(0,n^{\frac{1}{d}}y,H_1) - \mu n^{\frac{1}{d}}\|y\|\right)^2\right] &\lesssim n^{\frac{1}{d}}\|y\| \log^2(n^{\frac{1}{d}}\|y\|) \\
\Rightarrow\Ex\left[\left(n^{\frac{p}{d}}\ell_p^p(0,y,H_n) - n^{\frac{1}{d}} \mu \L_p^p\right)^2\right] &\lesssim n^{\frac{1}{d}}\|y\| \log^2(n^{\frac{1}{d}}\|y\|) \\
\Rightarrow\Ex\left[\left(n^{\frac{p-1}{d}}\ell_p^p(0,y,H_n) - \mu \L_p^p\right)^2\right] &\lesssim n^{-\frac{1}{d}}\|y\| \log^2(n^{\frac{1}{d}}\|y\|) \\
\Rightarrow\Ex\left[\left(\widetilde{\ell}_p^p(0,y,H_n) - \mu \L_p^p\right)^2\right] &\lesssim n^{-\frac{1}{d}} \log^2(n) \,.
\end{align*}
For $n$ large, the above dominates $q_1,q_2$, so that for a constant $C$ depending on $p,d,\|y\|$, \[\Ex\left[\left(\tilde{\ell}_{p}(0,y,\X_n)-C_{p,d}\L_p\right)^{2}\right] \leq C n^{-\frac{1}{d}} \log^2(n).\]
\end{proof}

\subsection{Estimating the Fluctuation Exponent}
\label{subsec:EstimatingChi}

As an application, we utilize the 1-spanner results of Section \ref{sec:Spanners} to empirically estimate the fluctuation rate $\chi(d)$. Since there is evidence that the variance dominates the bias, this important parameter likely determines the convergence rate of $\widetilde{\ell}_p$ to $\L_p$. Once again utilizing the change of variable $z=n^{\frac{1}{d}}y$, we note
\begin{align*}
\Var\left[\ell_p^p(0,z,H_1)\right] &\lesssim \|z\|^{2\chi} \iff \Var\left[\widetilde{\ell}_p(0,y,\X_n)\right] \lesssim n^{\frac{2(\chi-1)}{d}} \, ,
\end{align*}
and we estimate the right hand side from simulations. Specifically, we sample $n$ points uniformly from the unit cube $[0,1]^d$ and compute $\widetilde{\ell}_p(x,y,\X_n)$ for $x=(0.25, 0.5,\ldots,0.5)$, $y=(0.75, 0.5,\ldots,0.5)$ in a $k$NN graph on $\X_n$, with $k=\ceil{1+3 \left(\frac{4}{4^{1-1/p}-1}\right)^{d/2}\log(n)}$ 
as suggested by Theorem \ref{thm:required_kNN} (note that $f_{min}=f_{max}, \zeta=\infty, \kappa_{0}=1$ in this example). We vary $n$ from $n_{\min}=11,586$ to $n_{\max}=92,682$, and for each $n$ we estimate $\Var\left[\widetilde{\ell}_p(x,y,\X_n)\right]$ from $\NumSim$ simulations. Figure \ref{fig:var_plots} shows the resulting log-log variance plots for $d=2,3,4$ and various $p$, as well as the slopes $m$ from a linear regression. The observed slopes are related to $\chi$ by $\chi = md/2+1$, and one thus obtains the estimates for $\chi$ reported in Table \ref{tab:FluctuationRates}. See Appendix~\ref{app:Estimate_Fluctuation_Exponent} for confidence interval estimates.

These simulations confirm that $\chi$ is indeed independent of $p$. It is conjectured in the percolation literature that $\chi(d)\rightarrow 0^{+}$ as $d$ increases, with $\chi(2)=\frac{1}{3}$, $\chi(3) \approx \frac{1}{4}$, which is consistent with our results.  For $d=2$, the empirical convergence rate is thus $n^{-\frac{2}{3}}$ (not $n^{-\frac{1}{2}}$ as given in Theorem \ref{thm:MSE_of_PWSPD}), and for large $d$ one expects an MSE of order $n^{-\frac{2}{d}}$ instead of $n^{-\frac{1}{d}}$.   However estimating $\chi$ empirically becomes increasingly difficult as $d$ increases, since one has less sparsity in the $k$NN graph, and because $\chi$ is obtained from $m$ by $\chi = md/2+1$, so errors incurred in estimating the regression slopes are amplified by a factor of $d$.  Table \ref{tab:FluctuationRates} also reports the factor $n_{\max}/k$, which can be interpreted as the expected computational speed-up obtained by running the simulation in a $k$NN graph instead of a complete graph. We were unable to obtain empirical speed-up factors since computational resources prevented running the simulations in a complete graph. 

An important open problem is establishing that $\widetilde{\ell}_p$ computed from a nonuniform density enjoys the same convergence rate (with respect to $n$) as the uniform case. Although this seems intuitively true and preliminary simulation results support this equivalence, to the best of our knowledge it has not been proven, as the current proof techniques rely on ``straight line" geodesics. 

\begin{figure}[tbh]
	\centering
	\begin{subfigure}[b]{0.24\textwidth}
		\centering
		\includegraphics[width=\textwidth]{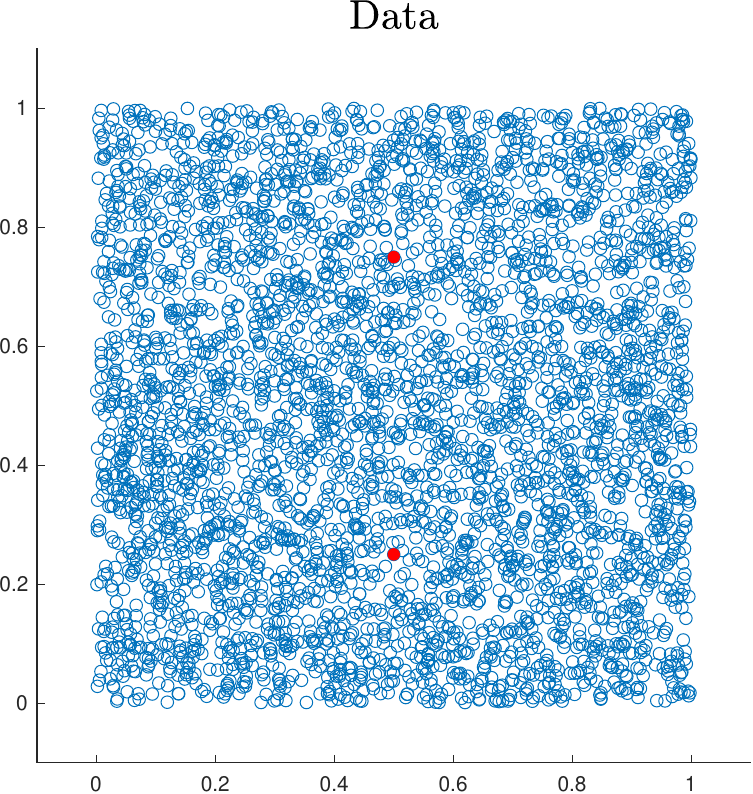}
		\caption{Data for $d=2$}
	\end{subfigure}
	\begin{subfigure}[b]{0.24\textwidth}
		\centering
		\includegraphics[width=\textwidth]{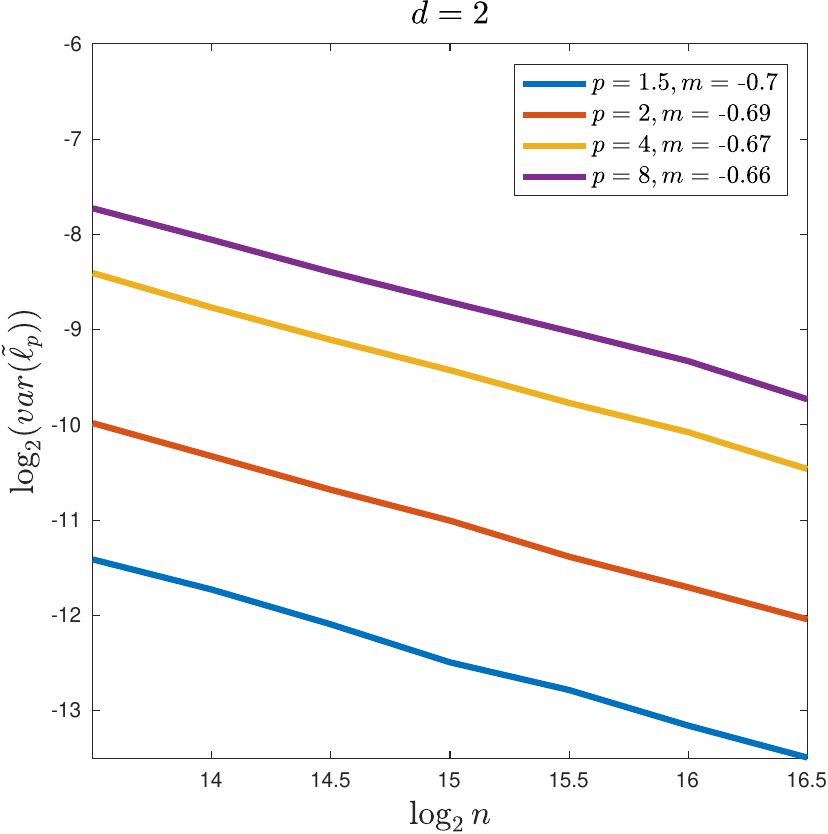}
		\caption{$d=2$}
	\end{subfigure}
	\begin{subfigure}[b]{0.24\textwidth}
		\centering
		\includegraphics[width=\textwidth]{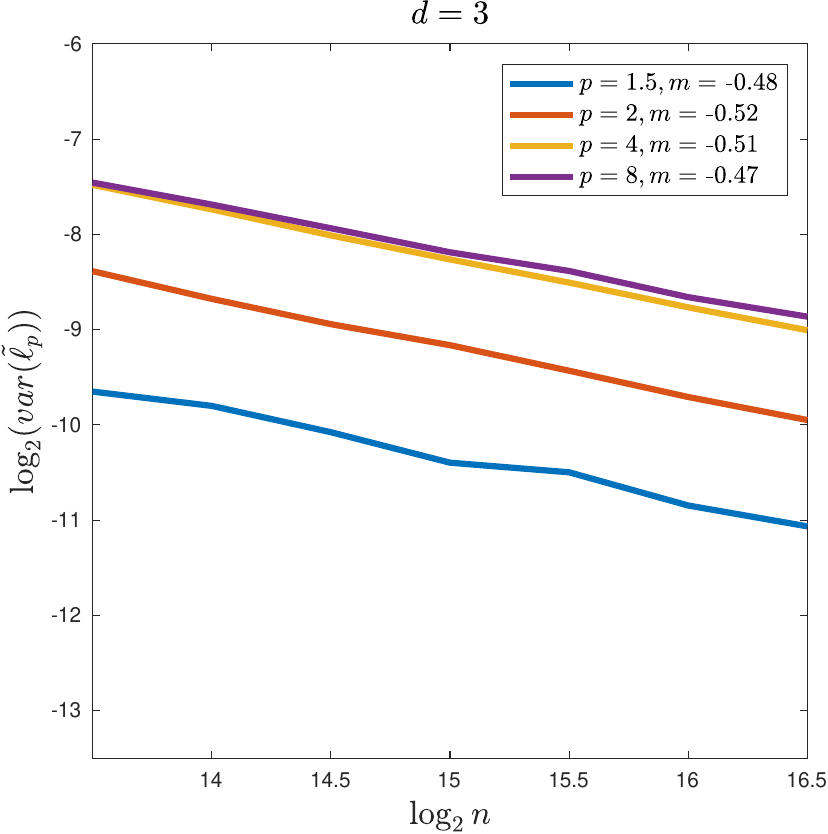}
		\caption{$d=3$}
	\end{subfigure}
	\begin{subfigure}[b]{0.24\textwidth}
		\centering
		\includegraphics[width=\textwidth]{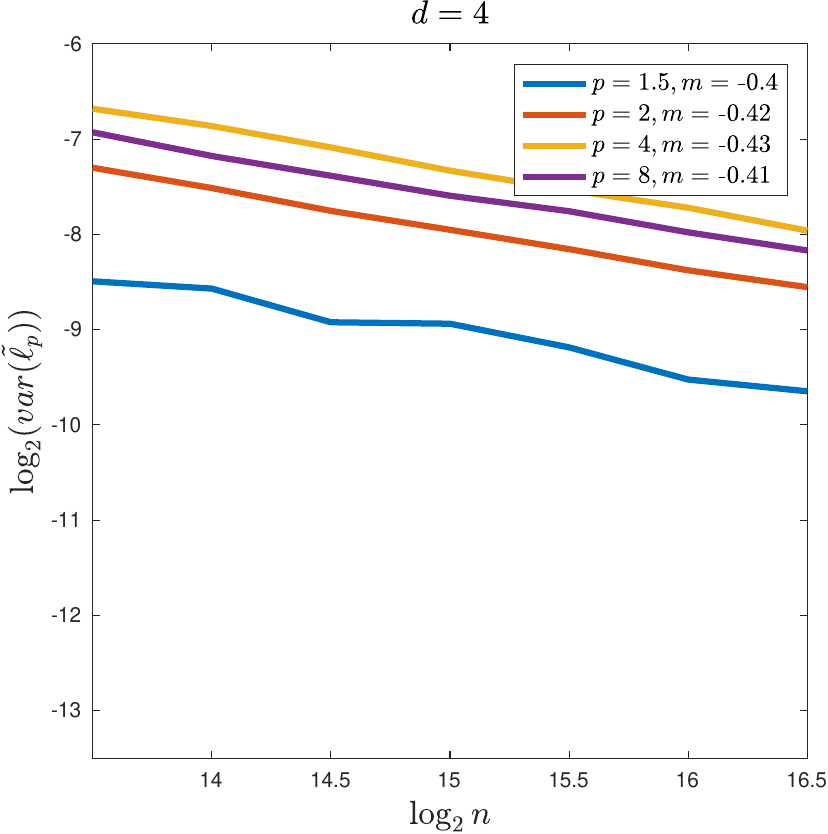}
		\caption{$d=4$}
	\end{subfigure}
	\caption{Variance plots for $\tilde{\ell}_{p}$. For each $n$, the variance was estimated from a maximum of $\NumSim=24000$ simulations, with a smaller $\NumSim$ when $p$ was small and/or the dimension was large. Specifically, when $d=2$, $\NumSim=14000$ was used for $p=1.5$; when $d=3$, $\NumSim=5000, 12000$ was used for $p=1.5,2$; when $d=4$, $\NumSim=2000, 6000, 19000$ was used for $p=1.5,2,4$. }
	\label{fig:var_plots}
\end{figure}

\begin{table}[tbh]	
	\begin{center}
		\begin{tabular}{cccc|cccc|cccc}
			\hline
			$d$ &	$p$ &	$\hat{\chi}$ & $n_{\max}/k$ & $d$ & $p$ & $\hat{\chi}$ & $n_{\max}/k$ &	$d$ & $p$ &	$\hat{\chi}$ & $n_{\max}/k$ \\ \hline
			2 & $1.5$ & 0.30 	& 394 & 3 & $1.5$ & 0.28  & 152 & 4 & $1.5$ & 0.19 & 58 \\ 
			2 & $2$ & 0.31 	&	667 & 3 & $2$ & 0.23 	& 336 & 4 & $2$ & 0.16 & 169 \\
			2 & $4$ & 0.33	&	1204 & 3 & $4$ & 0.24 	& 820 & 4 & $4$ & 0.14 	& 558 \\
			2 & $8$ & 0.34  & 1545 & 3 & $8$ & 0.29 	& 1204 & 4 & $8$ & 0.19 	& 927 \\ \hline
		\end{tabular}
		\caption{\label{tab:FluctuationRates}The slopes of $\log(n)$ versus $\Var[\tilde{\ell}_{p}]$ are shown for uniform data for different density weightings ($p$) and different dimensions $(d)$.   }
	\end{center}
\end{table}

\section{Conclusion and Future Work}

This article establishes local equivalence of PWSPD to a density-based stretch of Euclidean distance. 
We derive a near-optimal condition on $k$ for the $k$NN graph to be a 1-spanner for PWSPD, quantifying and improving the dependence on $p$ and $d$.  
Moreover, we leverage the theory of Euclidean FPP to establish statistical convergence rates for PWSPD to its continuum limit, and apply our spanner results to empirically support conjectures on the optimal dimension-dependent rates of convergence.
 
Many directions remain for future work. Our statistical convergence rates for PWSPD in Section \ref{sec:Statistics} are limited to uniform distributions. Preliminary numerical experiments indicate that these rates also hold for PWSPDs defined with varying density, but rigorous convergence rates for nonhomogeneous PPPs are lacking in the literature. 

The analysis of Section \ref{sec:LocalAnalysis} proved the local equivalence of PWSPDs with density-stretched Euclidean distances.  These results and the convergence results of Section \ref{sec:Statistics} are the first steps in a program of developing a discrete-to-continuum limit analysis for PWSPDs and PWSPD-based operators. 
A major problem is to develop conditions so that the discrete graph Laplacian (defined with $\tilde{\ell}_{p}$) converges to a continuum second order differential operator as $n\rightarrow\infty$.  A related direction is the analysis of how data clusterability with PWSPDs depends on $p$ for various random data models and in specific applications.

The numerical results of Section \ref{subsec:PWSPD_Experiments} confirm that $k \propto \log(n)$ is required for the $k$NN graph to be a 1-spanner, as predicted by theory. 
Relaxing the notion of $t$-spanners to $(t,\omega)$-spanners, as suggested in Section \ref{subsec:PWSPD_Experiments}, is a topic of future research.

Finally, the results of this article require data to be generated from a distribution supported exactly on a low-dimensional manifold $\mathcal{M}$.  An arguably more realistic setting is the noisy one in which the data is distributed only approximately on $\mathcal{M}$.  Two potential models are of interest: (i) replacing $\mathcal{M}$ with $B(\mathcal{M},\tau)=\{x\in\mathbb{R}^{D}\ | \ \text{dist}(x,\mathcal{M})\le \tau\}$ (tube model) and (ii) considering a density that \emph{concentrates} on $\mathcal{M}$, rather than being supported on it (concentration model). PWSPDs may exhibit very different properties under these two noise models, for example under bounded uniform noise and Gaussian noise, especially for large $p$. For the concentration model one expects noisy PWSPDs to converge to manifold PWSPDs for $p$ large, since the optimal PWSPD paths are density driven. Preliminary empirical results (Figure \ref{fig:HeatMap_Ridge}) suggest that when the measure concentrates sufficiently near a low-dimensional set $\mathcal{M}$, the number of nearest neighbors needed for a 1-spanner benefits from the intrinsic low-dimensional structure. For the tube model, although noisy PWSPDs will not converge to manifold PWSPDs, they will still scale according to the intrinsic manifold dimension for $\tau$ small. For both models, incorporating a denoising procedure such as local averaging \cite{Trillos2019_Local} or diffusion \cite{hein2006manifold} before computing PWSPDs is expected to be advantageous. Future research will investigate robust denoising procedures for PWSPD, computing PWSPDs after dimension reduction, and which type of noise distributions are most adversarial to PWSPD. 

\section*{Acknowledgements}

AVL acknowledges partial support from the US National Science Foundation under grant DMS-1912906. DM acknowledges partial support from the US National Science Foundation under grant DMS-1720237 and the Office of Naval Research under grant N000141712162.  JMM acknowledges partial support from the US National Science Foundation under grants DMS-1912737 and DMS-1924513. DM thanks Matthias Wink for several useful discussions on Riemannian geometry.  We thank the two reviewers and the associate editor for many helpful comments that greatly improved the manuscript.
\bibliographystyle{plain}
\bibliography{PWSPD_V2.bib}

\appendix

\section{Proofs for Section~\ref{sec:LocalAnalysis}}
\label{app:Proofs_for_LocalAnalysis}

\begin{proof}[Proof of Lemma \ref{lem:local_equivalence}]
	Let $\gamma_1(t)$ be a path which achieves $\D(x,y)$. Since $\D(x,y)\leq \epsilon(1+\kappa\epsilon^2)$, $f(\gamma_1(t))\geq \fmin(x,\epsilon)$ for all $t$.  Then:
	\begin{align*}
	\L_p^p(x,y) &\leq \int_0^1 \frac{1}{f(\gamma_1(t))^{\frac{p-1}{d}}} |\gamma_1'(t)|\ dt  \leq \frac{\D(x,y)}{\fmin(x,\epsilon)^{\frac{p-1}{d}}} \leq \frac{\epsilon(1+\kappa\epsilon^2)}{\fmin(x,\epsilon)^{\frac{p-1}{d}}} \, .
	\end{align*}
Note $y \in B_{\L_p^p}(x,\epsilon(1+\kappa \epsilon^2)/\fmin(x,\epsilon)^{\frac{p-1}{d}})$ implies $f(y)\leq \fmax(x,\epsilon)$, and thus $ \frac{\fmax(x,\epsilon)^{\frac{p-1}{d}}}{(f(x)f(y))^{\frac{p-1}{2d}}}\geq 1$, so that $	\L_p^p(x,y) \leq \frac{\D(x,y)}{\fmin(x,\epsilon)^{\frac{p-1}{d}}} \frac{\fmax(x,\epsilon)^{\frac{p-1}{d}}}{(f(x)f(y))^{\frac{p-1}{2d}}}$. This yields
	\begin{align*}
	\L_p^p(x,y)
	\leq (\rho_{x,\epsilon})^{\frac{p-1}{d}} \frac{\|x-y\|(1+\kappa\|x-y\|^2)}{(f(x)f(y))^{\frac{p-1}{2d}}} 
	\leq (\rho_{x,\epsilon})^{\frac{p-1}{d}} (1+\kappa\epsilon^2)\D_{f,\text{Euc}}(x,y),
	\end{align*}
	which proves the upper bound.
	Now let $\gamma_0(t)$ be a path achieving $\L_p^p(x,y)$; 
	note that since $\L_p^p(x,y) \leq \frac{\D(x,y)}{\fmin(x,\epsilon)^{\frac{p-1}{d}}}$,
	the path $\gamma_0$ is contained in $B_{\L_p^p}(x,\epsilon(1+\kappa \epsilon^2)/\fmin(x,\epsilon)^{\frac{p-1}{d}})$. 
	Thus
	\begin{align*}
	\L_p^p(x,y) &= \int_0^1 \frac{1}{f(\gamma_0(t))^{\frac{p-1}{d}}} |\gamma_0'(t)|\ dt 
	\geq \frac{\D(x,y)}{\fmax(x,\epsilon)^{\frac{p-1}{d}}} 
	\geq \frac{\D(x,y)}{\fmax(x,\epsilon)^{\frac{p-1}{d}}} \cdot \frac{\fmin(x,\epsilon)^{\frac{p-1}{d}}}{(f(x)f(y))^{\frac{p-1}{2d}}}
	\end{align*}
	so that
	\begin{align*}
	\L_p^p(x,y)&\geq \frac{\D(x,y)}{(\rho_{x,\epsilon})^{\frac{p-1}{d}}(f(x)f(y))^{\frac{p-1}{2d}}} 
	\geq \frac{\|x-y\|}{(\rho_{x,\epsilon})^{\frac{p-1}{d}}(f(x)f(y))^{\frac{p-1}{2d}}}
	= \frac{1}{(\rho_{x,\epsilon})^{\frac{p-1}{d}}}\D_{f,\text{Euc}}(x,y).
	\end{align*}

\end{proof}

\section{Proofs for Section~\ref{sec:Spanners}}
\label{app:Proofs_for_Spanners}

\begin{proof}[Proof of Lemma~\ref{lemma:Volume_of_Region}] Let $s := \|x - y\|$ and choose a coordinate system $x^{(1)},\ldots, x^{(n)}$ such that $y = (-s/2,0,\ldots 0)$, $x = (s/2,0,\ldots,0)$ and $x_{M} = \mathbf{0}$. $\mathcal{D}_{\alpha,p}(x,y)$ is now the interior of: 
	\begin{small}
		\[\left( \left(x^{(1)}+\frac{s}{2}\right)^{2} + (x^{(2)})^{2} + \ldots + (x^{(n)})^{2}\right)^{p/2} + \left( \left(x^{(1)} - \frac{s}{2}\right)^{2} + (x^{(2)})^{2} + \ldots + (x^{(n)})^{2}\right)^{p/2} = \alpha s^{p}.\]
	\end{small}
	In spherical coordinates the boundary of this region may be expressed as:
	\begin{align}
		\label{equ:region}
		\left(r^2+ sr\cos \theta_1+ s^2/4\right)^{p/2} + \left(r^2-s r\cos \theta_1+s^2/4\right)^{p/2} &=\alpha s^p
	\end{align}
	where $(x^{(1)})^{2} + \ldots +(x^{(n)})^2 = r^2$ and $x_1 = r \cos\theta_1$. Define $r=H(\theta_1)$ as the unique positive solution of (\ref{equ:region}).  Implicitly differentiating in $\theta_1$ yields
	\begin{small}
		\begin{align*}
			&\frac{p}{2}\left(r^2+s r\cos(\theta_1)+\frac{s^2}{4}\right)^{\frac{p}{2}-1}\left(2r\frac{dr}{d\theta_1}-s r \sin(\theta_1)+s \cos(\theta_1)\frac{dr}{d\theta_1}\right) \\
			+&\frac{p}{2}\left(r^2-s r\cos(\theta_1)+\frac{s^2}{4}\right)^{\frac{p}{2}-1}\left(2r\frac{dr}{d\theta_1}+sr \sin(\theta_1)-s \cos(\theta_1)\frac{dr}{d\theta_1}\right) = 0 \, .
		\end{align*}
	\end{small}
	Solving for $\frac{dr}{d\theta_1}$ and setting the result to 0 yields
	\begin{small}
		\begin{align*}
			\left[ \left(r^2+s r\cos(\theta_1)+\frac{s^2}{4}\right)^{\frac{p-2}{2}} - \left(r^2-s r\cos(\theta_1)+\frac{s^2}{4}\right)^{\frac{p-2}{2}} \right] \sin(\theta_1):=\text{(I)}\cdot\text{(II)} &=0.
		\end{align*}
	\end{small}
	Thus we obtain two solutions to $\frac{dr}{d\theta_1}=0$:
	\begin{align*}
		\text{(I)} = 0 \Rightarrow \cos(\theta_1) = 0 \Rightarrow \theta_1 =\frac{\pi}{2} \qquad \text{(min.)} \quad 
		\text{(II)} = 0 \Rightarrow \sin(\theta_1) = 0 \Rightarrow \theta_1 =0 \qquad \text{(max.)}
	\end{align*}
	Thus the minimal radius occurs when $\theta_1 = \frac{\pi}{2}$. Substituting $\theta_1 = \frac{\pi}{2}$ into (\ref{equ:region}) yields:\[ r = s \sqrt{\alpha^{2/p}\big/4^{1/p} - 1/4} = \|x - y\|\sqrt{\alpha^{2/p}\big/4^{1/p} - 1/4}.\] Hence $B(x_{M},r)\subset \mathcal{D}_{\alpha,p}(x_{i},x_{j})$, as desired. To see $\mathcal{D}_{\alpha,p}(x,y) \subset B(x,r)$ observe that if $z \notin B(x,r)$ then:
	\begin{equation}
		\|x-z\| > r = \|x-y\| \Rightarrow \|x-z\|^{p} > \alpha\|x-y\|^{p} \quad \text{ for all } \alpha \in (0,1] \text{ and } p \geq 1
	\end{equation}
	hence $z$ cannot be in $\mathcal{D}_{\alpha,p}(x,y)$.
\end{proof}


\begin{lem}
\label{lem:Ball_radius_reach}
With assumptions and notation as in Theorem \ref{theorem:EuclideanCase}, \begin{small}$B\left(\tilde{x}_{M},r_{2}^{\star}\right)\subset B(x_{M},r_1^{\star})$\end{small}.
\end{lem}

\begin{proof}
By \cite[Lemma 1]{boissonnat2019reach}, $\|x_{M} - \tilde{x}_{M}\| \leq \zeta - \sqrt{\zeta^2 - r^2/4}< r^{2}/(4\zeta)$.  Now, suppose $y\in B(\tilde{x}_{M},r_{2}^{\star})$.  Then 
\[\|x_{M}-y\|\le \|x_{M}-\tilde{x}_{M}\|+\|\tilde{x}_{M}-y\| \le r^{2}\big/(4\zeta)+r\left(\sqrt{1/4^{1/p} - 1/4}- r/(4\zeta)\right)=r_{1}^{\star},\]
so that $y\in B(x_{M},r_{1}^{\star})$, as desired.

\end{proof}

\begin{lem}
\label{lem:r_star_vol_lower_bound}
With notation and assumptions as in Theorem~\ref{theorem:EuclideanCase}:
\begin{align}
	\Prob\left[x_{i_j}\in B(\tilde{x}_{M},r_{2}^{\star})  \ | \ x_{i_j}\in B(x,r)\right] \geq \frac{3}{4}\kappa_{0}^{-2}\frac{f_{\min}}{f_{\max}}\left(\frac{1}{4^{1/p}} - \frac{1}{4}\right)^{d/2}.
\end{align}
\end{lem}

\begin{proof}
By the definition of conditional probability and $B(\tilde{x}_{M},r_{2}^{\star})\subset B(x,r)$:
\begin{align}
\Prob\left[x_{i_j}\in B(\tilde{x}_{M},r_{2}^{\star})  \ | \ x_{i_j}\in B(x,r)\right] = \int_{B(\tilde{x}_{M},r_{2}^{\star})\cap \mathcal{M}}f\bigg/\displaystyle\int_{B(x,r)\cap\mathcal{M}}f.
\label{eq:App_Cond_Prob}
\end{align}
By Definition~\ref{defn:V}, $\int_{B(\tilde{x}_{M},r_{2}^{\star})\cap \mathcal{M}}f \ge f_{\min}\text{vol}\left(B(\tilde{x}_{M},r_{2}^{\star})\cap \mathcal{M}\right) \geq f_{\min}\kappa_{0}^{-1}(r_2^{\star})^{d}\text{vol}\left(B(0,1)\right)$ and $\int_{B(x,r)\cap\mathcal{M}}f\leq f_{\max}\text{vol}\left(B(x,r)\cap\mathcal{M}\right)\leq f_{\max}\kappa_{0} r^{d}\text{vol}\left(B(0,1)\right)$. Returning to \eqref{eq:App_Cond_Prob}:
\begin{equation}
\Prob\left[x_{i_j}\in B(\tilde{x}_{M},r_{2}^{\star})  \ | \ x_{i_j}\in B(x,r)\right] \geq \kappa_{0}^{-2}\frac{f_{\min}}{f_{\max}}\left(\frac{r_2^{\star}}{r}\right)^{d}.
\label{eq:App_Cond_Prob_2}
\end{equation}
The result follows by noting 

\[\left(\frac{r_{2}^{\star}}{r}\right)^{d} = \left(\sqrt{\frac{1}{4^{1/p}}-\frac{1}{4}}-\frac{r}{4\zeta}\right)^{d} \ge\left(\sqrt{\frac{1}{4^{1/p}}-\frac{1}{4}}-\frac{1}{4d}\sqrt{\frac{1}{4^{1/p}}-\frac{1}{4}}\right)^{d}\ge \frac{3}{4}\left(\frac{1}{4^{1/p}}-\frac{1}{4}\right)^{d/2}.\]

\end{proof}

\section{Local Analysis: Proof of Corollary 2.4}
\label{app:Proof_of_Cor}
\begin{proof}
	First note that by Theorem 2.3
	\begin{align*}
		\left|\L_p^2(x,y) - \D_{f,\text{Euc}}^{2/p}(x,y)\right| &= \left|\left( \L_p(x,y) - \D_{f,\text{Euc}}^{1/p}(x,y)\right) \left( \L_p(x,y) + \D_{f,\text{Euc}}^{1/p}(x,y)\right)\right| \\
		&\leq \left| \left(C_1\epsilon^{1+\frac{1}{p}}+C_2\epsilon^{2+\frac{1}{p}}+O(\epsilon^{3+\frac{1}{p}}) \right) \frac{\epsilon^{1/p}(1+O(\epsilon^2))}{\fmin^{\frac{p-1}{pd}}}\right| \\
		&=\tilde{C}_1 \epsilon^{1+2/p}+\tilde{C}_2\epsilon^{2+2/p}+O(\epsilon^{3+2/p}) \, .
	\end{align*}
	Thus
	\begin{align*}
		\frac{\left|h_{\frac{1}{p}}\left(\L_p^p(x,y)/\epsilon\right) - h_{\frac{1}{p}}\left(\D_{f,\text{Euc}}(x,y)/\epsilon\right) \right|}{h_{\frac{1}{p}}\left(\L_p^p(x,y)/\epsilon\right)} &= \left|1 - \exp\left(-\frac{\D_{f,\text{Euc}}^{2/p}(x,y)}{\epsilon^{2/p}} + \frac{\L_p^2(x,y)}{\epsilon^{2/p}} \right) \right| \\
		&= \left| 1 - \exp( \pm [\tilde{C}_1 \epsilon+\tilde{C}_2\epsilon^{2}+O(\epsilon^{3}) ] ) \right| \\
		&\leq \tilde{C}_1 \epsilon + \left(\tilde{C}_2 +\frac{1}{2}\tilde{C}_1^2\right)\epsilon^2 + O(\epsilon^3).
	\end{align*}
\end{proof}	

\section{Additional Clustering Results}
\label{app:Additional_Clustering_Results}

\subsection{Spectral Clustering with Self-Tuning}

Clustering results for Euclidean SC with self-tuning (ST) are shown in Figure \ref{fig:SC_ST}.  Similarly to Euclidean SC and Euclidean SC with the diffusion maps normalization, SC+ST can cluster well given $K$ a priori but struggles to simultaneously learn $K$ and cluster accurately. 

\begin{figure}[htbp!]
	\centering
	\begin{subfigure}{0.32\textwidth}
		\centering
		\includegraphics[width=\textwidth]{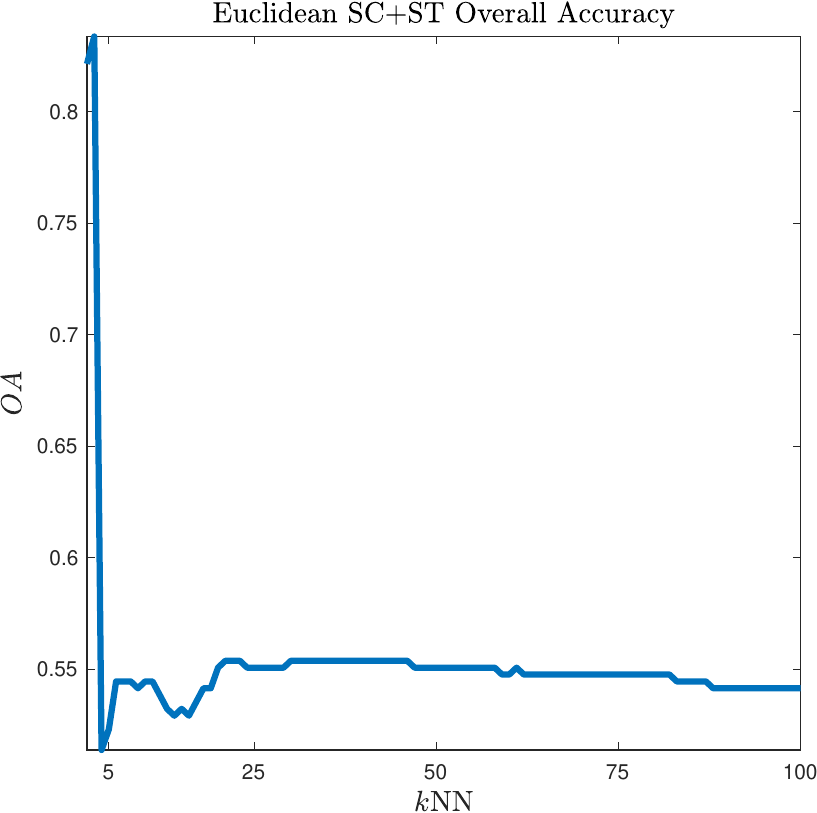}
		\subcaption{\tiny{OA, Euc. SC+ST, Two Rings}}
	\end{subfigure}
	\begin{subfigure}{0.32\textwidth}
		\centering
		\includegraphics[width=\textwidth]{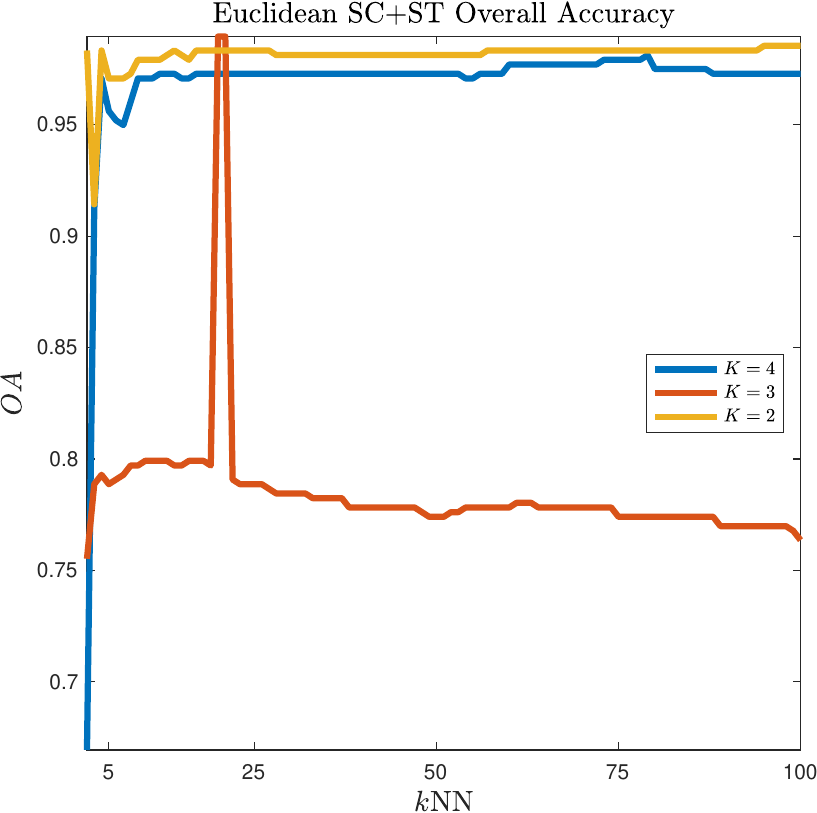}
		\subcaption{\tiny{OA, Euc. SC+ST, Long Bottleneck}}
	\end{subfigure}
	\begin{subfigure}{0.32\textwidth}
		\centering
		\includegraphics[width=\textwidth]{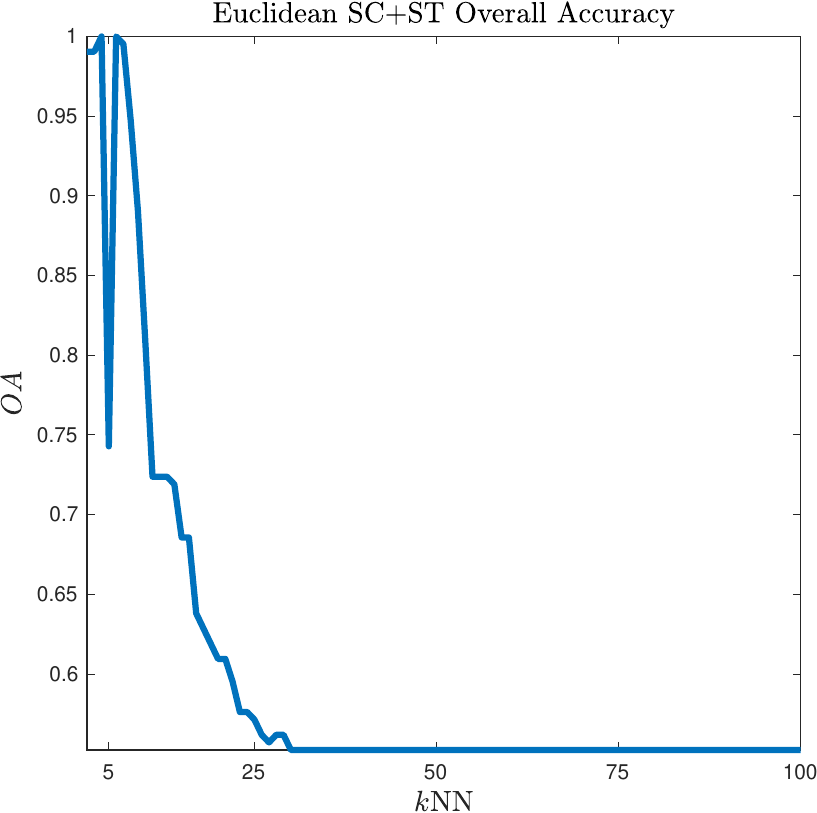}
		\subcaption{\tiny{OA, Euc. SC+ST, Short Bottleneck}}
	\end{subfigure}
	\begin{subfigure}{0.32\textwidth}
		\centering
		\includegraphics[width=\textwidth]{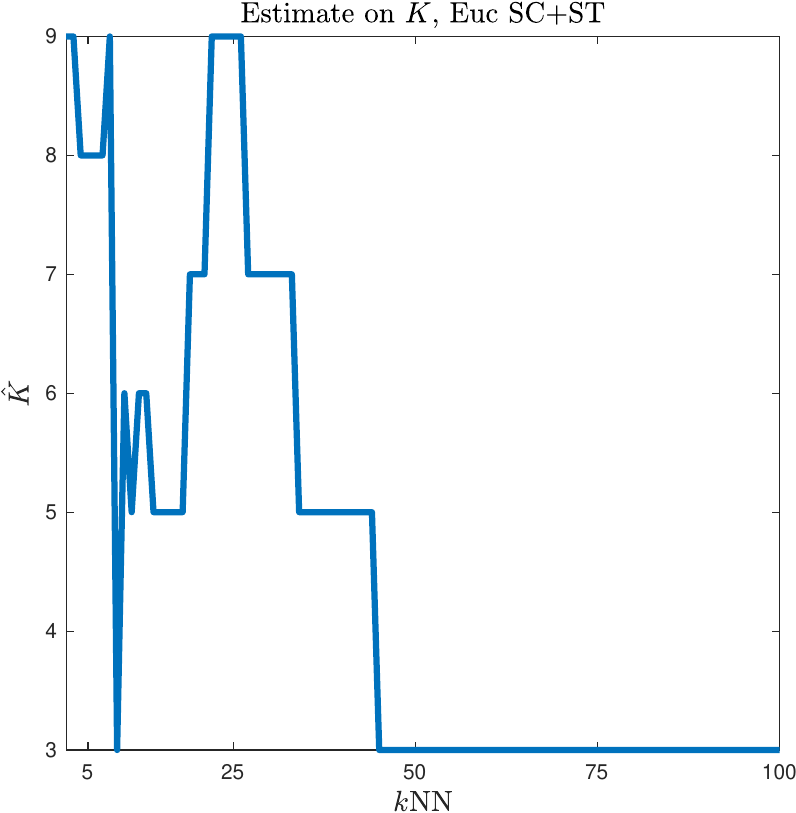}
		\subcaption{\tiny{$\hat{K}$, Euc. SC+ST, Two Rings}}
	\end{subfigure}
	\begin{subfigure}{0.32\textwidth}
		\centering
		\includegraphics[width=\textwidth]{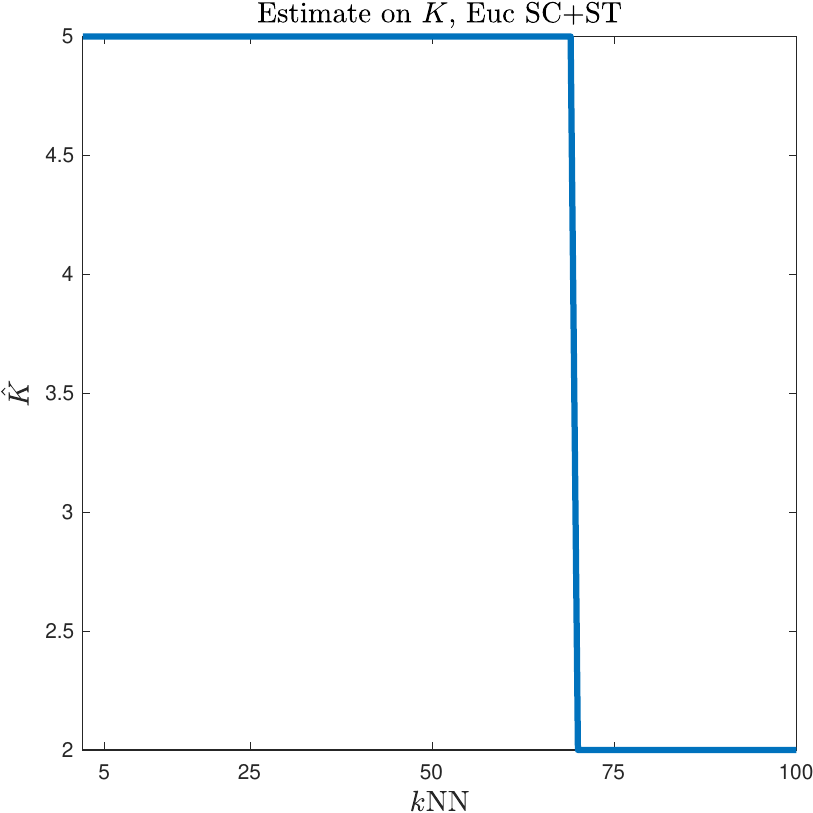}
		\subcaption{\tiny{$\hat{K}$, Euc. SC+ST, Long Bottleneck}}
	\end{subfigure}
	\begin{subfigure}{0.32\textwidth}
		\centering
		\includegraphics[width=\textwidth]{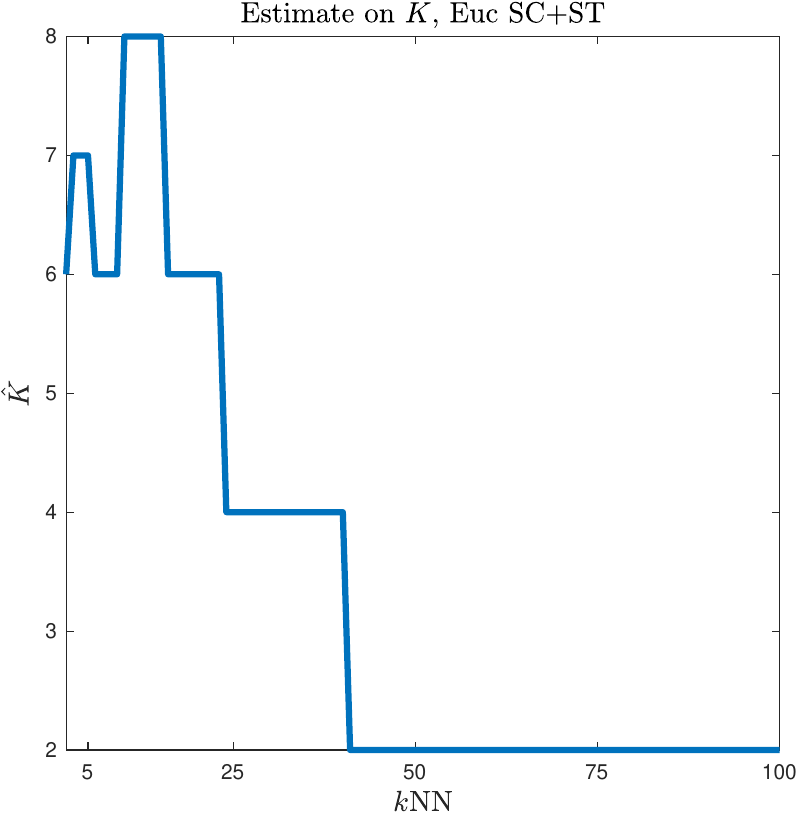}
		\subcaption{\tiny{$\hat{K}$, Euc. SC+ST, Short Bottleneck}}
	\end{subfigure}
	\caption{\label{fig:SC_ST}Clustering results with self-tuning SC.}
\end{figure}

\subsection{Clustering with Fiedler Eigenvector}

Clustering results on the Two Rings and Short Bottleneck data (both of which have $K=2$) appear in Figure \ref{fig:FiedlerPlots}.  The results are very similar to running $K$-means on the second eigenvector of $L$.

\begin{figure}[h]
	\centering
	\begin{subfigure}[t]{0.3\textwidth}
		\centering
		\includegraphics[width=\textwidth]{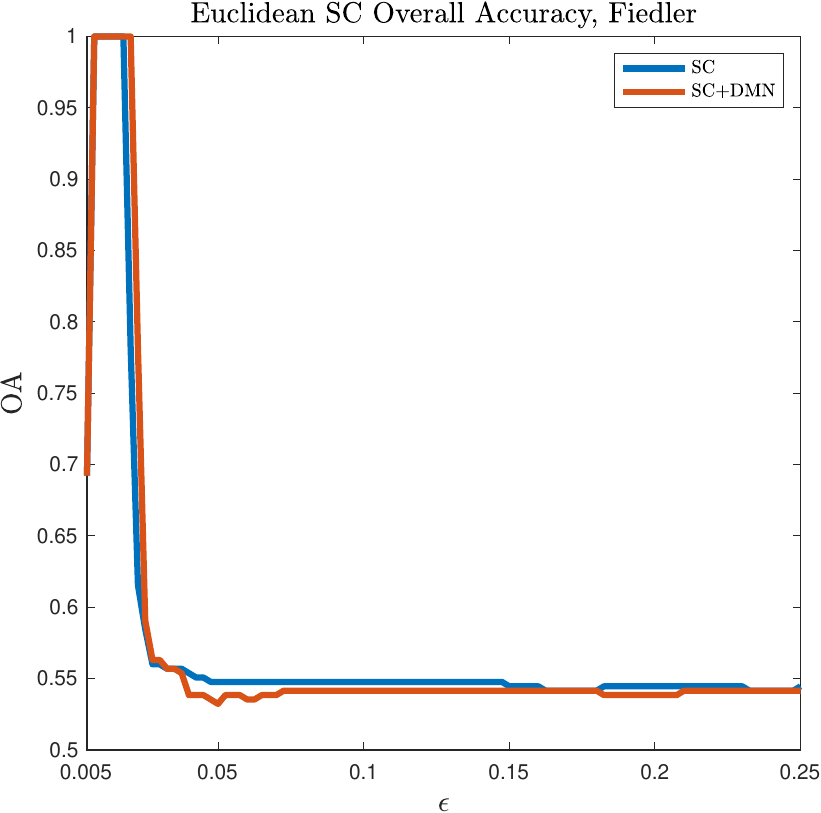}
		\subcaption{\tiny{OA, Euc. SC, Two Rings}}
	\end{subfigure}
	\begin{subfigure}[t]{0.3\textwidth}
		\centering
		\includegraphics[width=\textwidth]{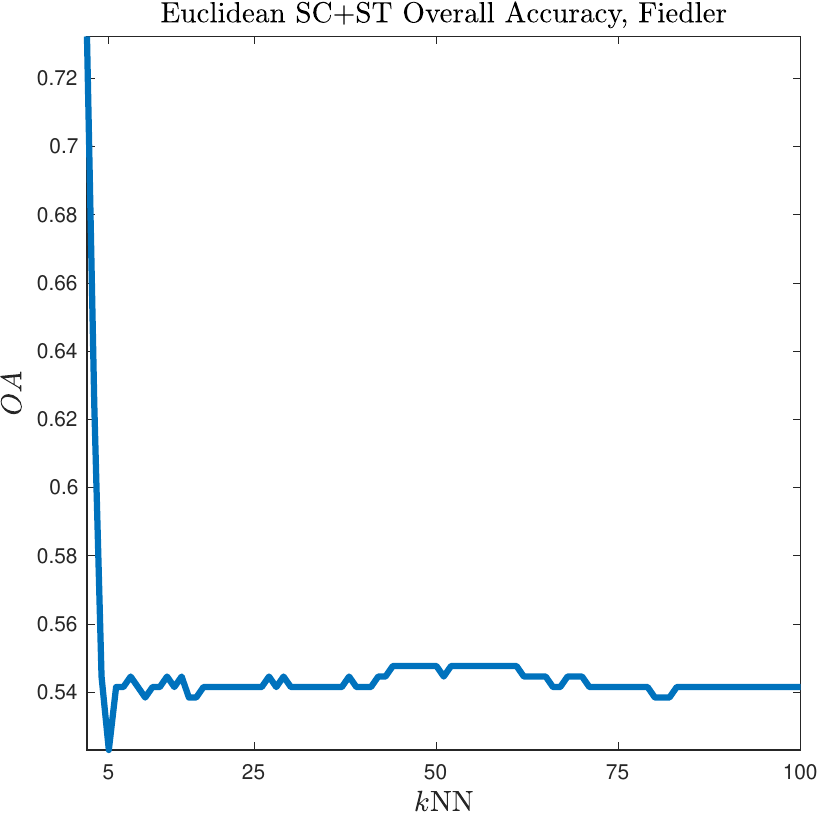}
		\subcaption{\tiny{OA, Euc. SC+ST, Two Rings}}
	\end{subfigure}
	\begin{subfigure}[t]{0.33\textwidth}
		\centering
		\includegraphics[width=\textwidth]{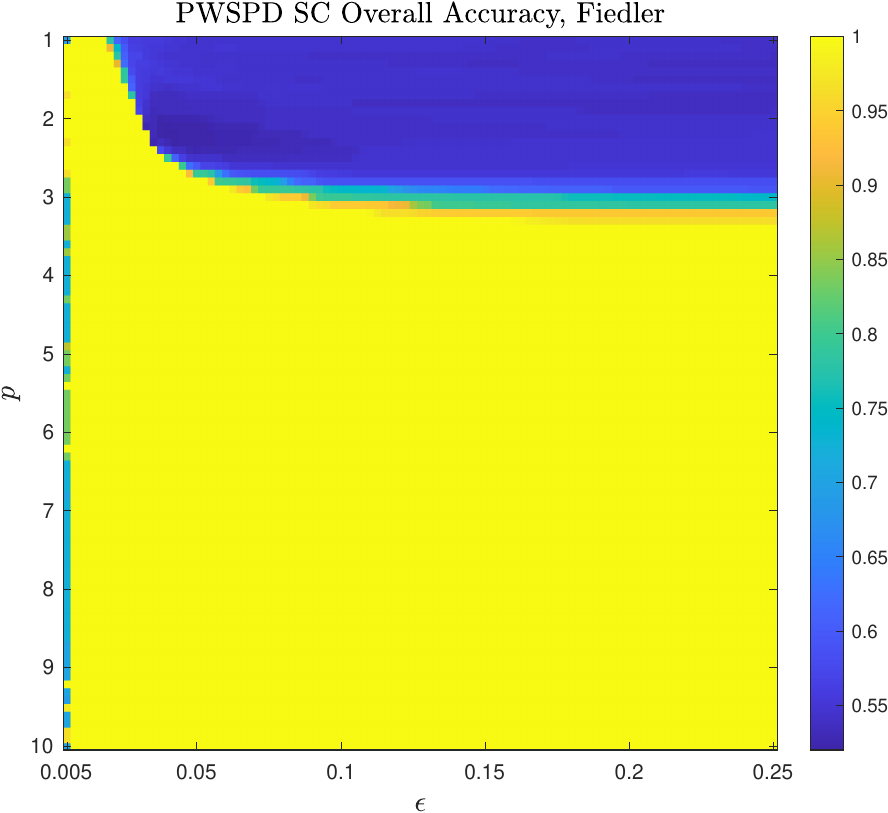}
		\subcaption{\tiny{OA, PWSPD SC, Two Rings}}
	\end{subfigure}
	\begin{subfigure}[t]{0.3\textwidth}
		\centering
		\includegraphics[width=\textwidth]{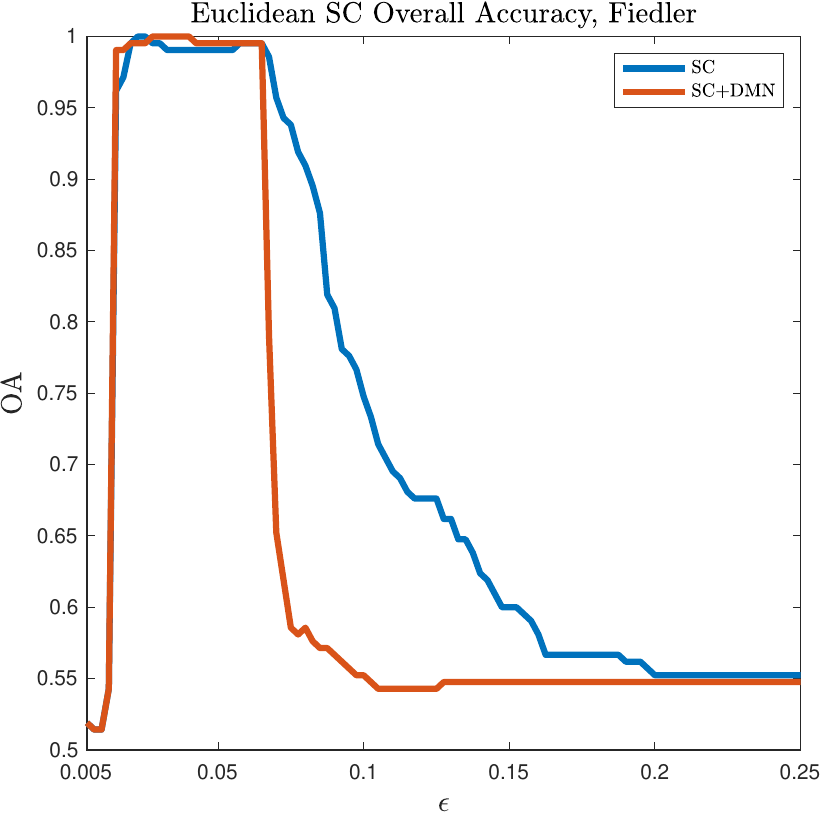}
		\subcaption{\tiny{OA, Euc. SC, Short Bottleneck}}
	\end{subfigure}
	\begin{subfigure}[t]{0.3\textwidth}
		\centering
		\includegraphics[width=\textwidth]{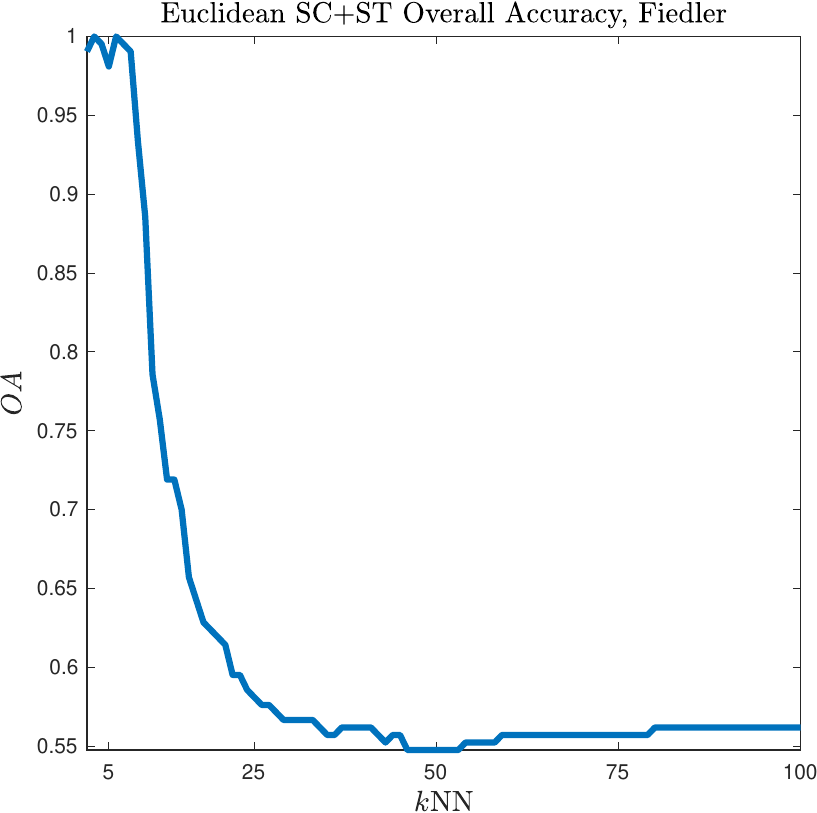}
		\subcaption{\tiny{OA, Euc. SC+ST, Short Bottleneck}}
	\end{subfigure}
	\begin{subfigure}[t]{0.33\textwidth}
		\centering
		\includegraphics[width=\textwidth]{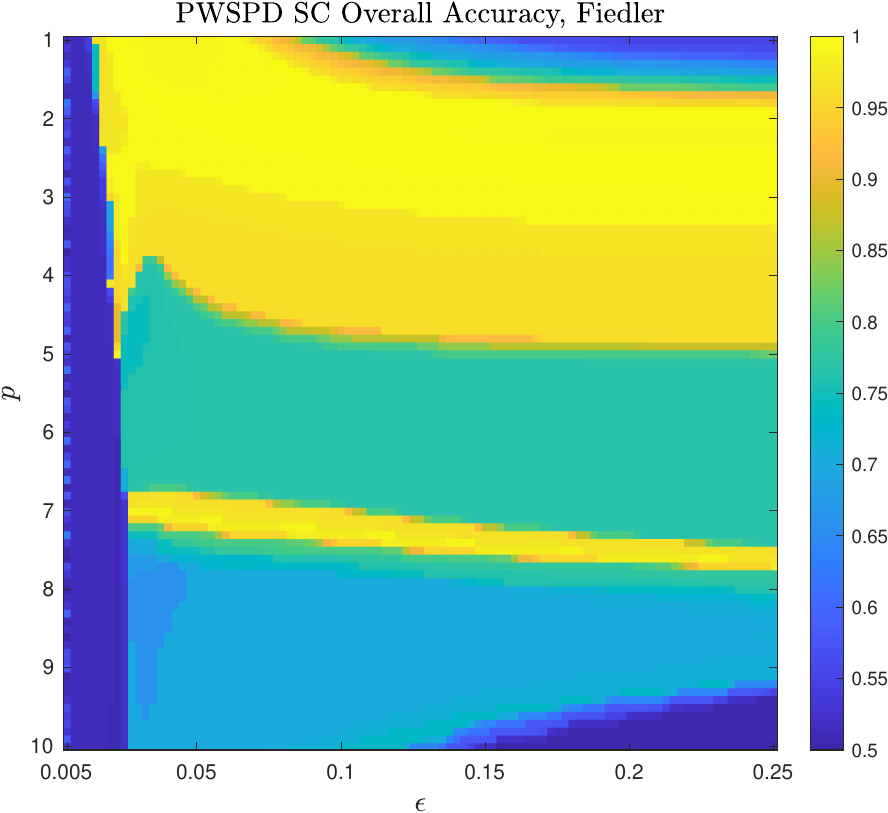}
		\subcaption{\tiny{OA, Euc. SC, Short Bottleneck}}
	\end{subfigure}
	\caption{\label{fig:FiedlerPlots}Clustering with Fiedler eigenvector.}
\end{figure}

\section{PWSPD Spanners: Intrinsic Path Distance}\label{subsec:PWSPD_Manifold}

In this section we assume that $\mathcal{M}$ is a compact Riemannian manifold with metric $g$, but we {\em do not} assume that $\mathcal{M}$ is isometrically embedded in $\mathbb{R}^{D}$.  Let us first establish some notation. For precise definitions of terms in italics, we refer the reader to \cite{Petersen2006}. 

\begin{itemize}
     \item For any $v,w \in T_{x}\mathcal{M}$, $R_{x}(v,w)$ denotes the {\em Riemannian curvature} while $K_{x}(v,w)$ denotes the {\em sectional curvature}. In this notation, $R_{x}(v,w)$ is a matrix while $K_{x}(v,w)$ is a scalar. These notions of curvature are related:
    \begin{equation}
 \langle R_{x}(w,v)w,v\rangle = K_{x}(v,w)\left( \|v\|^{2}\|w\|^{2} - \langle v,w\rangle \right).
 \label{eq:Rearranged_Sec_Curvature_main}
\end{equation}
We shall drop the ``$x$'' in the subscript when it is clear from context. Let $K_{\min}$ be such that $K_{x}(v,w) \geq K_{\min}$ for all $x\in\mathcal{M}$ and $v,w\in T_{x}\mathcal{M}$. Because $\mathcal{M}$ is compact, such a $K_{\min}$ exists. Similarly, $S_{x}$ denotes the {\em scalar curvature} and $S_{\min} = \displaystyle\min_{x\in\mathcal{M}}S_{x}$ while $S_{\max} = \displaystyle\max_{x\in\mathcal{M}}S_{x}$. 
	\item For $r > 0$, define $V_{\min}(r) := \displaystyle\min_{x\in\mathcal{M}}\text{vol}(B_{\D}(x,r))$, $V_{\max}(r) = \displaystyle\max_{x\in\mathcal{M}}\text{vol}(B_{\D}(x,r))$. Both quantities can be expressed as a Taylor series in $r$, with coefficients depending on the curvature of $\mathcal{M}$ \cite{gray1974volume}:
	\begin{align}
V_{\min}(r)  &= \left(1 - \frac{S_{\max}}{6(d+2)}r^{2} + O(r^{4})\right) \text{vol}(B(0,1))r^{d}, \label{eq:SmallBall1}\\
V_{\max}(r) &= \left(1 - \frac{S_{\min}}{6(d+2)}r^{2} + O(r^{4})\right) \text{vol}(B(0,1))r^{d}.
\label{eq:SmallBall2}
\end{align}	
    \item For any $x\in\mathcal{M}$, $\exp_{x}: T_{x}\mathcal{M}\to \mathcal{M}$ denotes the {\em exponential map.} By construction, $\D(x,\exp_{x}(v)) = g_{x}(v,v)$. The exponential map is used to construct {\em normal coordinates}.
    \item The {\em injectivity radius} at $x$ is denoted by $\text{inj}(x)$, while $\displaystyle\text{inj}(\mathcal{M})$ denotes the injectivity radius of $\mathcal{M}$.  Because $\mathcal{M}$ is closed, the injectivity radius is bounded away from 0.
\end{itemize}

\begin{prop}
In normal coordinates the metric has the following expansion: \[g_{kl}(x) = \delta_{kl} + \frac{1}{3} R_{ijkl}x^{i}x^{j} + O(\|x\|^3),\]
where $\delta_{kl} = 1$ if $k=l$ and is zero otherwise. Moreover, for any $v,w\in B(0,1)$ and $ r\leq \text{inj}(\mathcal{M})$:
\begin{equation}
\D\left(\exp_p(rv),\exp_p(rw)\right)^2 = r^2\|v - w\|^{2} + \frac{r^4}{3}\langle R(w,v)w,v\rangle +  O(r^5).
\label{eq:Metric_Expansion}
\end{equation}
\end{prop}

\begin{proof}
See the discussion in \cite{Petersen2006} above Theorem 5.5.8 and Exercises 5.9.26 and 5.9.27.
\end{proof}Combining (\ref{eq:Rearranged_Sec_Curvature_main}) and \eqref{eq:Metric_Expansion} yields:\[\D\left(\exp_p(rv),\exp_p(rw)\right)^2  = r^2\|v - w\|^{2} + \frac{r^4K(v,w)}{3}\left(\|v\|^{2}\|w\|^{2} - \langle v,w\rangle^{2}\right)   + O(r^{5}).\]
For any $x_i\in\mathcal{X}$ let $r_{i,k}$ denote the distance to the $k$NN of $x_i$. Because $r_{i,k} \to0^{+}$ as $n\to\infty$, we have that $r_{i,k} \leq \text{inj}(\mathcal{M})$ for all $x_i$, almost surely for $n$ large enough. \\

The proof of Theorem \ref{thm:k_nn_Riemannian_case} proceeds in a similar manner to the proof of Theorem 3.9. However, care must be taken to account for the curvature. Note that for any quantity $c := c(n)$ implicitly dependent on $n$, we shall say that $c = o(1)$ if for any $\epsilon > 0 $ there exists an $N$ such that for all $n > N$ we have that $c(n) < \epsilon$.  Let $\mathcal{G}_{\mathcal{M},\mathcal{X}}^{p}$ denote the complete graph on $\mathcal{X}$ with edge weights $\D(x_i,x_j)^{p}$ while $\mathcal{G}^{p,k}_{\mathcal{M},\mathcal{X}}$ shall denote the $k$NN subgraph of $\mathcal{G}_{\mathcal{M},\mathcal{X}}^{p}$.  

\begin{thm}
\label{thm:k_nn_Riemannian_case}
Let $(\mathcal{M},g)$ be a closed and compact Riemannian manifold. Let $\mathcal{X} = \{x_i\}_{i=1}^{n}$ be drawn i.i.d. from $\mathcal{M}$ according to a probability distribution with continuous density $f$ satisfying $0<f_{\min}\le f(x) \leq f_{\max}$ for all $x \in \mathcal{M}$.  For $p > 1$, $n$ sufficiently large, and
\begin{equation}
k \geq 3\left[1 + o(1)\right]\left[\frac{f_{\max}}{f_{\min}}\right] \left[\frac{4}{4^{1-1/p}(1-o(1))^{2/p} - 1}\right]^{d/2}\log(n),
\end{equation}
$\mathcal{G}^{p,k}_{\mathcal{M},\mathcal{X}}$ is a $1$-spanner of $\mathcal{G}^{p}_{\mathcal{M},\mathcal{X}}$ with probability at least $1-1/n$.
\end{thm} 

\begin{proof}
As in Theorem~3.9, we prove this by showing that every edge of $\mathcal{G}^{p}_{\mathcal{M},\mathcal{X}}$ not contained in $\mathcal{G}^{p,k}_{\mathcal{M},\mathcal{X}}$ is not critical. As before, w.h.p. $\ell_{\mathcal{M},p}(x,y) \to 0$ uniformly in $x,y$ \cite{Hwang2016_Shortest}. So, let $n$ be large enough so that $\Prob\left[\ell_{\mathcal{M},p}(x,y)\le \text{inj}(\mathcal{M}) \text{ for all } x,y\in \mathcal{X}\right]\ge  \left(1-\frac{1}{2n}\right)$.  

Pick any $x,y\in\mathcal{X}$ which are not $k$NNs and let $r:=\D(x,y)$.  If $r> \text{inj}(\mathcal{M})$, then $\ell_{p}(x,y)< \D(x,y)$ and thus the edge $\{x,y\}$ is not critical.  So, suppose without loss of generality in what follows that $r\le \text{inj}(\mathcal{M})$. Let $x_{i_1},\ldots, x_{i_k}$ denote the $k$NNs of $x$. Because $y$ is not a $k$NN of $x$, $\D(x,x_{i_j}) \leq \D(x,y) = r$ for $j=1,\ldots, k$. We show $\{x,y\}$ is not critical by showing there exists an $\ell\in\{1,\dots, k\}$ such that $\D(x,x_{i_\ell})^{p} + \D(x_{i_\ell}, y)^{p} \leq \D(x, y)^{p}$ with probability at least $1-1/n$. 

Let $x_{M}$ be the midpoint of the (shortest) geodesic from $x$ to $y$. As $r \leq \text{inj}(\mathcal{M})$ the exponential map $ \exp := \exp_{x_M}$ is a diffeomorphism onto $B_{\D}(x_M,r)$. Choose Riemannian normal coordinates centered at $x_{M}$ such that $y = \exp_{x_M}(\frac{r}{2}e_1)$ and $x = \exp_{x_M}(-\frac{r}{2}e_1)$. For any $z\in B_{\D}(x_M,r)$, we may write $z = \exp_{x_M}(rv)$ for some $v\in B(0,1) \subset T_{x_M}\mathcal{M}$. Now, by \eqref{eq:Metric_Expansion}
\begin{align*}
\D(x,z)^{2} &=   \D\left(\exp_{x_M}\left(-\frac{r}{2}e_1\right),\exp_{x_m}(rv)\right)^{2} \\
    & = r^2\left\|\frac{1}{2}e_1 + v\right\|^{2} - \frac{r^4K(e_1,v)}{3}\left(\frac{1}{4}\|v\|^{2} - \frac{1}{4}\langle e_1,v\rangle^{2}\right)   + O(r^{5}) \\
    & \leq r^2\left\|\frac{1}{2}e_1 + v\right\|^{2} - \frac{r^{4}K_{\min}}{12}\left(\|v\|^{2} - \langle e_1,v\rangle^{2}\right) + O(r^{5})
\end{align*}
and similarly: $ \D(z,y)^{2} \leq r^2\left\|-\frac{1}{2}e_1 + v\right\|^{2} - \frac{r^{4}K_{\min}}{12}\left(\|v\|^{2} - \langle e_1,v\rangle^{2}\right) + O(r^{5}).$ We split the analysis into the case where $K_{\min} \geq 0$ and where $K_{\min} < 0$. 

\underline{Case $K_{\min}\ge 0$ (Positive Sectional Curvature):} If $K_{\min} \geq 0$ then the terms proportional to $r^4$ are strictly non-positive, and hence may be dropped. We get:
\begin{align*}
\D(x,z)^{2} &\leq r^{2}\left\|\frac{1}{2}e_1 + v\right\|^{2} + O(r^5) \text{ and } \D(z,y)^{2} \leq r^2\left\|-\frac{1}{2}e_1 + v\right\|^{2} + O(r^5),
\end{align*}and hence $\D(x,z)^{p} + \D(z,y)^{p} \leq r^{p}\left(\left\|\frac{1}{2}e_1 + v\right\|^{p} + \left\|-\frac{1}{2}e_1 + v\right\|^{p}\right) + O(r^{p+3}).$  Thus, for $r$ sufficiently small we may guarantee that 
\begin{equation}
\D(x,z)^{p} + \D(z,y)^{p} \leq r^p = \D(x,y)^{p} 
\label{eq:Pos_curvature_distance_bound}
\end{equation}
by ensuring $\left\|\frac{1}{2}e_1 + v\right\|^{p} + \left\|-\frac{1}{2}e_1 + v\right\|^{p} \leq 1 - \alpha,$ where $\alpha\in (0,1)$ is such that the $O(r^{p+3})$ term is less than $\alpha r^{p}$. As $r\to 0$ with $n\to \infty$, we observe that $\alpha = o(1)$. 

\underline{Case $K_{min}<0$ (Negative Sectional Curvature):} If $K_{\min} < 0$ then one can upper bound the terms proportional to $r^4$ by $-r^{4}K_{\min}/12$ to obtain:
\begin{align*}
& \D(x,z)^{2} \leq r^2\|\frac{1}{2}e_1 + v\|^{2} - \frac{r^{4}K_{\min}}{12} + O(r^{5}),\\
&\D(z,y)^{2} \leq r^2\|-\frac{1}{2}e_1 + v\|^{2} - \frac{r^{4}K_{\min}}{12} + O(r^{5}).
\end{align*}
and so:
\begin{align}
\D(x,z)^{p} +  \D(z,y)^{p} \nonumber \leq & r^p\left(\left\|\frac{1}{2}e_1 + v\right\|^{p}  + \left\|-\frac{1}{2}e_1 + v\right\|^{p}\right) \nonumber \\
& + r^{p+2}\frac{-K_{\min}}{12}\left(\left\|\frac{1}{2}e_1 + v\right\|^{p-2}  + \left\|-\frac{1}{2}e_1 + v\right\|^{p-2}\right) + O(r^{p+3}). \label{eq:negative_curvature_bottleneck} 
\end{align}
As in the positive sectional curvature case, one can guarantee:
\begin{equation}
 \D(x,z)^{p} + \D(z,y)^{p} \leq r^{p} = \D(x,y)^{p} 
 \label{eq:negative_curvature_needed_to_guarantee}
\end{equation}
by ensuring 
\begin{equation}
\left\|\frac{1}{2}e_1 + v\right\|^{p} + \left\|-\frac{1}{2}e_1 + v\right\|^{p} \leq 1 - \alpha
\label{eq:Bound_norms_by_alpha}
\end{equation}
 where $\alpha \in (0,1)$ is such that the $O(r^{p+2})$ term is less than $\alpha r^{p}$. Again, we observe that $\alpha = o(1)$. Note that if \eqref{eq:Bound_norms_by_alpha} holds with $\alpha < 1$ then $\left\|\frac{1}{2}e_1 + v\right\|^{p-2}  + \left\|-\frac{1}{2}e_1 + v\right\|^{p-2} < 1$, and so \eqref{eq:negative_curvature_bottleneck} becomes:
\begin{equation}
\D(x,z)^{p} +  \D(z,y)^{p} \leq \alpha r^{p}  + \frac{-K_{\min}}{12}r^{p+2} + O(r^{p+3}).
\label{eq:negative_curvature_bottleneck_2} 
\end{equation}
For both cases, consider the $p$-elongated set defined in the tangent space:
\[\mathcal{D}_{1-\alpha,p} := \mathcal{D}_{1-\alpha,p}\left(\frac{1}{2}e_1,-\frac{1}{2}e_1\right)  = \left\{v\in T_{x_M}\mathcal{M} \ \bigg| \ \left\|\frac{1}{2}e_1 + v\right\|^{p} + \left\|-\frac{1}{2}e_1 + v\right\|^{p} \leq 1-\alpha\right\}\] as well as its scaled image under the exponential map: $\exp\left(r\mathcal{D}_{1-\alpha,p}\right)$. From the above arguments, \eqref{eq:Pos_curvature_distance_bound} (resp. \eqref{eq:negative_curvature_needed_to_guarantee}) will hold as long as $z\in \exp\left(r\mathcal{D}_{1-\alpha,p}\right)$. By Lemma 3.6 as long as $n$ is sufficiently large that $1-\alpha > 2^{1-p}$ we have $B\left(0,\sqrt{\frac{(1-\alpha)^{2/p}}{4^{1/p}} - \frac{1}{4}}\right)\subset \mathcal{D}_{1-\alpha,p}$ so $B(0,r_1^{\star}) \subset r \mathcal{D}_{1-\alpha,p}$ where $r^{\star} := r\sqrt{\frac{(1-\alpha)^{2/p}}{4^{1/p}} - \frac{1}{4}}$. Hence:
\begin{equation}
B_{\D}(x_M,r^{\star}) = \exp\left(B(0,r^{\star}) \right) \subset \exp\left(r\mathcal{D}_{1-\alpha,p}\right) \subset B_{\D}(x_M,r).
\label{eq:Inclusions}
\end{equation}
As in Theorem~3.9:
\begin{align*}
  \Prob\left[ x_{i_j}\in \exp\left(\mathcal{D}_{1-\alpha,p}\right) \ | \ x_{i_j}\in B_{\D}(x,r)\right] & = \frac{\Prob\left[ x_{i_j}\in \exp\left(\mathcal{D}_{1-\alpha,p}\right) \right]}{\Prob\left[x_{i_j}\in B_{\D}(x,r)\right]} \\
  	& \geq \frac{\Prob\left[ x_{i_j}\in B_{\D}(x_{M},r^{\star})\right]}{\Prob\left[x_{i_j}\in B_{\D}(x,r)\right]} \quad \text{\em (using \eqref{eq:Inclusions})}\\
  	& \geq \frac{f_{\min}V_{\min}(r^{\star})}{f_{\max}V_{\max}(r)}. \quad \text{\em (by definition of $V_{\min}$ and $V_{\max}$)}
\end{align*}

Using \eqref{eq:SmallBall1}, \eqref{eq:SmallBall2} and $r = o(1)$:\[\frac{V_{\min}(r^{\star})}{V_{\max}(r)} = \left(1 - o(1)\right)\frac{\left(r^{\star}\right)^{d}}{r^{d}} = \left(1 - o(1)\right)\left(\frac{(1-\alpha)^{2/p}}{4^{1/p}} - \frac{1}{4}\right)^{d/2}.\]
Recalling that $\alpha = o(1)$:
\begin{align*}
\Prob\left[ x_{i_j}\in \exp\left(\mathcal{D}_{1-\alpha,p}\right) \ | \ x_{i_j}\in B_{\D}(x,r)\right] \geq& \left(1 - o(1)\right)\frac{f_{\min}}{f_{\max}} \left(\frac{(1-o(1))^{2/p}}{4^{1/p}} - \frac{1}{4}\right)^{d/2} \\
=& \varepsilon_{\mathcal{M},p,f}.
\end{align*}

As in Theorem 3.9 for $k \geq \frac{3\log n}{-\log(1-\varepsilon_{\mathcal{M},p,f})}$,
\begin{align}
\Prob\left[\exists j \text{ with } x_{i_j}\in \exp\left(\mathcal{D}_{1-\alpha,p}\right)\right] = 1 - \Prob\left[ \not\exists j \text{ with } x_{i_j}\in \exp\left(\mathcal{D}_{1-\alpha,p}\right)\right] \ge 1 - \frac{1}{n^3}.
\label{eq:Apply_Union_Bound_2}
\end{align}
If $x_{i_j} \in \exp\left(\mathcal{D}_{1-\alpha,p}\right)$ then from \eqref{eq:Pos_curvature_distance_bound} (resp. \eqref{eq:negative_curvature_needed_to_guarantee}):
\begin{equation}
 \D(x,x_{i_j})^{p} + \D(x_{i_j},y)^{p} \leq r^{p} = \D(x,y)^{p} 
 \label{eq:Final_equation}
\end{equation}
and so $\{x,y\}$ is not critical. Thus, $\{x,y\}$ is not critical with probability exceeding $1 - \frac{1}{n^3}$. These edges $\{x,y\}$ are precisely those contained in $\mathcal{G}^{p}_{\mathcal{M},\mathcal{X}}$ but not in $\mathcal{G}^{p,k}_{\mathcal{M},\mathcal{X}}$. There are fewer than $n(n-1)/2$ such non $k$-NN pairs $x,y\in\mathcal{X}$. By the union bound and \eqref{eq:Apply_Union_Bound_2} we conclude that none of these are critical with probability greater than $1 - \frac{n(n-1)}{2}\frac{1}{n^3} \geq 1 - \frac{1}{2n}$.  This was conditioned on $\ell_{\mathcal{M},p}(x,y)\le \text{inj}(\mathcal{M})$ for all $x,y\in \mathcal{X}$, which holds with probability $1-\frac{1}{2n}$.  Thus, all critical edges are contained in $\mathcal{G}_{p,k}^{\mathcal{M},\mathcal{X}}$ with probability exceeding $1-\left(\frac{1}{2n}+\frac{1}{2n}\right)=1-\frac{1}{n}$.  Unpacking $\varepsilon_{\mathcal{M},p,f}$ yields the claimed lower bound on $k$. 
\end{proof}
\section{Estimating the Fluctuation Exponent}
\label{app:Estimate_Fluctuation_Exponent}
Table \ref{tab:FluctuationRatesCI} shows confidence interval estimates for $\chi$ obtained by computing $\tilde{\ell}_{p}$ in a sparse graph. 

\begin{table}[tbh]	
	\begin{center}
		\begin{tabular}{cccc|cccc|cccc}
			\hline
			$d$ &	$p$ &	$\hat{\chi}$ & CI for $\chi$  &	$d$ & $p$ & $\hat{\chi}$ & CI for $\chi$  &	$d$ & $p$ &	$\hat{\chi}$ & CI for $\chi$  \\ \hline
			2 & $1.5$ & 0.30 	& (0.28, 0.32) & 3 & $1.5$ & 0.28  & (0.20, 0.36) & 4 & $1.5$ & 0.19 & (0.03, 0.36) \\ 
			2 & $2$ & 0.31 	&	(0.30, 0.32) & 3 & $2$ & 0.23 	& (0.20, 0.25) & 4 & $2$ & 0.16 & (0.13, 0.19) \\
			2 & $4$ & 0.33	&	(0.31, 0.34) & 3 & $4$ & 0.24 	& (0.22, 0.25) & 4 & $4$ & 0.14 	& (0.11,  0.18) \\
			2 & $8$ & 0.34  & (0.32, 0.37) & 3 & $8$ & 0.29 	& (0.27, 0.32) & 4 & $8$ & 0.19 	& (0.14, 0.23) \\ \hline
		\end{tabular}
		\caption{\label{tab:FluctuationRatesCI}Confidence interval estimates of $\chi$ for uniform data for different density weightings ($p$) and different dimensions $(d)$.  }
	\end{center}
\end{table}
\end{document}